\documentclass[preprint,11pt]{imsart}

\pdfoutput=1

\PassOptionsToPackage{numbers}{natbib}

\RequirePackage[OT1]{fontenc}
\RequirePackage{amsthm,amsmath}
\RequirePackage[numbers]{natbib}
\RequirePackage[colorlinks=true, citecolor=blue, filecolor=black,  urlcolor=blue]{hyperref}
\usepackage{hyperref}

\usepackage[utf8]{inputenc} %
\usepackage[T1]{fontenc}    %
\usepackage{url}            %
\usepackage{booktabs}       %
\usepackage{amsfonts}       %
\usepackage{amssymb}
\usepackage{nicefrac}       %
\usepackage{microtype}      %
\usepackage{graphicx}
\usepackage{epstopdf}
\usepackage{algorithm, algorithmic}
\usepackage{bm}
\usepackage{bbm}
\usepackage{xspace}
\usepackage{tikz}
\usepackage{standalone}
\usepackage{mathtools}
\usepackage{subcaption}
\usepackage[textsize=tiny,textwidth=0.8in]{todonotes}
\usepackage[normalem]{ulem}
\usepackage{amsopn}
\usepackage{cleveref}

\usepackage[centering,letterpaper,margin=1in]{geometry}
\usetikzlibrary{shapes.misc}
\tikzset{cross/.style={cross out, very thick, draw=black, minimum size=2*(#1-\pgflinewidth), inner sep=0pt, outer sep=0pt},
cross/.default={5pt}}

\ifpdf
  \DeclareGraphicsExtensions{.eps,.pdf,.png,.jpg}
\else
  \DeclareGraphicsExtensions{.eps}
\fi

\newcommand\rev[1]{\textcolor{black}{#1}}
\newcommand{\red}[1]{\rev{#1}}

\usepackage{enumitem}
\setlist[enumerate]{leftmargin=.5in}
\setlist[itemize]{leftmargin=.5in}

\newtheorem{theorem}{Theorem}
\newtheorem{proposition}[theorem]{Proposition}
\newtheorem{lemma}[theorem]{Lemma}
\newtheorem{corollary}[theorem]{Corollary}

\newtheorem{remark}{Remark}

\newcommand{\bx}{\bm{x}}
\newcommand{\by}{\bm{y}}
\newcommand{\bz}{\bm{z}}
\newcommand{\bw}{\bm{w}}

\newcommand{\bX}{\bm{X}}
\newcommand{\bY}{\bm{Y}}
\newcommand{\bZ}{\bm{Z}}
\newcommand{\bW}{\bm{W}}

\newcommand{\sx}{X}

\newcommand{\bsx}{\mathbf{\sx}}

\newcommand{\R}{\mathbb{R}}
\renewcommand{\d}{\mathrm{d}}
\newcommand{\balpha}{\boldsymbol{\alpha}}
\newcommand{\bmu}{\boldsymbol{\mu}}
\newcommand{\bnu}{\boldsymbol{\nu}}
\newcommand{\bchi}{\boldsymbol{\chi}}
\newcommand{\Rectifier}{\mathcal{R}}
\newcommand{\KR}{\textrm{KR}}

\newcommand{\ALG}{\texttt{ATM}\xspace}

\newcommand{\CKDE}{\texttt{CKDE}\xspace}
\newcommand{\NKDE}{\texttt{NKDE}\xspace}
\newcommand{\MDN}{\texttt{MDN}\xspace}
\newcommand{\KMN}{\texttt{KMN}\xspace}
\newcommand{\NF}{\texttt{NF}\xspace}
\DeclareMathOperator*{\argmin}{arg\,min}
\DeclareMathOperator*{\argmax}{arg\,max}
\DeclareMathOperator*{\esssup}{ess\,sup}
\DeclareMathOperator*{\essinf}{ess\,inf}

\begin{document}

\begin{frontmatter}

\title{On the representation and learning of monotone triangular transport maps}

\runtitle{Learning monotone triangular maps}

\begin{aug}
\author[a1]{\fnms{Ricardo} \snm{Baptista}%
\ead[label=e1]{rsb@caltech.edu}},
\author[a2]{\fnms{Youssef} \snm{Marzouk}%
\ead[label=e2]{ymarz@mit.edu}},
\and
\author[a3]{\fnms{Olivier} \snm{Zahm}%
\ead[label=e3]{olivier.zahm@inria.fr}}

\runauthor{Baptista, Marzouk, Zahm}

\address[a1]{California Institute of Technology\\
Pasadena, MA 91125 USA\\
\printead*{e1},
}
\address[a2]{Massachusetts Institute of Technology\\
Cambridge, MA 02139 USA\\
\printead*{e2},
}
\address[a3]{Universit\'e Grenoble Alpes, Inria, CNRS, Grenoble INP, LJK\\
38000, Grenoble, France\\
\printead*{e3}
}
\end{aug}

\begin{abstract}
Transportation of measure provides a versatile approach for modeling complex probability distributions, with applications in density estimation, Bayesian inference, generative modeling, and beyond. Monotone triangular transport maps---approximations of the Knothe--Rosenblatt (KR) rearrangement---are a canonical choice for these tasks.
Yet the representation and parameterization of such maps have a significant impact on their generality and expressiveness, and on properties of the optimization problem that arises in learning a map from data (e.g., via maximum likelihood estimation). 
We present a general framework for representing monotone triangular maps via invertible transformations of smooth functions. 
We establish conditions on the transformation such that the associated infinite-dimensional minimization problem has no spurious local minima, i.e., all local minima are global minima; and we show for target distributions satisfying certain tail conditions that the unique global minimizer corresponds to the KR map. Given a sample from the target, we then propose an adaptive algorithm that estimates a sparse semi-parametric approximation of the underlying KR map. We demonstrate how this framework can be applied to joint and conditional density estimation, likelihood-free inference, and structure learning of directed graphical models, with stable generalization performance across a range of sample sizes.
\end{abstract}

\begin{keyword}
\kwd{Knothe--Rosenblatt rearrangement}
\kwd{normalizing flows}
\kwd{monotone functions}
\kwd{infinite-dimensional optimization}
\kwd{adaptive approximation}
\kwd{multivariate polynomials}
\kwd{wavelets}
\kwd{density estimation}
\end{keyword}

\begin{keyword}[class=MSC]
\kwd{65C20} %
\kwd[, ]{49Q22} %
\kwd[, ]{62G07} %
\kwd[, ]{41A10} %
\end{keyword}

\end{frontmatter}

\setcounter{tocdepth}{1}
\tableofcontents

\section{Introduction} \label{sec:introduction}

Many sampling, estimation, and inference algorithms seek to characterize a somehow intractable or complex %
probability distribution $\bmu$ on $\R^d$. Transportation of measure provides a useful and versatile approach to this problem. The underlying idea is to construct a coupling of $\bmu$ with a tractable ``reference'' distribution $\bnu$ on $\R^d$---for instance, a standard normal. Formally, one jointly constructs a pair of random variables $(\bX, \bZ)$ such that $\bX \sim \bmu$ and $\bZ \sim \bnu$. %
A special class of couplings is given by \emph{deterministic} transformations $S\colon \R^{d} \rightarrow \R^{d}$ such that $S(\bX) = \bZ$ in distribution, and a transformation $S$ that satisfies this property is called a \emph{transport map}~\cite{villani2008}. As a result, if the transport map is invertible, one can generate realizations of $\bX$ by first simulating $\bZ$. If $\bmu$ and $\bnu$ have densities $\pi$ and $\eta$
with respect to a common base measure, 
one can also explicitly represent the target density $\pi$ as a transformation of the reference density $\eta$. 
The construction of such transport maps has found numerous applications: density estimation \cite{tabak2013family,anderes2012general,dinh2017,huang2020convex}; variational Bayesian inference \cite{el2012bayesian,rezende2015,bigoni2019greedy}; generative modeling of images, video, and other structured objects \cite{oord2017parallel, kingma2018}; likelihood-free inference \cite{papamakarios2016fast,radev2020bayesflow,lueckmann2021benchmarking}; and beyond.

In general, there exist infinitely many transport maps between two absolutely continuous distributions on $\R^d$. In this paper, we focus on a specific, canonical choice:  \emph{triangular} transport maps \cite{bogachev2005triangular} of the form
\begin{equation}\label{eq:increasingMaps}
 S(\bm{x})=
 \begin{bmatrix*}[l]
  S_{1}(x_{1}) \\ S_{2}(x_{1},x_{2}) \\ ~~\vdots \\S_{d}(x_{1},\hdots,x_{d})
 \end{bmatrix*},
\end{equation}
where each component function $S_k$ depends only on the $k$ variables $\bm{x}_{\leq k} \coloneqq (x_1,\hdots,x_{k})$ and is monotone increasing with respect to the last input $x_k$. In particular, each function $S_k$ encodes an increasing rearrangement \cite{villani2008} between ordered marginal conditionals of $\bmu$ and $\bnu$, i.e., $\bmu(\d x_k \vert \bx_{<k})$ and $\bnu(\d z_k \vert \bz_{<k})$. Later we will discuss properties of such triangular maps---known also as Knothe--Rosenblatt (KR) rearrangements \cite{villani2008,rosenblatt1952,santambrogio2015optimal}---more precisely, but we comment here on their utility. First of all, triangular structure facilitates computational tractability: $S$ is easy to invert and the determinant of its Jacobian (a lower triangular matrix) is easy to evaluate (see Section~\ref{sec:background}). Triangular maps have thus been used extensively in the density estimation, inference, and generative modeling applications noted above. For example, triangular maps are the building blocks of many \emph{normalizing flows}, popularized by the machine learning community \cite{kobyzev2020normalizing,papamakarios2021normalizing}; specifically, autoregressive normalizing flows define triangular maps \cite{jaini2019} via particular structural choices and parameterizations, and compose these maps to produce more expressive transformations. More fundamentally, because triangular maps expose certain conditionals of $\bmu$, they are particularly well suited to \emph{conditional density estimation} and \emph{conditional sampling}; we will describe this link in Section~\ref{sec:background}. Triangular maps also inherit \emph{sparsity} from the conditional independence properties of $\bmu$ and $\bnu$, as detailed in \cite{spantini2018inference}. %

It is worth noting, of course, that other canonical choices of transport have different attractive features. For instance, optimal transport maps are invariant under relabeling of the coordinates or more general isometries on $\mathbb{R}^d$ (unlike triangular maps), and have deep links to partial differential equations \cite{villani2008}. But optimal transport maps are in general more challenging to represent, evaluate, and estimate from data in the continuous setting; moreover, they do not enjoy such a direct link to conditioning or to graphical models. %

Many representations and finite-dimensional parameterizations of monotone triangular maps have been proposed in recent years. These include representations based on polynomials~\cite{marzouk2016sampling,jaini2019}, radial basis functions~\cite{tabak2013family}, neural networks of varying capacity~\cite{dinh2017,kingma2018}, and tensor decompositions~\cite{cui2021deep,cui2021conditional}. A core challenge in this setting is to satisfy the monotonicity constraint $\partial_{x_k} S_k > 0$. For instance, one might enforce the monotonicity constraint at a finite collection of points in the support of $\pi$~\cite{parno2018transport}, but this approach 
cannot in general guarantee that $S$ is monotone over the entire support of $\pi$. %
Other efforts have sought to enforce monotonicity by construction---via the parameterization of $S$ itself. For example,~\cite{papamakarios2017masked} employs map components with affine dependence on the last variable, i.e., $S_k(\bx_{\leq k}) = \alpha(\bx_{< k}) + \exp(\beta(\bx_{< k}))x_{k}$, where $\alpha$ and $\beta$ are neural networks. While $S$ is then guaranteed to be monotone, it can only represent a restricted class of distributions $\bmu$. (If $\bnu$ is Gaussian, then $\bmu$ can only be a product of Gaussian marginal conditionals.) Such representations therefore cannot consistently approximate the \emph{true} KR rearrangement for general $\bmu$. 
Recent work~\cite{teshima2020coupling} has shown that a \textit{composition} of such affine maps, interleaved with rotations and permutations, can approximate a general class of distributions, though approximation rates remain unknown. The required rotations or permutations break the overall triangular structure of the transformation, however.
Alternatively, to increase the ``expressiveness'' of a given triangular function, \cite{wehenkel2019, huang2018neural, durkan2019neural} have introduced more general parametric representations of the monotone component functions $S_{k}$.
For distributions with analytic densities on \emph{bounded domains}, a complete approximation theory for the KR map was recently developed in \cite{zech2021sparse,zech2022sparse}, using polynomial or ReLU neural network representations of range-constrained monotone triangular functions.

Despite these myriad proposals, relatively little attention has been paid to the structure and tractability of the \emph{optimization problem} involved in learning triangular maps (e.g., in estimating a map given an i.i.d.\ sample $\{\bX^i\}$ from $\bmu$). Properties of this optimization problem are intimately tied to the representation and parameterization of the associated map. It is desirable to have a flexible and general representation---one that can consistently recover the KR rearrangement for a broad class of distributions $(\bmu, \bnu)$---that at the same time makes optimization tractable. It is also desirable to have adaptivity: a parameterization whose size or complexity can be adapted to properties of the target distribution and the available sample size, for good empirical statistical performance.

This paper directly addresses these desiderata. We do so by developing and analyzing a \emph{functional framework} for representing and learning triangular maps. Our main contributions are as follows. We propose a general representation of monotone triangular functions, based on a \emph{rectification} operator $\mathcal{R}_k$ that transforms sufficiently smooth non-monotone functions $f_k: \R^k \to \R$ into monotone component functions $S_k$ of a triangular map.
This operator takes the form
\begin{equation}\label{eq:rectifierintro}
 \mathcal{R}_k(f_k)(\bx_{\leq k}) = f_k(\bx_{<k},0) + \int_0^{x_k} g\big(\partial_{k} f_k(\bx_{<k},t)\big)\mathrm{d}t ,
\end{equation}
where $g \colon \R \to \R_{>0}$ is a positive function.
We then analyze the infinite-dimensional optimization problem associated with learning maps from data, recast as optimization over functions $\{f_k\}_{k=1}^d$. Specifically, we establish 
conditions on the rectification operator and on the target distribution 
such that the resulting optimization problem is well-posed, smooth, and has \emph{no spurious local minima}. Under further conditions on the target distribution (essentially that it has Gaussian tails), we show that the optimization problem has a \emph{unique global minimizer} corresponding to the KR map. %

These theoretical results guarantee, in practice, fast and reliable learning of monotone triangular maps given an appropriate function space $V_k$ for each $f_k$. The second main contribution of our paper is algorithmic: given a hierarchical basis (e.g., polynomials or wavelets) for each $V_k$, we propose a greedy adaptive procedure to learn parametric representations of $f_k$. The procedure naturally produces map representations that are \emph{sparse} and interpretable---in particular, it exploits and implicitly discovers {conditional independence}. We use these learned maps for density estimation, given an 
i.i.d.\ sample $\{\bX^i\}$ from $\pi$. Maintaining a strict triangular structure also exposes marginal conditionals of the target density, thus enabling \emph{conditional density estimation}. Our numerical experiments show that the algorithm provides robust performance at small-to-moderate sample sizes, and constitutes a \emph{semi-parametric} approach that naturally links map complexity to the size of the data.

The remainder of the paper is organized as follows. Section~\ref{sec:background} recalls properties of triangular transport maps and introduces some estimation problems of interest. Our main theoretical contributions are in Section~\ref{sec:monotone_maps}, which introduces a framework for representing monotone triangular maps and analyzes the resulting optimization problem. Section~\ref{sec:ATM} then introduces our greedy adaptive procedure for learning maps, and Section~\ref{sec:experiments} contains numerical experiments. Proofs of certain theoretical results are deferred to the appendix.

\section{Triangular transport for density estimation and simulation} \label{sec:background}
Consider the unsupervised learning problem of approximating a target probability density function $\pi$ defined on $\mathbb{R}^d$, given an i.i.d.\ sample from $\pi$.
Our goal is to construct a sufficiently smooth and invertible map $S\colon\mathbb{R}^d\rightarrow \mathbb{R}^d$ such that the pullback density 
\begin{equation} \label{eq:pullback_density}
 S^{\sharp}\eta (\bx)= \eta  \left ( S(\bm{x})  \right ) |\det \nabla S(\bx)|,
\end{equation}
is a good approximation to $\pi$, where $\eta$ is a simple/tractable probability density function on $\mathbb{R}^d$.
The choice of $\eta$ is a degree of freedom of the method, and here we take $\eta(\bx)\propto\exp(-\|\bx\|^2/2)$; 
i.e., $\eta$ is the density of the standard normal distribution on $\mathbb{R}^d$, where $\|\cdot\|$ is the canonical norm of $\R^d$. %

To ensure invertibility of $S$, a common practice is to constrain $S$ to be an increasing lower triangular map of the form \eqref{eq:increasingMaps},
where each component $S_{k}\colon\mathbb{R}^k\rightarrow\mathbb{R}$ is such that $x_k\mapsto S_{k}(\bx_{<k},x_k)$ is increasing for all $\bm{x}_{<k}=(x_1,\hdots,x_{k-1})\in\mathbb{R}^{k-1}$.
Such a map is easy to invert\footnote{For any $\bm{z}\in\mathbb{R}^d$, $\bm{x}=S^{-1}(\bz)$ can be computed recursively as $x_{k}=T^{k}(\bx_{< k},z_k)$ for $k=1,\dots,d$, where the function $T^{k}(\bx_{< k},\cdot)$ is the inverse of $x_k\mapsto S_{k}(\bx_{< k},x_k)$. \red{In practice, evaluating $T^{k}$ requires solving a root-finding problem which is guaranteed to have a unique (real) root, and for which the bisection method converges geometrically fast. Therefore, $S^{-1}(\bz)$ can be evaluated to machine precision in negligible computational time.
}} %
and has a lower triangular Jacobian $\nabla S(\bx)$ so that $|\det\nabla S(\bx)| = \prod_{k=1}^{d} \partial_{k} S_{k}(\bx_{\leq k})$ is readily computable.
\red{This is a major benefit of working with monotone triangular maps.}
This structure in fact corresponds to the \emph{Knothe--Rosenblatt} (KR) rearrangement 
$S_{\text{KR}}$: the increasing lower triangular map satisfying
$$
 \pi(\bx)=S_{\text{KR}}^{\sharp}\eta(\bx).
$$
For a measure $\bmu$ on $\R^d$ that is absolutely continuous with respect to a Gaussian measure $\bnu$ (and hence has a density $\pi$ with respect to the Lebesgue measure), %
the KR rearrangement $S_{\KR}$ exists and is the \emph{unique} map of the form~\eqref{eq:increasingMaps} that \emph{pulls back} $\eta$ to $\pi$ (or equivalently \emph{pushes forward} $\pi$ to $\eta$), up to sets of measure zero~\cite{bogachev2005triangular}.

A useful way to measure the difference between $\pi$ and its approximation $S^{\sharp}\eta$ is the Kullback--Leibler (KL) divergence $\mathcal{D}_\text{KL}(\pi || S^{\sharp}\eta) = \int\log(\pi/S^{\sharp}\eta)\d\pi$. As explained below, this choice has direct links to maximum likelihood estimation of $S$. 
Furthermore, the following inequality shows that convergence in the KL sense $\mathcal{D}_\text{KL}(\pi || S^{\sharp}\eta) \rightarrow0$ implies convergence of $S$ towards $S_{\mathrm{KR}}$ in the $L^2_\pi$ sense.
This result is a direct consequence of Corollary 3.10 in 
\cite{bogachev2005triangular}; see~\Cref{proof:KLboundsL2norm} for a proof. 
\begin{proposition} \label{prop:KLboundsL2norm}
Let $S_{\KR}$ be the KR rearrangement pushing forward a distribution with density $\pi$ on $\R^d$ to the standard normal distribution on $\R^d$, with density $\eta$. For any map $S:\R^d\rightarrow\R^d$ as in \eqref{eq:increasingMaps}, we have
\begin{equation} \label{eq:KLboundsL2norm}
  \int\|S_{\mathrm{KR}}(\bx)-S(\bx)\|^2\d\pi(\bx) \leq 2 \mathcal{D}_\text{KL}(\pi || S^{\sharp}\eta).
\end{equation}
\end{proposition} %

Since the standard normal density $\eta$ is a product of its marginal densities, 
the KL divergence decomposes as
\begin{align} \label{eq:KL_decomposition}
 \mathcal{D}_\text{KL}(\pi || S^{\sharp}\eta)
 =\sum_{k=1}^d \mathcal{J}_{k}(S_{k}) -\mathcal{J}_{k}( S_{\text{KR},k}),
\end{align}
where the functionals $\mathcal{J}_{1},\hdots,\mathcal{J}_{d}$ are given by
\begin{equation}\label{eq:opt_map_component}
 \mathcal{J}_{k}(s) = 
\int \left( \frac{1}{2}s(\bx_{\leq k})^2-\log\left|\partial_k s(\bx_{\leq k})\right | \right)\pi(\bx)\d\bx.
\end{equation}
Minimizing the KL divergence~\eqref{eq:KL_decomposition} over triangular maps $S$ of the form \eqref{eq:increasingMaps} is therefore equivalent to \emph{independently} minimizing each objective functional $\mathcal{J}_k$ to find the associated map component $S_k$~\cite{marzouk2016sampling}. Solution of these optimization problems is thus embarrassingly parallel. This parallel structure was also exploited for Cholesky factorization via KL minimization in~\cite{schafer2021sparse}. %
In addition, minimizing each objective $s\mapsto \mathcal{J}_k(s)$ over functions $s\colon \mathbb{R}^k\rightarrow\mathbb{R}$ that are strictly increasing in the last variable is a \emph{strictly convex} optimization problem; see~\Cref{lem:convexity}.

Given an i.i.d.\ sample $\{\bsx^i\}_{i=1}^{n}$ from $\pi$, we can replace the expectation in~\eqref{eq:opt_map_component} by the sample average, which yields the objective %
\begin{equation}\label{eq:empirical_objective}
\widehat{\mathcal{J}}_k(s) = \frac{1}{n} \sum_{i=1}^{n} \left(\frac{1}{2}s(\bsx^i_{\leq k})^2-\log\left|\partial_{k} s(\bsx^i_{\leq k})\right | \right).
\end{equation}
Minimizing~\eqref{eq:empirical_objective} under the constraint $\partial_k s(\bx_{\leq k}) > 0$ produces an estimator $\widehat{S}_k$ of $S_{\text{KR},k}$. The collection of all such component functions defines an estimator $\widehat{S} = (\widehat{S}_1, \ldots, \widehat{S}_d)$ of the KR rearrangement, along with an estimate of the density $\pi$ as $\widehat{\pi}(\bx) \coloneqq \widehat{S}^{\sharp}\eta(\bx)$. This $\widehat{S}$ is also the maximum likelihood estimator of $S_{\text{KR}}$, i.e.,  
$$
\widehat{S} = \argmax_{\{ S \in \text{\eqref{eq:increasingMaps}}, \ \partial_k S_k > 0, \ k=1, \ldots, d \} }
\sum_{i=1}^n \log S^\sharp \eta(\bX^i),
$$
with the optimization being over the space of monotone increasing and triangular maps of the form~\eqref{eq:increasingMaps}. 
This connection between maximum likelihood estimation and minimization of an empirical forward KL divergence is standard. 
Furthermore, convexity of the optimization objective is preserved when replacing the expectation in~\eqref{eq:opt_map_component} by the sample average in \eqref{eq:empirical_objective}.

A core question when estimating maps and densities by minimizing \eqref{eq:empirical_objective} is how to parameterize sufficiently expressive monotone map components $S_k$---i.e., maps capable of representing a wide class of distributions $\pi$---while ensuring that the optimization problem can be solved efficiently. As explained in the introduction, this question is intimately tied to the monotonicity constraint $\partial_k S_k(\bx_{\leq k}) > 0$.
For example, choosing a linear parameterization for $S_k$ that admits only affine dependence on $x_k$ (to easily enforce monotonicity) allows the map component be identified efficiently through the solution of a least-squares problem (see \cite[Appendix~A]{spantini2019coupling}), but such maps can only capture distributions that factor into a sequence of Gaussian marginal conditionals. 
On the other hand, a more complex ansatz for $S_k$ will often yield a much more difficult optimization problem. Note that with any nonlinear parameterization of $S_k$, the convexity of \eqref{eq:opt_map_component} and \eqref{eq:empirical_objective} with respect to $S_k$ does not in general yield convexity in the parameters. 
As we shall demonstrate later, many nonlinear parameterizations that enforce monotonicity yield optimization problems that may not even be smooth, and that have many local minima. We will address these issues in \Cref{sec:monotone_maps}.

\medskip

\paragraph{Conditional density estimation and sampling} Another important feature of the triangular structure~\eqref{eq:increasingMaps} is that each component of the map represents one marginal conditional density of $\pi$. More precisely, in the present setting where $\eta$ is a product density, $S_{\text{KR},k}$ pushes forward the marginal conditional $\pi_k(x_k|\bx_{<k})$ to the $k$-th marginal of the reference $\eta_k(z_k)$ \cite[Section 2.3]{santambrogio2015optimal}.
Now partition $\bx = (\by, \bw)$, where $\by \in \mathbb{R}^{m}$ and $\bw \in \mathbb{R}^{p}$, with $d=m+p$. This property of the KR map lets us estimate the conditional probability density function $\pi(\bw|\by)$, for any value of $\by$, given a sample $\{(\bY^i,\bW^i)\}_{i=1}^{n}$ from the joint density $\pi(\by,\bw)$. Observe that the KR map immediately has the block structure:
\begin{equation}\label{eq:ConditionalMaps}
    S(\by,\bw) = \begin{bmatrix*}[l] S^{\mathcal{Y}}(\by) \\ S^{\mathcal{W}}(\by,\bw) \end{bmatrix*},
\end{equation} 
where $\by \mapsto S^{\mathcal{Y}}(\by)\colon\mathbb{R}^m\rightarrow\mathbb{R}^m$ and $\bw \mapsto S^{\mathcal{W}}(\by^\ast,\bw)\colon\mathbb{R}^p\rightarrow\mathbb{R}^p$ are increasing lower triangular maps, the latter for any $\by^
{\ast}\in\mathbb{R}^m$. Recall that the reference density $\eta$ is the standard normal on $\R^{d}$, and thus is a product of two standard normals, $\eta_{1}(\bz')\eta_{2}(\bz)$, where $\bz' \in \R^m$ and $\bz \in \R^p$. Using a KR map of the form \eqref{eq:ConditionalMaps}, we can write the marginal density of $\bY$ as $\pi(\by)=(S_{\text{KR}}^{\mathcal{Y}})^{\sharp}\eta_1(\by)$ and, more interestingly, the conditional density of $\bW$ as $ \pi(\bw|\by) = S^{\mathcal{W}}_{\text{KR}}(\by,\cdot)^{\sharp}\eta_2(\bw)$.

Each of the last $p$ components of the KR rearrangement, $S^{\mathcal{W},k}_{\text{KR}}(\by,\bw_{\leq k})$, $1\leq k\leq p$, can be estimated from a sample $\{(\bY^i,\bW^i)\}_{i=1}^{n}$ by minimizing 
\begin{equation}\label{eq:empirical_objective_conditional}
 \widehat{\mathcal{J}}_k(s) = \frac{1}{n}\sum_{i=1}^n\left( \frac{1}{2}s(\bY^i,\bW^i_{\leq k})^2-\log\left|\partial_{m+k} s(\bY^i,\bW^i_{\leq k})\right | \right),
\end{equation}
under the constraints $\partial_{m+k} s(\by,\bw_{\leq k})>0$. This produces an estimator $\widehat{S}^{\mathcal{W}}$ of $S_{\text{KR}}^{\mathcal{W}}$, which in turn yields an estimator of the conditional density $\pi(\bw|\by)$ as $\widehat{\pi}(\bw|\by) \coloneqq \widehat{S}^{\mathcal{W}}(\by,\cdot)^{\sharp}\eta_{2}(\bw)$.
This property has been used to perform conditional density estimation (CDE) in~\cite{marzouk2016sampling,papamakarios2016fast, kovachki2020conditional,cui2021conditional}. 

One important application of CDE is likelihood-free inference, where $\bW$ represents a parameter to be inferred and $\bY$ are data whose conditional density $\pi(\by \vert \bw)$ is computationally intractable or unavailable in closed form. This setting arises in many applications: inference in stochastic models, or inference in the presence of high-dimensional nuisance parameters or latent variables (including parameter inference for state-space models).
Given a joint sample  $\{(\bY^i,\bW^i)\}_{i=1}^{n}$, we can estimate the map component $S^{\mathcal{W}}$ and use it to simulate from the estimated conditional density $\widehat{\pi}(\bw|\by^*)$ given any realization of the data $\by^*$, simply by sampling $\bz^i \sim \eta_{2}$ and solving the triangular system $\widehat{S}^{\mathcal{W}}(\by^*,\bw^i) = \bz^i$ for each $\bw^i$. Note that we can simulate from or evaluate the estimated conditional densities for multiple realizations of the data $\by^*$, including values that are not present in the dataset $\{\bY^i\}_{i=1}^{n}$. Thus, learning a single map $S^{\mathcal{W}}$ parameterized by $\by$ is said to \emph{amortize} the cost of conditional sampling over the data.

\section{Representing and learning continuous monotone functions}
\label{sec:monotone_maps}
In this section, we define a general representation for components of monotone triangular maps and present our main theoretical results. The essential idea, as mentioned in the introduction, is to express each  component $S_k$ as a nonlinear transformation $S_{k}=\mathcal{R}_k(f)$ of a smooth function $f$, where $\Rectifier_k$ \eqref{eq:rectifierintro} is an operator that enforces the monotonicity constraint by construction. We then identify the map component by solving the re-parameterized optimization problem
\begin{equation}\label{eq:minJoverV}
 \min_{f\in V_k} \mathcal{L}_k(f) , \quad\text{where }  \mathcal{L}_k(f) = \mathcal{J}_k( \mathcal{R}_k(f) ),
\end{equation}
where $V_k$ is a linear space of functions in which we seek $f$. 
With this nonlinear transformation, we lose the convexity of the constrained problem $\min_{\{s\colon \partial_k s > 0\}} \mathcal{J}_k(s)$, but obtain an unconstrained minimization problem instead. It turns out that, with appropriate conditions on $\Rectifier_k$, the transformed optimization problem retains many desirable properties.

In Section~\ref{subsec:choice_g}, we discuss the construction of $\Rectifier_k$ and motivate certain critical choices therein. Then, in Section~\ref{subsec:choice_Vk}, we analyze the regularity of the \emph{composed} objective functional $\mathcal{L}_k$ for a certain choice of the function space $V_k$. In Section~\ref{subsec:opt_properties} we discuss the existence and uniqueness of solutions to \eqref{eq:minJoverV}, and describe conditions under which solving~\eqref{eq:minJoverV} exactly recovers the KR rearrangement.

\subsection{The rectification operator}
\label{subsec:choice_g}

Recalling \eqref{eq:rectifierintro}, for any sufficiently smooth $f\colon\R^k\rightarrow\R$ we let $\mathcal{R}_k(f)\colon\R^k\rightarrow \R$ be the function defined by
\begin{equation}\label{eq:bijective_transformation}
 \mathcal{R}_k(f)(\bx_{\leq k}) =  f(\bx_{<k},0) + \int_0^{x_k} g\big(\partial_k f(\bx_{<k},t)\big)\mathrm{d}t ,
\end{equation}
where $g\colon\R\rightarrow\R_{>0}$ is a positive function.
We call the operator $\mathcal{R}_k\colon f\mapsto\mathcal{R}_k(f)$ a \emph{rectifier} because it transforms any $f$ %
into a function which is increasing in the $k$-th variable, i.e., $\partial_k\mathcal{R}_k(f)(\bx_{\leq k})=g(\partial_k f(\bx_{\leq k}))>0$. 
As a simple example, functions of the form $f(\bx_{\leq k})=\alpha(\bx_{<k})+\beta(\bx_{<k}) x_k$ are transformed into $\mathcal{R}_k(f)(\bx_{\leq k})=\alpha(\bx_{<k})+g(\beta(\bx_{<k})) x_k $. Figure~\ref{fig:tmapprox-fandS} illustrates the application of the rectifier to a nonlinear univariate function $f$.

\begin{figure}[!htb]
    \centering
    \begin{minipage}[b]{0.4\textwidth}
        \centering
        \includegraphics[width=\linewidth]{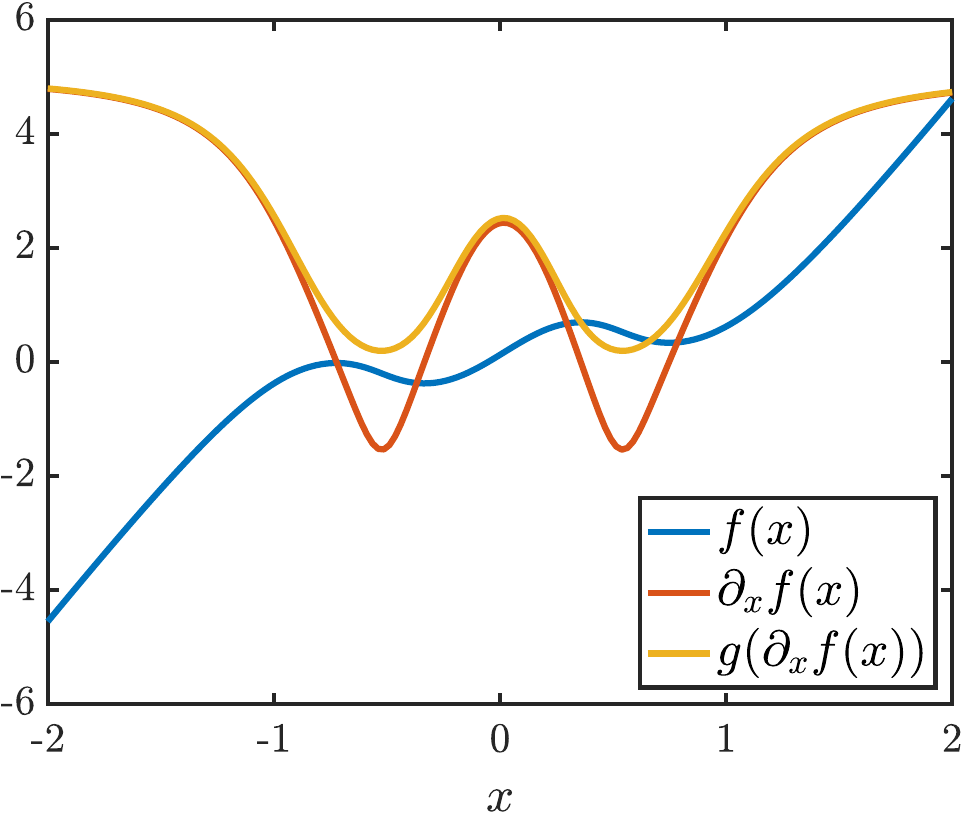}
    \end{minipage}%
    \hspace{1cm}
    \begin{minipage}[b]{0.4\textwidth}
        \centering
        \includegraphics[width=\linewidth]{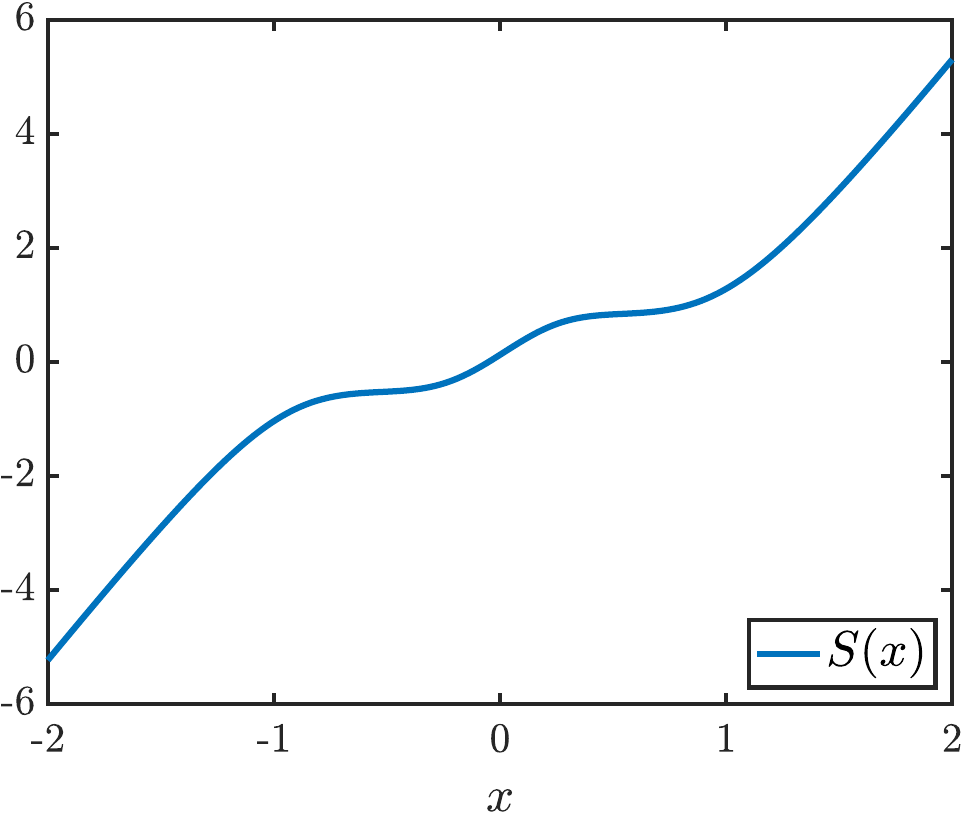}
    \end{minipage}
    \caption{The rectifier \eqref{eq:rectifierintro} transforms the non-monotone function $f$ into the monotone function $S = \mathcal{R}(f)$. Here we choose $g(\cdot) = \log (1 + \exp(\cdot) )$. $S$ is an increasing transport that pushes forward a one-dimensional mixture of Gaussians $\pi(x) = 0.5\mathcal{N}(x;-1,1) + 0.5\mathcal{N}(x; 1,1)$ to the standard Gaussian reference density $\eta$. \label{fig:tmapprox-fandS}}
\end{figure}

The choice of the function $g$ in \eqref{eq:rectifierintro} has a crucial impact on properties of the optimization problem \eqref{eq:minJoverV}.
One possible choice, proposed in \cite{jaini2019, bigoni2019greedy}, is the square function $g(\xi) = \xi^2$. While this choice permits closed-form computation of the integral in \eqref{eq:bijective_transformation} when $f$ is polynomial, it yields an optimization problem \eqref{eq:minJoverV} which possesses many spurious local minima and is in fact non-smooth; see Figure \ref{fig:compare_obj} for an illustration. 
This can be explained in part %
by the fact that this $g$ is not bijective. 

Instead, one can choose $g$ to be a bijective function from $\R$ to $\R_{>0}$, such as the soft-plus function,
\begin{equation}\label{eq:def_g}
    g(\xi) = \log(1 + \exp(\xi)),
\end{equation}
whose inverse is $g^{-1}(\xi) = \log(\exp(\xi)-1)$. Another example of a bijective function, considered in~\cite{wehenkel2019}, is the shifted exponential linear unit (ELU),
\begin{equation} \label{eq:elu}
    g(\xi) = \left\{\begin{array}{ll} \exp(\xi) & \xi < 0 \\
    \xi + 1 & \xi \geq 0 \end{array} \right. ,
\end{equation}
whose inverse is $g^{-1}(\xi)=\xi-1$ if $\xi\geq1$ and  $g^{-1}(\xi)=\log(\xi)$ otherwise. 

As a consequence of $g$ being bijective, the inverse of the rectifier $\mathcal{R}_k^{-1}(s)$ exists for any sufficiently smooth $s\colon\R^k\rightarrow\R$ with $\partial_k s(\bx_{\leq k})>0$ and can be written as
\begin{equation}\label{eq:bijective_transformation_inverse}
 \mathcal{R}_k^{-1}(s)(\bx_{\leq k}) =  s(\bx_{<k},0) + \int_0^{x_k} g^{-1}\big(\partial_k s(\bx_{<k},t)\big)\mathrm{d}t.
\end{equation}
More importantly, the fact that $g$ is invertible yields an objective function $\mathcal{L}_k$ that is far better behaved than with $g(\xi)=\xi^2$; see the numerical illustration in Figure \ref{fig:compare_obj}, with both soft-plus and shifted ELU $g$.  
With these choices of $g$, we observe that $\mathcal{L}_k = \mathcal{J}_k \circ \mathcal{R}_k$, though non-convex in general, has no local minima and no saddle points.  In the next sections, we will analyze these properties of the optimization problem \eqref{eq:minJoverV} (e.g., smoothness, existence and uniqueness of solutions) and elucidate their precise dependence on $g$, $V_k$, and $\pi$.

\begin{remark}
In the finite-dimensional setting, it can easily be shown that composing a smooth and (strictly) convex objective function with a $\mathcal{C}^1$-diffeomorphic map preserves the (unique) global minima of the objective. Such diffeomorphic maps have been explored for accelerating optimization on finite-dimensional manifolds; see, e.g., \cite{lezcano2019trivializations}. Establishing that the \emph{operator} $\Rectifier_k$  in~\eqref{eq:bijective_transformation} is $\mathcal{C}^1$-diffeomorphic is, however, non-trivial. We will show in Section~\ref{subsec:opt_properties} that, under additional assumptions, our reparameterization yields an infinite-dimensional optimization problem where local minima are global minima. 
\end{remark}

\begin{figure}[!ht]
\begin{minipage}{.32\linewidth}
	\centering
	\includegraphics[width=\textwidth]{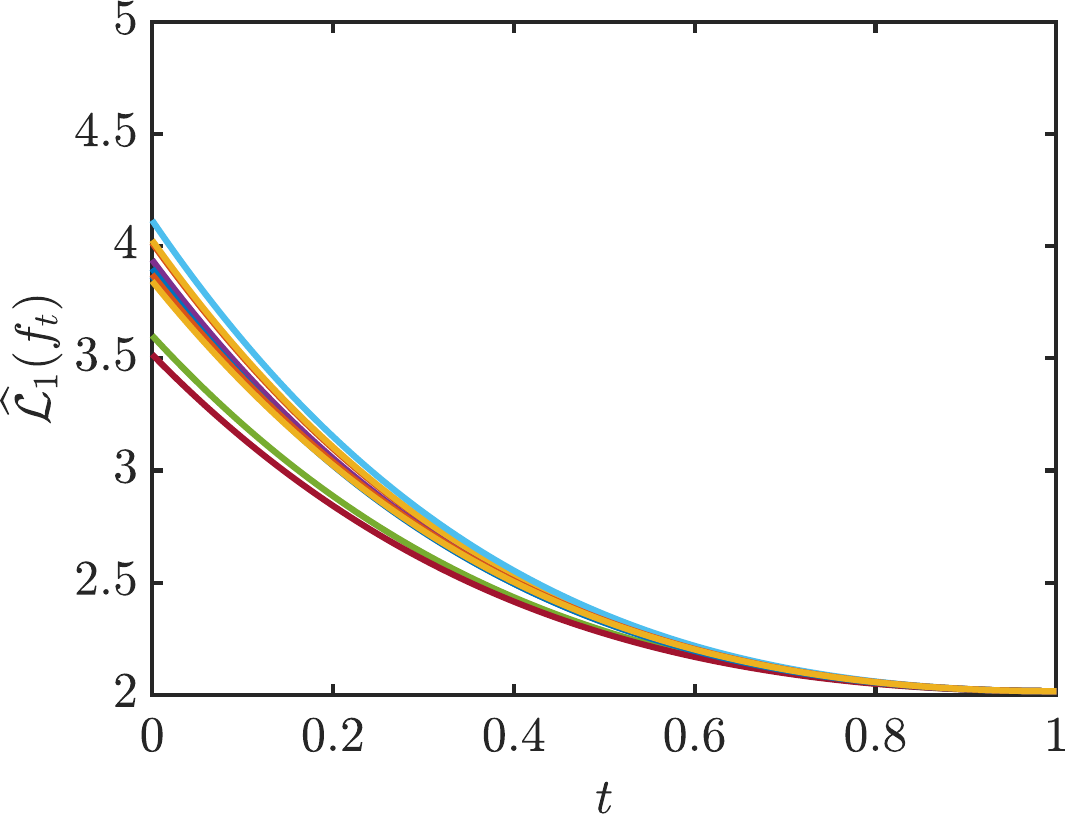}
\end{minipage}
\hspace{1pt}
\begin{minipage}{.32\linewidth}
	\centering
	\includegraphics[width=\textwidth]{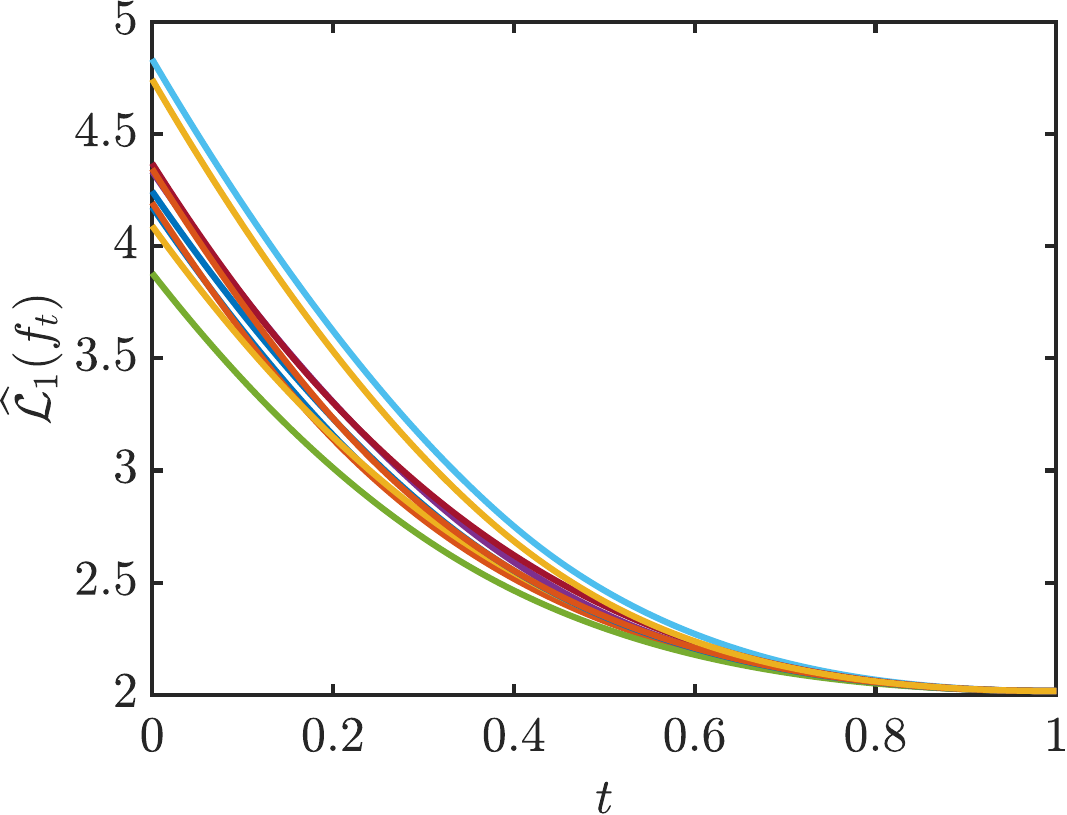}
\end{minipage}
\hspace{1pt}
\begin{minipage}{.32\linewidth}
   \centering
	\includegraphics[width=\textwidth]{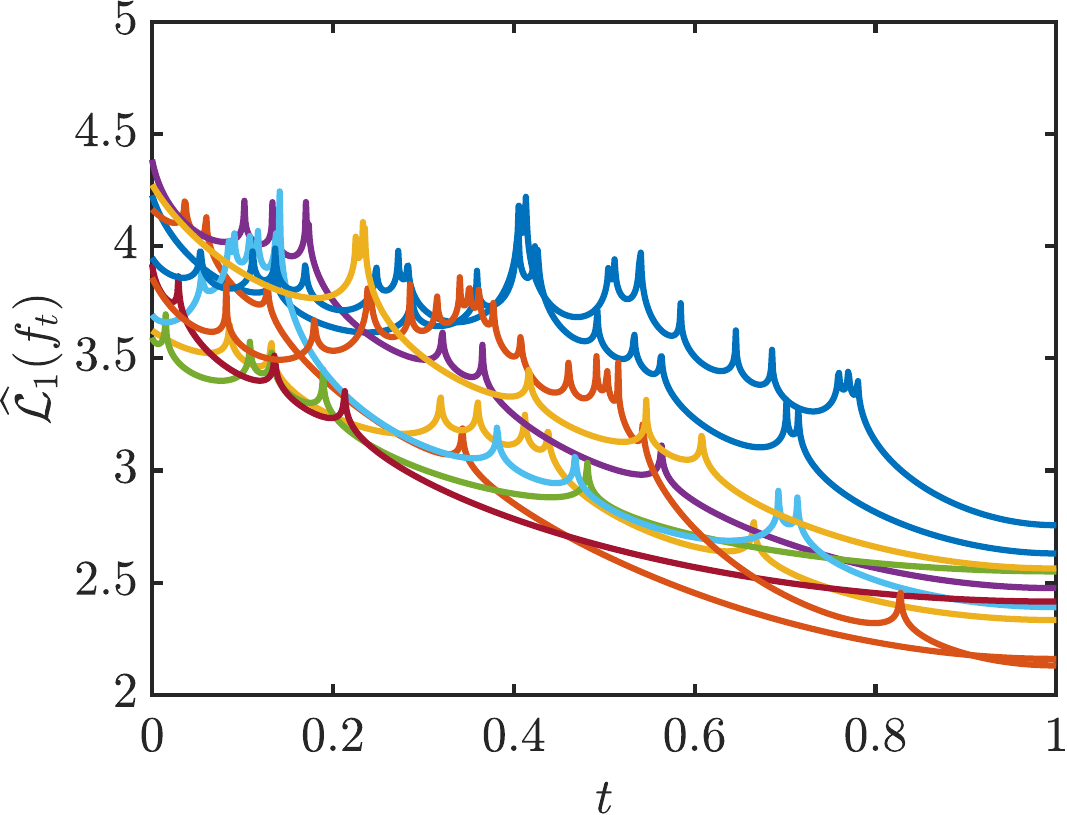}
\end{minipage}

\caption{Objective function $\widehat{\mathcal{L}}_{1} = \widehat{\mathcal{J}}_{1} \circ \mathcal{R}_{1}$ where the rectifier $\mathcal{R}_1$ is defined using the soft-plus $g$ \eqref{eq:def_g} (\emph{left}), the shifted ELU $g$ \eqref{eq:elu} (\emph{middle}), or the square function $g(\xi)=\xi^2$ (\emph{right}). Here, $\pi(x) = 1/2\mathcal{N}(x; -2,0.5) +1/2 \mathcal{N}(x; 2,2)$ is a univariate Gaussian mixture, and we use $n = 50$ to define $\mathcal{J}_1$ with $f_1$ represented using a linear combination of Hermite functions up to degree $10$. 
The objective is evaluated along line segments that interpolate between random initial maps ($t=0$) 
and critical points resulting from a gradient-based optimization method ($t=1$).
Observe that with bijective $g$ (\emph{left} and \emph{middle}) the algorithm always arrives at the same optimal value, whereas with the square function $g$ (\emph{right}) the algorithm gets stuck in local minima
and rarely attains the optimal value.
\label{fig:compare_obj}}
\end{figure}

\subsection{Smoothness of the optimization problem}
\label{subsec:choice_Vk}
In this section we give sufficient conditions on the function $g$ and the target density $\pi$ to guarantee that the objective function $\mathcal{L}_k$ is smooth over an appropriate function space for $f$. We introduce this function space, $V_k$, and show continuity properties of the rectifier that hold for functions $f \in V_k$, before stating our main result in Theorem~\ref{prop:finite_objective}. %

We begin by defining the weighted Sobolev space
\begin{align} \label{eq:hilbert_space}
 V_k  &= \left\{ f\colon\R^k\rightarrow\R \text{ such that } \|f\|_{V_k}^2 \coloneqq \int \Big(|f(\bx)|^2 + |\partial_k f(\bx)|^2 \Big) \eta_{\leq k}(\bx)\d\bx <+\infty \right\} ,
\end{align}
where $\eta_{\leq k}$ is the standard normal density on $\R^k$. By \cite[Theorem 1.11]{kufner1984define}, this space is complete and thus is a Hilbert space. The space $V_k$ has sufficient regularity for $\partial_k f$ to exist, but also to permit the pointwise evaluation $f(\bx_{<k},0)$, as required in the definition  \eqref{eq:bijective_transformation} of the rectifier. This property is formalized by the following proposition, a  trace theorem, which shows that for any $f\in V_k$, the function $\bx_{<k}\mapsto f(\bx_{<k},0)$ is a function in $L^2_{\eta_{<k}}$, the $\eta_{<k}$-weighted space of square integrable functions.

\begin{proposition} \label{thm:trace} There exists a constant $C_T<\infty$ such that for any $f \in V_k$ 
\begin{equation} \label{eq:trace_theorem}
   \int f(\bx_{<k},0)^2 \,\eta_{<k}(\bx)\d\bx \leq C_T \|f\|_{V_k}^2.
\end{equation}
\end{proposition}
\begin{proof}
 See Appendix \ref{proof:trace}.
\end{proof}

\begin{remark}
Notice that $H^1_{\eta_{\leq k}}  \subset V_k \subset L^2_{\eta_{\leq k}}$, where $L^2_{\eta_{\leq k}}=\{f\colon\R^k\rightarrow\R:\|f\|_{L^2_{\eta_{\leq k}}}^2\coloneqq \int f^2 \d\eta_{\leq k} <\infty\}$ and $H^1_{\eta_{\leq k}}=\{f\colon\R^k\rightarrow\R:\|f\|_{H^1_{\eta_{\leq k}}}^2\coloneqq \int f^2 + \|\nabla f\|^2 \d \eta_{\leq k}<\infty\}$ are the standard weighted Sobolev spaces. Given that  %
the standard normal density factorizes as $\eta_{\leq k}(\bx)=\prod_{i=1}^k \eta_i(x_i)$, the space $V_k$ admits the following tensor product structure
\begin{equation} \label{eq:tensorproduct_Vk}
 V_k = L^2_{\eta_1} \otimes\cdots\otimes L^2_{\eta_{k-1}} \otimes H^1_{\eta_k},
\end{equation}
and the norm $\|\cdot\|_{V_k}$ is a product norm.\footnote{That is $\|v_1\otimes\cdots\otimes v_k\|_{V_k} = \|v_1\|_{L^2_{\eta_1}}\|v_2\|_{L^2_{\eta_2}}\cdots \|v_{k-1}\|_{L^2_{\eta_{k-1}}} \|v_k\|_{H^1_{\eta_k}}$ for any $v_j\in L^2_{\eta_k}$ and $v_k\in H^1_{\eta_k}$.} This tensor product structure will be used later in Section~\ref{sec:ATM} to construct a numerical scheme for approximating the map components $S_k$.
\end{remark}

The following proposition shows that, under mild assumptions on $g$, the rectifier $\mathcal{R}_k$ is a Lipschitz continuous operator from $V_k$ to itself. The proof relies on the Hardy inequality \cite{muckenhoupt1972hardy} and on Proposition~\ref{thm:trace}.

\begin{proposition}\label{prop:RectifierIsLip}
 Let $g\colon\R\rightarrow\R_{>0}$ be a Lipschitz function, i.e., there exists a constant $L<\infty$ so that
 \begin{align}
  |g(\xi)-g(\xi')| &\leq L |\xi-\xi'|, \label{eq:lipschitz_g_bis}
 \end{align}
 holds for any $\xi,\xi'\in\mathbb{R}$.
 Then $\mathcal{R}_k(f)\in V_k$ for any $f\in V_k$, where $\mathcal{R}_k(f)$ is defined in \eqref{eq:bijective_transformation}.
 Furthermore there exists a constant $C<\infty$ such that
 \begin{equation}\label{eq:RectifierIsLip}
  \| \mathcal{R}_k(f_1)-\mathcal{R}_k(f_2) \|_{V_k} \leq C \| f_1-f_2 \|_{V_k},
 \end{equation}
 holds for any $f_1,f_2\in V_k$.
\end{proposition}
\begin{proof}
See~\Cref{proof:RectifierContinuity}. %
\end{proof} %

The next result relies on Proposition~\ref{prop:RectifierIsLip} to show that the objective functional $\mathcal{L}_{k}$ in \eqref{eq:minJoverV}, seen as a function from $V_k$ to $\R$, is well-defined, continuous, and differentiable.

\begin{theorem} \label{prop:finite_objective}
Let $\pi$ be a probability density function on $\R^d$ %
satisfying 
\begin{equation} \label{eq:AssumptionPiBounded}
    \pi(\bx) \leq C_\pi \eta(\bx),
\end{equation}
for all $\bx \in \mathbb{R}^{d}$ and a constant $C_\pi<\infty$, and let $g\colon\R\rightarrow\R_{\geq0}$ be an increasing function where %
 \begin{align}
  |g(\xi)-g(\xi')| &\leq L |\xi-\xi'|, \label{eq:lipschitz_g}\\
  |\log \circ \, g(\xi)-\log \circ \, g(\xi')| &\leq L |\xi-\xi'|, \label{eq:lipschitz_logg}
 \end{align}
holds for any $\xi,\xi'\in\mathbb{R}$ and a constant $ L <\infty$. %
 Then
$$
 \mathcal{L}_k(f) = \mathcal{J}_k ( \mathcal{R}_k(f) ) <\infty,
$$
for any $f\in V_k$, where $\mathcal{J}_k(s)=\int \left( \frac{1}{2}s(\bx_{\leq k})^2-\log\left|\partial_k s(\bx_{\leq k})\right | \right)\pi(\bx)\d\bx$ as in \eqref{eq:opt_map_component} and where $\mathcal{R}_k(f)$ is defined as in \eqref{eq:bijective_transformation}.
Moreover, the function $\mathcal{L}_k\colon V_k\rightarrow\R$ is continuous and differentiable at any $f\in V_k$ and its gradient $\nabla\mathcal{L}_k(f) \in V_k$ is given by 
\begin{align}
 \langle \nabla\mathcal{L}_k(f) , \varepsilon \rangle_{V_k} 
 &= \int \mathcal{R}_k(f)(\bx)\left(\varepsilon(\bx_{<k},0) + \int_0^{x_k} g'\big(\partial_k f(\bx_{<k},t)\big)\partial_k \varepsilon(\bx_{<k},t)\mathrm{d}t\right) \pi(\bx)\d\bx \nonumber\\
 & - \int(\log\circ\,  g)'(\partial_{k}f(\bx))\partial_{k}\varepsilon(\bx) ~\pi(\bx)\d\bx.
 \label{eq:gradL}
\end{align}
for any $\varepsilon\in V_k$, where $\langle \cdot, \cdot\rangle_{V_k} $ is the scalar product in $V_k$.
\end{theorem}
\begin{proof}
 See \Cref{proof:finite_objective}.
\end{proof}
It is useful to discuss some implications of Theorem \ref{prop:finite_objective}. A smooth objective function permits us to use deterministic first or second-order optimization algorithms. 
Furthermore, in Section~\ref{sec:ATM}, we exploit the gradients of $\mathcal{J}_k$ in \eqref{eq:gradL} to construct an adaptive polynomial (or wavelet) basis for $V_k$.
We also note that many functions $g$ satisfy the conditions~\eqref{eq:lipschitz_g} and~\eqref{eq:lipschitz_logg} in the proposition. For example, these include the soft-plus~\eqref{eq:def_g} and shifted ELU~\eqref{eq:elu} functions considered in Figure~\ref{fig:compare_obj}.

\begin{remark}[Assumption \eqref{eq:AssumptionPiBounded} implies Gaussian tails]\label{rmk:Guassian_tails}
The condition \eqref{eq:AssumptionPiBounded} implies that 
$\mathbb{P}_{\bX\sim\pi}(\|\bX\|_2 > t) \leq C_\pi \mathbb{P}_{\bZ\sim\eta}(\|\bZ\|_2 > t) 
$ holds for any $t\geq0$, 
and hence that $\pi$ has sub-Gaussian tails~\cite{vershynin2018high}. 
The reverse, however, does not hold in general. 
For instance, a compactly supported random variable with density $\pi(x) \propto 1/\sqrt{|x|}\mathbbm{1}_{\{x \in [-1,1]\}}$ is sub-Gaussian, but the density is unbounded at $x = 0$ and thus does not satisfy~\eqref{eq:AssumptionPiBounded}. 
\end{remark}

\rev{
\begin{remark}
We note that~\eqref{eq:AssumptionPiBounded} can be relaxed with the (less interpretable) assumption that there exist diagonal scalings $d_{k} > 0$ and a constant $C_\pi < \infty$ such that $\pi(\bx)\leq C_\pi \prod_{k=1}^{d} \eta_k(d_{k} x_k)$ for all $\bx \in \R^d$.
In practice, we can always %
standardize (i.e., center and rescale) %
the marginals of $\pi$ to match the mean and variance of the standard Gaussian reference. 
This preconditioning step should yield a smaller constant $C_\pi$ in \eqref{eq:AssumptionPiBounded}.
Thus, without loss of generality, we will use the assumption above for the remainder of this article.
\end{remark}
}

\begin{remark} 
Under additional assumptions on $g$, the gradient $f\mapsto \nabla \mathcal{L}_k(f)$ is locally Lipschitz from $\overline{V}_k$ to $V_k$, where $\overline{V}_k = \{f \in V_k, \partial_k f \in L^\infty\}$ is the space endowed with the norm $\| f \|_{\overline{V}_k} = \|f\|_{V_k} + \|\partial_k f \|_{L^\infty}$.
More specifically, if $g'$ and $(\log\circ\,  g)'$ are Lipschitz functions, then
there exists a constant $M < \infty$ such that
    \begin{equation} \label{eq:Lipschitz_gradient}
        \| \nabla \mathcal{L}_k(f_1) - \nabla \mathcal{L}_k(f_2) \|_{V_k} \leq M(1 + \| \rev{\Rectifier_k}(f_2)\|_{V_k}) \|f_1 - f_2 \|_{\overline{V}_k} ,
    \end{equation}
holds for any $f_1,f_2 \in \overline{V}_k$.
See the derivation of this result in Appendix~\ref{proof:Lipschitz_gradient}. Such local Lipschitz regularity is useful to analyze the convergence of backtracking gradient descent procedures, i.e., gradient descent with an inexact line search such as Armijo's rule, to a stationary point; see Theorem 2.1 in~\cite{truong2021backtracking} and Proposition 2.1.1 in~\cite{bertsekas1997nonlinear}. 
We leave the extension of such optimization guarantees for solving $\min_{f \in V_k} \mathcal{L}_k(f)$ to future work.
\end{remark} 

\subsection{Existence and uniqueness of solutions} \label{subsec:opt_properties}

In this section, we show that the optimization problem~\eqref{eq:minJoverV} does not admit any spurious local minima, meaning that local minimizers are in fact global minimizers. We also show that problem~\eqref{eq:minJoverV} admits a unique global minimizer which permits us to recover the KR rearrangement. To prove these results, we will need the following proposition, which provides conditions ensuring that the \emph{inverse} rectifier $\Rectifier_k^{-1}$ is continuous.

\begin{proposition}\label{prop:continuity_Rinv}
Let $g$ be a bijective function from $\R$ to $\R_{>0}$ such that for any $c>0$ there exists a constant $L_c <\infty$ so that
\begin{equation}\label{eq:ginvLip}
    |g^{-1}(\xi) - g^{-1}(\xi')| \leq L_c |\xi - \xi'|,
\end{equation}
holds for any $\xi,\xi' \geq c$. 
Then, for any $s \in V_k$ 
such that $\essinf \partial_k s > 0$, 
we have $\Rectifier_k^{-1}(s) \in V_k$ and $\essinf \partial_k \Rectifier_k^{-1}(s) > -\infty$.
Furthermore for any $c>0$, there exists a constant $C_{c} < \infty$ such that 
\begin{equation} \label{eq:Lipschitz_Rinv}
    \| \Rectifier_k^{-1}(s_1) - \Rectifier_k^{-1}(s_2) \|_{V_k} \leq C_{c} \| s_1 - s_2 \|_{V_k},
\end{equation}
holds for any $s_{1},s_{2} \in V_{k}$ such that $\essinf\partial_k s_i \geq c$.
\end{proposition}
\begin{proof}
See \Cref{proof:continuity_Rinv}.
\end{proof}

Note that the softplus function~\eqref{eq:def_g} and the shifted exponential linear unit~\eqref{eq:elu} satisfy~\eqref{eq:ginvLip} with $L_c = (1 - e^{-c})^{-1}$ and $L_c = \max\{1/c,1\}$, respectively.

\subsubsection{Local minima are global minima}
The following proposition shows that, under certain conditions on the function $g$, %
the image set $\mathcal{R}_k(V_k)=\{\mathcal{R}_k(f) \, : \, f\in V_k\}$ is convex.
\begin{proposition} \label{prop:Convexity_ImR}
Let $g\colon\R\rightarrow\R_{\geq0}$ be an increasing function such that $\xi\mapsto(g^{-1}(\xi))^2$ is a convex function.
Let $\Rectifier_k(f)$ be defined as in~\eqref{eq:bijective_transformation} and $V_k$ as in~\eqref{eq:hilbert_space}.
Then $\mathcal{R}_k(V_k)=\{\mathcal{R}_k(f) \, : \, f\in V_k\}$ is convex.
\end{proposition}
\begin{proof}
See \Cref{proof:ConvexityImageR}.
\end{proof}

When $g$ is the softplus function~\eqref{eq:def_g} or the exponential linear unit~\eqref{eq:elu}, the mapping $\xi\mapsto(g^{-1}(\xi))^2$ is indeed convex.
Interestingly, however, the exponential function $g(\xi) = \exp(\xi)$---which has been used to parameterize monotone maps in~\cite{marzouk2016sampling,spantini2018inference} and for monotone regression in \cite{ramsay1998estimating, shin2022joint}---does not satisfy this property, and there is no guarantee for $\Rectifier_k(V_k)$ to be convex in this case. %

\rev{An important consequence of Proposition~\ref{prop:Convexity_ImR} is that the constrained problem $\min_{s \in \Rectifier_k(V_k)} \mathcal{J}_k(s)$ remains convex.
Hence, the strict convexity of $\mathcal{J}_k$ (see Appendix \ref{proof:ConvexityOfJ}) yields the \emph{uniqueness} of any solution %
to $\min_{s \in \Rectifier_k(V_k)} \mathcal{J}_k(s)$.
The following theorem now establishes the {existence} of a unique global minimizer of the unconstrained problem~\eqref{eq:minJoverV}, under the assumption that a \emph{local} minimizer exists.}

\begin{theorem} \label{prop:globalMinima}
 Under the assumptions of Propositions \ref{prop:continuity_Rinv} and \ref{prop:Convexity_ImR}, let $f^*\in V_k$ be a local minimizer 
 of $f\mapsto\mathcal{J}_k(\Rectifier_k(f))$, meaning that there exists $\rho>0$ such that 
 \begin{equation}\label{eq:LocalMinima}
     \mathcal{J}_k(\Rectifier_k(f^*)) \leq \mathcal{J}_k(\Rectifier_k(f)),
  \quad \forall f\in V_k \text{ such that } \|f-f^*\|_{V_k}\leq \rho.
 \end{equation}
 If $\essinf \partial_k f^* > -\infty$, then $f^*$ is \red{the unique} global minimizer, meaning that 
 $$
  \red{\mathcal{J}_k(\Rectifier_k(f^*)) < \mathcal{J}_k(\Rectifier_k(f)),
  \quad \forall f\in V_k \setminus \{f^*\}.}
 $$
\end{theorem}

\begin{proof}
 Let $f\in V_k$ with $f\neq f^*$. For any $0<t<1$ we let $s_t = t \Rectifier_k(f) + (1-t) \Rectifier_k(f^*)$. 
 By Proposition \ref{prop:Convexity_ImR}, $\Rectifier_k(V_k)$ is convex so that $s_t\in\Rectifier_k(V_k)$. By \red{the strict} convexity of $\mathcal{J}_k$ (cf.\ Appendix \ref{proof:ConvexityOfJ}), we have $\mathcal{J}_k(s_t) < t  \mathcal{J}_k(\Rectifier_k(f))+ (1-t)\mathcal{J}_k(\Rectifier_k(f^*)) $,
 or equivalently
 \begin{equation}\label{eq:tmp3298}
     \mathcal{J}_k(s_t) - \mathcal{J}_k(\Rectifier_k(f^*))
  < t \Big( \mathcal{J}_k(\Rectifier_k(f))-\mathcal{J}_k(\Rectifier_k(f^*)) \Big).
 \end{equation}
 Next we show that there exists a sufficiently small $t>0$ such that the above left-hand side is non-negative. As a consequence, the right hand side of \eqref{eq:tmp3298} will be positive, which will conclude the proof.
 
 Let $b=\essinf \partial_k f^* $ and $c=g(b)/2>0$.
 For any $t\leq 1/2$, we have $\partial_k s_t = t g(\partial_k f)+ (1-t) g(\partial_k f^*) \geq 1/2 g(\partial_k f^*) $ so that
 $\essinf \partial_k s_t \geq c$. In addition, we have $\essinf \partial_k \Rectifier_k(f^*) = \essinf g(\partial_k f^*) \geq c$. Thus, by Proposition \ref{prop:continuity_Rinv}, there exists a constant $C_c<\infty$ such that 
 \begin{align*}
     \| \Rectifier^{-1}_k( s_t ) -  f^* \|_{V_k} 
     &=\| \Rectifier^{-1}_k( s_t ) -  \Rectifier^{-1}_k( \Rectifier_k(f^*) ) \|_{V_k} \\
     &\leq C_c \| s_t-\Rectifier_k(f^*) \|_{V_k}\\
     &=  t C_c  \| \Rectifier_k(f)-\Rectifier_k(f^*) \|_{V_k} .
 \end{align*}
 By letting $t = \rho / (C_c  \| \Rectifier_k(f)-\Rectifier_k(f^*) \|_{V_k})$, we have $\| \Rectifier^{-1}_k( s_t ) -  f^* \|_{V_k} \leq \rho $. 
 Therefore, setting $f = \Rectifier^{-1}_k( s_t )$ in~\eqref{eq:LocalMinima} %
 ensures that $0\leq \mathcal{J}_k(s_t) - \mathcal{J}_k(\Rectifier_k(f^*)) $.
\end{proof}

Let us remark that Theorem~\ref{prop:globalMinima} holds for general target densities $\pi$ without the bound assumed in Theorem~\ref{prop:finite_objective}. 
\rev{It establishes that any local minimizer of $\mathcal{J}_k\circ \Rectifier_k$ on $V_k$ is in fact the unique global minimizer.} %
Applying this result, however, depends on the existence of a local minimizer in the function space $V_k$, which must be established for a given target density. 
The following section studies a relatively broad class of distributions where the $k$-th component of the KR rearrangement is in $\Rectifier_k(V_k)$, thereby providing a candidate solution to apply Theorem~\ref{prop:globalMinima}.

\subsubsection{Recovery of the KR rearrangement} \label{subsubsec:uniqueKR}
As discussed in Section \ref{sec:background}, the Knothe--Rosenblatt rearrangement $S_{\KR}$ is the unique lower triangular and monotone map such that $\mathcal{D}_{\textrm{KL}}(\pi||S_{\KR}^\sharp\eta)=0$; see~\cite{bogachev2005triangular}.
The decomposition \eqref{eq:KL_decomposition} of the KL divergence $\mathcal{D}_{\textrm{KL}}(\pi||S^\sharp\eta)$ thus permits us to write
\begin{equation}\label{eq:Jinf}
    \mathcal{J}_k(S_{\KR,k}) \leq  \mathcal{J}_k( \Rectifier_k(f) ) ,
\end{equation}
for any $f\in V_k$ and for any $1\leq k \leq d$. 
Indeed, by letting $S$ be the map such that $S_k=\Rectifier_k(f)$ and $S_i=S_{\KR,i}$ for $i\neq k$, \eqref{eq:KL_decomposition} yields \eqref{eq:Jinf}.
Thus, if there exists a function $f_{\KR,k}\in V_k$ such that $S_{\KR,k} =\Rectifier_k(f_{\KR,k}) $, then $f_{\KR,k}$ is a global minimizer of $f\mapsto \mathcal{J}_k( \Rectifier_k(f) )$ over $f\in V_k$. 
To show this, we first need the following intermediate result.

\begin{proposition} \label{prop:KRmapProperties}
Let $\pi$ be a probability density function on $\R^d$ such that 
\begin{equation} \label{eq:margConditional_bothTailsBounded}
    c\eta(\bx) \leq \pi(\bx) \leq C\eta(\bx) ,
\end{equation}
for all $\bx \in \R^d$ with some constants $0< c \leq C < \infty$. %
Then, for all $\bx_{< k}\in\R^{k-1}$ and $k = 1,\dots,d$,  $S_{\KR,k}(\bx_{<k},x_k) = \mathcal{O}(x_k)$ and $\partial_k S_{\KR,k} (\bx_{<k},x_k) = \mathcal{O}(1)$ as $|x_k| \rightarrow \infty$. Furthermore, we have $S_{\KR,k} \in V_k$ 
and $\essinf \partial_k S_{\KR,k} > 0$ for all $k = 1,\dots, d$. 
\end{proposition}
\begin{proof}
    See Appendix~\ref{proof:KRrearrangementTail}.
\end{proof}

Let us comment on the assumption~\eqref{eq:margConditional_bothTailsBounded}.
\red{First, we note that it is equivalent to assuming $\pi(x)=\exp(\Psi(x))\eta(x)$ for some integrable function $\Psi\colon\R^d\rightarrow\R$ with bounded oscillation, meaning $\sup_{x\in\R^d}\Psi(x)-\inf_{x\in\R^d}\Psi(x)<\infty$. 
As an example, it is satisfied for any Gaussian mixture of the form $\pi(x)=\sum_{i=1}^k w_i \eta(x-x_i)$ with weights $w_i \geq 0$ that sum to one, centers $x_i\in\R^d$, and $\eta(x)\propto\exp(-\frac{\|x\|^2}{2})$; %
see Example 2.7 in~\cite{zahm2022certified} for more details.}

Second, while the results in Section~\ref{subsec:choice_Vk} assume only an upper bound on the joint density $\pi$, here we have both an upper and a lower bound. These bounds together imply that the $k$-th marginal conditional density, for $1 \leq k \leq d$, satisfies $(c/C) \eta_{k}(x_k) \leq \pi(x_k|\bx_{<k}) \leq (C/c)\eta_k(x_k)$ for all $\bx_{\leq k} \in \R^{k}$. This means that the marginal conditionals of the target density have Gaussian tails, rather than potentially being lighter than Gaussian (as in assumption~\eqref{eq:AssumptionPiBounded}). This condition guarantees that the asymptotic behavior of each map component $S_{\KR,k}$ in its last variable (as $|x_k| \to \infty$) is affine.

\medskip

We now combine our previous results to show that under the assumption~\eqref{eq:margConditional_bothTailsBounded} on the target density $\pi$, we have $f_{\KR,k} \in V_k$ and the existence of a unique solution to $\min_{f\in V_k}\mathcal{J}_k( \Rectifier_k(f) )$.  %

\begin{corollary} \label{cor:KR_recovery}
Under the assumptions on $\pi$ in Proposition~\ref{prop:KRmapProperties}, let $g$ be a Lipschitz bijection from $\R$ to $\R_{\geq0}$ such that $\xi\mapsto (g^{-1}(\xi))^2$ is convex and, for any $c>0$, there exists a constant $L_c<\infty$ such that $|g^{-1}(\xi) - g^{-1}(\xi')| \leq L_c |\xi - \xi'|$ for any $\xi,\xi'\geq c$.
Then, for any $1\leq k \leq d$, there exists a unique function $f_{\KR,k}\in V_{k}$ that satisfies $\essinf \partial_k f_{\KR,k} > -\infty$ such that $\Rectifier_k(f_{\KR,k})$ is the $k$-th component of the KR rearrangement $S_{\KR}$. As a consequence, $\min_{f\in V_{k}} \mathcal{J}_k( \Rectifier_k(f) )$ admits a unique solution and
$$
 \mathcal{J}_k( \Rectifier_k(f_{\KR,k}) ) = \min_{f\in V_{k}} \mathcal{J}_k( \Rectifier_k(f) ).
$$
\end{corollary}

\begin{proof}
 Combining Proposition~\ref{prop:KRmapProperties} and Proposition~\ref{prop:continuity_Rinv}, there exists a function $f_{\KR,k} \in V_k$ that satisfies $\essinf \partial_k f_{\KR,k} > -\infty$ such that  $\Rectifier_k(f_{\KR,k}) = S_{\KR,k}$. Then from Theorem~\ref{prop:globalMinima} and inequality~\eqref{eq:Jinf}, we have that $f_{\KR,k}$ is  \rev{the unique} global minimizer of $f\mapsto \mathcal{J}_k(\Rectifier_k(f))$ over $f\in V_k$.
\end{proof}

\rev{
Corollary \ref{cor:KR_recovery} gives sufficient conditions for the existence of a unique solution to $\argmin_{f\in V_k} \mathcal{J}_k(\mathcal{R}_k(f))$, and shows that this solution (after rectification) is in fact the $k$-th component of the KR rearrangement. 
It is natural to ask whether, outside of these assumptions, there still exists a local minimizer of $\mathcal{J}_k \circ \Rectifier_k$ over $V_k$. If so, then Theorem~\ref{prop:globalMinima} would guarantee that this function is in fact the unique global minimizer of the same objective.
Following the preceding analysis, one route towards answering this question would be to elucidate whether $\Rectifier_k^{-1}(S_{\KR,k}) \in V_k$ holds for tail assumptions on $\pi$ that are different from \eqref{eq:margConditional_bothTailsBounded}. In general, this is not clear. For targets with lighter tails, the KR map to a Gaussian $\eta$ will increase super-linearly as $|x_k| \to \infty$, but we hypothesize that many such maps are \emph{still} recoverable using $V_k$, due to its Gaussian weighting. We comment further on this possibility in Section~\ref{sec:conclusions}. For targets with heavier tails, however, we may have that $\partial_k S_{\KR, k}(x) \to 0$ as $x_k \to \pm \infty$, which implies that $\essinf \partial_k f_{\KR,k} = -\infty$. In this setting, a more promising route may be to modify the reference distribution, i.e., to choose a reference that itself has heavier tails, so that assumption~\eqref{eq:margConditional_bothTailsBounded} is satisfied. See \cite{jaini2020tails} for work in this direction. Doing so would break the strong log-concavity of $\eta$, however, which is an essential element of our analysis; hence it is unclear whether any optimization guarantees could then hold.} %

\rev{Alternatively, one could avoid appealing to the properties of the KR rearrangement for a given target $\pi$ and more directly ask whether a local minimizer of \eqref{eq:minJoverV} exists even if it is \emph{not} $f_{\KR, k}$. Our current theory does not address this question. One possible strategy involves elucidating under what additional assumptions (including for what class of targets $\pi$) the objective function $\mathcal{L}_k$ is coercive on $V_k$; a minimizer of~\eqref{eq:minJoverV} would therefore exist by the direct method in the calculus of variations. The resulting map would, by construction, be the closest approximation to $S_{\KR}$ (in the forward KL sense of~\eqref{eq:KL_decomposition}) within the space of triangular functions $\bigtimes_k \Rectifier_k(V_k)$. Proposition~\ref{prop:KLboundsL2norm}  ensures that such a map is also close to $S_{\KR}$ in the $L_\pi^2$ sense.} %

${}$

\section{Adaptive parameterization of transport maps}
\label{sec:ATM}
Given an i.i.d.\ sample $\{\bsx^i\}_{i=1}^{n} \sim \pi$, we now propose an adaptive algorithm to build $f \in V_k$ that minimizes the empirical objective $\widehat{\mathcal{L}}_{k} \coloneqq \widehat{\mathcal{J}}_k \, \circ \, \Rectifier_k(f)$, where $\widehat{\mathcal{J}}_{k}$ is as in~\eqref{eq:empirical_objective}. Doing so for each $k=1,\ldots,d$ yields a monotone triangular map, as described in Section~\ref{sec:background}.
For each $k$, we represent $f \in V_k$ with an $m$-term expansion
\begin{equation} \label{eq:linear_exp}
f(\bx_{\leq k}) = \sum_{\balpha\in\Lambda} c_{\balpha} \psi_{\balpha}(\bx_{\leq k}), 
\end{equation}
where $\Lambda$ is a set of multi-indices $\balpha = (\alpha_1,\dots,\alpha_k) \in \mathbb{N}_0^k$ with $\#\Lambda=m$, $c_{\balpha} \in \mathbb{R}$ are coefficients, and $\psi_{\balpha}\colon \mathbb{R}^{k} \rightarrow \mathbb{R}$ are basis functions for $V_k$, constructed as products of univariate functions, 
$$
 \psi_{\balpha}(\bx_{\leq k}) = \prod_{j=1}^{k} \psi_{\alpha_{j}}^{j}(x_{j}) .
$$
Here, $\{\psi_{\alpha}^{j}\}_{\alpha \in \mathbb{N}_0}$ is chosen to be a basis of  $L^2_{\eta_j}$ if $j<k$ and of $H^1_{\eta_k}$ if $j=k$.
Because $V_k$ possesses a tensor product structure \eqref{eq:tensorproduct_Vk} and because its norm $\|\cdot\|_{V_k}$ is a product norm, the basis $\{\psi_{\balpha}\}_{\balpha \in \mathbb{N}_0^k}$ is orthonormal if $\{\psi_{\alpha}^{j}\}_{\alpha \in \mathbb{N}_0}$ is an orthonormal basis for all $j\leq k$.
Since $f$ depends linearly on the coefficients $c_{\balpha}$,  %
the smoothness properties of the objective $\mathcal{L}_k$ 
transfer to the objective function parameterized by the coefficients $c_{\balpha}$, $\balpha\in\Lambda$. 

Sections~\ref{subsec:polynomials} and~\ref{subsec:wavelets} describe two different choices of basis: polynomials and wavelets, respectively. Then, in Section~\ref{subsec:choosing_basis}, we present a greedy algorithm for building the multi-index set $\Lambda$, applicable to any such hierarchical basis.
In addition, we propose a cross validation procedure to determine when to stop enriching $\Lambda$ in order to avoid over-fitting to the finite sample.

\subsection{Polynomial basis} \label{subsec:polynomials}

Probabilists' Hermite polynomials $\{\varphi_{\alpha}\}_{\alpha\in\mathbb{N}}$ form an orthogonal  basis for $L^2_\eta$ where %
$\eta$ is the univariate standard Gaussian density. $\varphi_{\alpha}$ is a polynomial of degree $\alpha\geq 0$ defined as
$$
 \varphi_\alpha(x) = (-1)^\alpha \exp(x^2/2) \frac{\d^\alpha}{\d x^\alpha} \exp(-x^2/2).
$$
Furthermore, we have $\langle \varphi_\alpha,\varphi_\beta \rangle_{L^2_{\eta}} = \alpha! \, \delta_{\alpha,\beta}$ so that $\{\varphi_{\alpha}/\sqrt{\alpha!}\}_{\alpha\in\mathbb{N}_0}$ forms an orthonormal basis for $L^2_\eta$. Similarly, $\langle \varphi_\alpha,\varphi_\beta \rangle_{H^1_{\eta}} = (\alpha+1)! \, \delta_{\alpha,\beta}$ so that $\{\varphi_{\alpha}/\sqrt{(\alpha+1)!}\}_{\alpha\in\mathbb{N}_0}$ forms a orthonormal %
basis for $H^1_{\eta}$; see~\cite[Proposition 1.3]{schmuland1992dirichlet}. We can thus tensorize these basis functions to obtain an orthonormal basis for the anisotropic space $V_k$ \eqref{eq:tensorproduct_Vk}.

In practice, we modify the univariate Hermite polynomials to be linear outside of an arbitrary compact domain. Following~\cite{parno2018transport}, we let
\begin{equation} \label{eq:linearixed_poly}
    \psi_{\alpha_j}(x_j) =
    \frac{1}{\sqrt{Z_{\alpha_j}}}\left\{ \begin{array}{ll}
    \varphi_{\alpha_j}(x_j^a) + \varphi_{\alpha_j}'(x_j^a)(x_j - x_j^a) & \text{ if }x_j < x_j^a \\
    \varphi_{\alpha_j}(x_j) &\text{ if } x_j^a \leq x \leq x_j^b \\ 
    \varphi_{\alpha_j}(x_j^b) + \varphi_{\alpha_j}'(x_j^b)(x_j - x_j^b) & \text{ if }x_j > x_j^b \\
    \end{array} \right.
\end{equation} 
where $Z_{\alpha_j} = \alpha_j!$ for $j = 1,\dots,k-1$ and $Z_{\alpha_j} = (\alpha_j+1)!$ for $j=k$. In our numerical experiments, we set $x_j^a$ and $x_j^b$ to be the $0.01$ and $0.99$ (empirical) quantiles, respectively, of the marginal sample $\{X_j^i\}_{i=1}^{n}$. The basis $\{\psi_{\alpha_j}\}_{\alpha_j \in \mathbb{N}_0}$, albeit not exactly orthonormal, is close to being orthonormal for sufficiently small $x_j^a$ and large $x_j^b$. This is important for the numerical stability of the algorithm presented in the following section.%

The basis $\{\psi_{\balpha}\}_{\balpha}$ generated by the piecewise polynomials \eqref{eq:linearixed_poly} has numerical and structural advantages over the %
Hermite polynomials.
First, these functions avoid the fast growth rate of standard polynomials as $\|\bx\| \rightarrow \infty$, providing a more stable extrapolation of the map to the tails of $\pi$, where there are typically no samples available to otherwise constrain its growth. 
Second, a map that reverts to linearity far from the origin %
yields densities %
in the class considered in Subsection~\ref{subsubsec:uniqueKR}; see condition~\eqref{eq:margConditional_bothTailsBounded}.
Indeed, each map component $S_k = \Rectifier_k(f)$ with $f\in\text{span}\{\psi_{\balpha}:\balpha\in\Lambda\}$ behaves linearly as $\vert x_k \vert \rightarrow \infty$, so that the pullback density $S^\sharp\eta$ has Gaussian tails like the reference $\eta$. %

\subsection{Wavelet basis} \label{subsec:wavelets}

Wavelets are a popular approach for approximating functions that are not uniformly smooth~\cite{cohen2003numerical, vidakovic2009statistical}. In particular, wavelet techniques define a multi-resolution approximation to a function that can better capture local features than, for instance, global polynomials. 
Given a compactly supported function $\psi\colon\R\rightarrow\R$  called the mother wavelet, and indices $(l,q)\in \mathbb{Z}^2$, the function
$$
 \psi_{(l,q)} : x\mapsto 2^{l/2}\psi(2^{l} x - q),
$$
is the $q$th wavelet at the level $l$.
Common choices for the mother wavelet include the Haar, the Daubechies, and the Meyer wavelets. 

In our numerical experiments, we use the continuous Mexican hat mother wavelet
\begin{equation} \label{eq:ricker_wavelet} 
    \psi(x) = \frac{2}{\sqrt{3\sigma}\pi^{1/4}} (1-\left(x/\sigma\right)^2)\exp(-x^2/(2\sigma^2)),
\end{equation}
with scale parameter $\sigma = 1$~\cite{mallat1999wavelet}, which is obtained from the second derivative of the univariate Gaussian density. The Mexican hat function is also commonly referred to as the Ricker wavelet in the geophysics community. We treat this function as having essentially compact support on $[-6,6]$, up to numerical precision. %

Tensorizing a univariate wavelet in the $k$th coordinate direction with modified %
Hermite polynomials in the first $k-1$ directions yields an expansion of the form in~\eqref{eq:linear_exp}, where the indices in $\balpha = (\alpha_1,\hdots,\alpha_k)$  %
are tuples of the form %
\begin{equation}\label{eq:ptnegor35}
    \alpha_{1},\hdots,\alpha_{k-1}\in\mathbb{N}_0,
    \quad \alpha_{k} = (l_k,q_k) \in \mathbb{Z}^2.
\end{equation}
In practice, we consider a truncated set of wavelets by
first rescaling the mother wavelet $\psi$ to have the same support as the interval between the 0.01 and 0.99 empirical quantiles of the marginal sample $\{X_j^i\}_{i=1}^n$. %
Then, we constrain $l_j \geq 0$ and consider $l_j = 0$ to be the coarsest level of the approximation. Next, for any $l_j$, the translation parameter $q_j$ is bounded as $0 \leq q_j \leq 2^{l_j} - 1$. %
Thus, we can consider %
$$
 \alpha_{1},\hdots,\alpha_{k-1}\in\mathbb{N}_0,
    \quad \alpha_{k} = (l_k,q_k) \in \mathbb{N}_0^2.
$$
rather than \eqref{eq:ptnegor35}.

Lastly, we note that any function $f$ of the form \eqref{eq:linear_exp} expanded with compactly supported wavelets will decay to $0$ as $|x_k| \rightarrow \infty$. Thus, the map $S_k = \Rectifier_k(f)$ will have asymptotic linear growth as $|x_k| \rightarrow \infty$, similarly to the modified %
Hermite polynomial expansions.

\subsection{Adaptive transport map algorithm}
\label{subsec:choosing_basis}

Now we propose a greedy method for constructing the multi-index set $\Lambda=\Lambda_t$ in \eqref{eq:linear_exp}. 
For simplicity, we consider only the case of single indices $\alpha_j\in\mathbb{N}_0$. The extension to the case $\alpha_j\in\mathbb{N}_0^2$ (as in the wavelet construction above) %
is described in Appendix~\ref{subsec:multi_index_wavelet}. 
At each greedy iteration $t$, we add one multi-index $\balpha_t^\ast$ to $\Lambda_{t}$. Starting with $\Lambda_{0} = \emptyset$ we let
\begin{equation}\label{eq:greedyLambda}
    \Lambda_{t+1} = \Lambda_{t} \cup \{\balpha_t^\ast\} ,
\end{equation}
where $\balpha_t^\ast \notin\Lambda_{t}$ is a multi-index that yields the \emph{best} improvement of $\widehat{\mathcal{L}}_k$ in a sense to be defined below.
Borrowing ideas from \cite{migliorati2015adaptive,migliorati2019adaptive}, we constrain the sets $\{\Lambda_{t}\}_{t\geq0}$ to be \emph{downward closed} \cite{chkifa2015breaking,cohen2018multivariate}, meaning that they satisfy the property
\begin{equation}\label{eq:downwardClosedLambdaK}
 \balpha \in \Lambda_t \text{ and } \balpha'\leq\balpha ~\Rightarrow~ \balpha' \in \Lambda_t,
\end{equation}
where $\balpha'\leq\balpha $ denotes $\alpha_i'\leq\alpha_i$ for all $1\leq i \leq k$. 
Intuitively, \eqref{eq:downwardClosedLambdaK} means that $\Lambda_t$ has a pyramidal shape containing no holes. Downward-closed sets are known to preserve good approximation properties and permit a tractable construction of $\Lambda_t$; see~\cite{cohen2018multivariate}.
Indeed, $\Lambda_{t+1}$ remains downward-closed if and only if the multi-index $\balpha_t^\ast$ is selected from the so-called \emph{reduced margin} of $\Lambda_{t}$, defined by
$$
 \Lambda_{t}^{\text{RM}}
 = \{ \balpha\notin \Lambda_t  \, : \, %
 \balpha-\mathbf{e}_i\in\Lambda_t \text{ for all } 1\leq i \leq k \text{ with } \alpha_i\neq0 \} ,
$$
where $\mathbf{e}_i$ denotes the $i$-th canonical vector of $\mathbb{N}^k$.
The reduced margin is a subset of the margin set $\Lambda_{t}^{\text{M}}$ (i.e., multi-indices $\balpha \not\in \Lambda_{t}$ such that $\exists \, i > 0$ where $\balpha - \mathbf{e}_{i} \in \Lambda_{t}$); see Figure~\ref{fig:multi_indices}. 
The size of the reduced margin typically grows more slowly with respect to the dimension $k$ than the margin itself. For instance, if $\Lambda_{t}$ contains all multi-indices in a hypercube of width $p$ in dimension $k$, the margin has cardinality $(p+1)^{k} - p^{k}$, while the reduced margin has cardinality $k$.

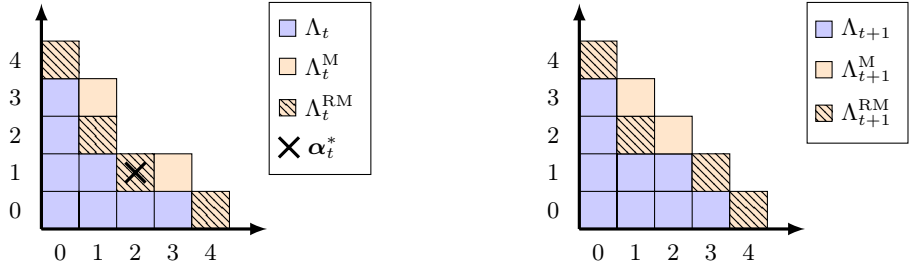
\begin{figure}
  \begin{center}
   \usetikzlibrary{patterns}
\newcommand{\midxscale}{0.5}

\begin{minipage}{0.45\linewidth}
\centering
\begin{tikzpicture}[scale=\midxscale]

\draw[fill=blue!20] (0,0)--(0,4)--(1,4)--(1,2)--(2,2)--(2,1)--(4,1)--(4,0)--cycle;
\draw[color=black] (0,0) grid (1,4) (1,0) grid (4,1);

\draw[preaction={fill, orange!20}, pattern=north west lines] (0,4)--(1,4)--(1,5)--(0,5)--cycle;
\draw[preaction={fill, orange!20}, pattern=north west lines] (1,2)--(1,3)--(2,3)--(2,2)--cycle;
\draw[preaction={fill, orange!20}, pattern=north west lines] (2,1)--(2,2)--(3,2)--(3,1)--cycle;
\draw[preaction={fill, orange!20}, pattern=north west lines] (4,0)--(4,1)--(5,1)--(5,0)--cycle;
\draw[preaction={fill, orange!20}] (1,3)--(1,4)--(2,4)--(2,3)--cycle;
\draw[preaction={fill, orange!20}] (3,1)--(3,2)--(4,2)--(4,1)--cycle;
\draw (2.5,1.5) node[cross] {};

\draw[very thick,-latex] (0,0) -- (6,0);
\draw[very thick,-latex] (0,0) -- (0,6);

\foreach \x in {0,1,...,4} { \node [anchor=north] at (\x+0.5,-0.2) {\x}; }
\foreach \y in {0,1,...,4} { \node [anchor=east] at (-0.3,\y+0.5) {\y}; }

\matrix [draw,below right] at (current bounding box.north east) {
  \node [fill=blue!20,draw,label=right:{\footnotesize $\Lambda_{t}$}] {}; \\
  \node [preaction={fill, orange!20},draw,label=right:{\footnotesize $\Lambda_{t}^{\text{M}}$}] {}; \\
  \node [preaction={fill, orange!20}, pattern=north west lines,draw,label=right:{\footnotesize $\Lambda_{t}^{\text{RM}}$}] {}; \\
  \node [cross,label=right:{\footnotesize $\balpha_{t}^{*}$}] {}; \\
};

\end{tikzpicture}
\end{minipage}
\hspace{-0.5cm}
\begin{minipage}{0.45\linewidth}
\centering
\begin{tikzpicture}[scale=\midxscale]

\draw[fill=blue!20] (0,0)--(0,4)--(1,4)--(1,2)--(2,2)--(3,2)--(3,1)--(4,1)--(4,0)--cycle;
\draw[color=black] (0,0) grid (1,4) (1,1) grid (3,2) (1,0) grid (4,1);

\draw[preaction={fill, orange!20}, pattern=north west lines] (0,4)--(1,4)--(1,5)--(0,5)--cycle;
\draw[preaction={fill, orange!20}, pattern=north west lines] (1,2)--(1,3)--(2,3)--(2,2)--cycle;
\draw[preaction={fill, orange!20}, pattern=north west lines] (4,0)--(4,1)--(5,1)--(5,0)--cycle;
\draw[preaction={fill, orange!20},pattern=north west lines] (3,1)--(3,2)--(4,2)--(4,1)--cycle;
\draw[preaction={fill, orange!20}] (1,3)--(1,4)--(2,4)--(2,3)--cycle;
\draw[preaction={fill, orange!20}] (2,2)--(2,3)--(3,3)--(3,2)--cycle;

\draw[very thick,-latex] (0,0) -- (6,0);
\draw[very thick,-latex] (0,0) -- (0,6);

\foreach \x in {0,1,...,4} { \node [anchor=north] at (\x+0.5,-0.2) {\x}; }
\foreach \y in {0,1,...,4} { \node [anchor=east] at (-0.3,\y+0.5) {\y}; }

\matrix [draw,below right] at (current bounding box.north east) {
  \node [fill=blue!20,draw,label=right:{\footnotesize $\Lambda_{t+1}$}] {}; \\
  \node [preaction={fill, orange!20},draw,label=right:{\footnotesize $\Lambda_{t+1}^{\text{M}}$}] {}; \\
  \node [preaction={fill, orange!20}, pattern=north west lines,draw,label=right:{\footnotesize $\Lambda_{t+1}^{\text{RM}}$}] {}; \\
};

\end{tikzpicture}   
\end{minipage}
  \end{center}
  \caption{A $k=2$ dimensional downward-closed active set of multi-indices $\Lambda_{t}$ with its margin $\Lambda_{t}^{\text{M}}$ and reduced margin $\Lambda_{t}^{\text{RM}}$. The margin and reduced margins are plotted before (\emph{left}) and after (\emph{right}) adding to $\balpha_t^\ast = (2,1)$, that is denoted with a cross, to $\Lambda_{t}$.~\label{fig:multi_indices}}
\end{figure}

Denoting by $f_t$ the minimizer of $f\mapsto \widehat{\mathcal{L}}_{k}(f)$ over $f\in\text{span}\{ \psi_{\balpha},\balpha\in\Lambda_t \}$, we select $\balpha_t^\ast$ in the reduced margin $\Lambda_{t}^{\text{RM}}$ with the following heuristic
\begin{equation} \label{eq:newAlpha}
\balpha_{t}^{*} \in \argmax_{\balpha \in \Lambda_{t}^{\text{RM}}} |\nabla_{ {\balpha}} \widehat{\mathcal{L}}_{k}(f_{t})|.
\end{equation}
Here, the notation $\nabla_{{\balpha}} \widehat{\mathcal{L}}_{k}(f_{t})$ denotes the derivative of $c_{\balpha}\mapsto \widehat{\mathcal{L}}_{k}(f_{t} + c_{\balpha} \psi_{\balpha})$ evaluated at $c_{\balpha}=0$.
In other words, we select the multi-index $\balpha_{t}^{*}$ by choosing the largest functional derivative of $\widehat{\mathcal{L}}_{k}$ at $f_{t}$ with respect to the basis functions not contained in the current expansion.
This greedy procedure for learning each map component is presented in detail in Algorithm~\ref{alg:learn-map}.

\begin{algorithm}[!htb]
   \caption{Estimate map component $S_{k}$}
   \label{alg:learn-map}
\begin{algorithmic}[1]
   \STATE {\bfseries Input}: training sample $\{\bsx_{1:k}^{i}\}_{i=1}^{n}$, maximum cardinality $m$ for $\Lambda_t$
   \STATE Initialize $\Lambda_{0} = \emptyset$, $f_{0} = 0$%
   \FOR{$t=0,\dots,m-1$}
   \STATE Construct the reduced margin: $\Lambda_{t}^{\text{RM}}$
   \STATE Select the new multi-index: $\balpha_{t}^{*} \in \argmax_{\balpha \in \Lambda_{t}^{\text{RM}}} |\nabla_{{\balpha}} \widehat{\mathcal{L}}_{k}(f_{t})|$
   \STATE Update the active set: $\Lambda_{t+1} = \Lambda_{t} \cup \{\balpha_{t}^{*}\}$
   \STATE Update the approximation: $f_{t+1} = \argmin_{f \in \text{span}\{\psi_{\balpha}:\balpha\in \Lambda_{t+1}\}} \widehat{\mathcal{L}}_{k}(f)$
   \ENDFOR
   \STATE {\bfseries Output}: $\widehat{S}_{k} = \mathcal{R}_{k}(f_{m})$
\end{algorithmic}
\end{algorithm}

\rev{We emphasize that this procedure achieves adaptivity by refining the approximation spaces for $(f_k)_{k=1}^d$, and seeking $f_k$ in these \textit{linear} spaces at each iteration of Algorithm~\ref{alg:learn-map}. 
This approach contrasts with normalizing flows that are typically formulated over \textit{nonlinear} spaces of deep neural networks; see~\cite{wehenkel2019} for an example of combining neural networks with a rectifier to represent monotone maps. 
Training the latter can be understood as simultaneously minimizing the KL objective \eqref{eq:KL_decomposition} and learning features relevant to the approximation of $f_k$, but to our knowledge optimization guarantees for such models are unavailable and much more difficult to achieve.} %

\begin{remark}
If the objective function $\widehat{\mathcal{L}}_k$ is the linear least-squares loss, then Algorithm~\ref{alg:learn-map} corresponds to orthogonal matching pursuit for sparse linear regression with normalized basis functions. An alternative heuristic for selecting $\balpha_t^\ast$ that requires second-order information from the objective is $\balpha_{t}^\ast \in \argmax_{\balpha \in \Lambda_t^{\text{RM}}} |\nabla_{\balpha} \widehat{\mathcal{L}}(f_t)|^2/|\nabla^2_{\balpha} \widehat{\mathcal{L}}(f_t)|$. We found this criterion to perform similarly to~\eqref{eq:newAlpha} for the polynomial basis, but it led to faster convergence of the objective when working with the wavelet basis.
\end{remark}

\begin{remark}
A potential drawback of Algorithm~\ref{alg:learn-map} is that the greedy enrichment procedure cannot ``see'' behind the reduced margin. For instance, if a relevant multi-index is located far beyond the reduced margin and the gradient $\nabla_{{\balpha}} \widehat{\mathcal{L}}_{k}(f_{t})$ vanishes for all $\balpha\in\Lambda_t^{RM}$, then the algorithm will be stuck. \cite{migliorati2019adaptive} suggested a safeguard mechanism to avoid this behaviour: arbitrarily activate the most ancient index from the reduced margin every $t_{\textrm{sg}} \in \mathbb{N}$ iterations. This modification, however, was not needed in our numerical experiments.
\end{remark}

\begin{remark}
As pointed out in \cite{cohen2018multivariate,migliorati2015adaptive,migliorati2019adaptive}, adding multiple multi-indices at each greedy iteration could also yield better performance compared to adding only one multi-index at a time. 
The so-called ``bulk-chasing'' procedure identifies a subset $\lambda_t^\ast \subset \Lambda_{t}^{\text{RM}}$ of multi-indices that capture a fraction of the $L_2$-norm of the gradient $\nabla_{\balpha} \widehat{\mathcal{L}}_k(f_t)$ along the reduced margin. %
\end{remark}

To determine the maximal cardinality of $\Lambda_t$ (i.e., when to stop adaptation) we use $\nu$-fold cross-validation as in~\cite{wasserman2013all,bigoni2021nonlinear}. For each fold, we run Algorithm~\ref{alg:learn-map} with $\nu-1$ folds of the training data for $m = n \frac{\nu-1}{\nu}$ iterations, and evaluate the objective function in~\eqref{eq:empirical_objective} on the held-out test set. We then
select the number of terms $m^\ast \leq n$ in the expansion \eqref{eq:linear_exp} that minimizes the test error averaged over the $\nu$ folds. The full sample is then used to run  Algorithm~\ref{alg:learn-map} for $m^\ast$ iterations.
The complete procedure for learning each map component $S_k$, with cross-validation, is called the Adaptive Transport Map (\ALG) algorithm. This procedure is detailed in Algorithm~\ref{alg:learn-map-withcv}. In practice, we also stop the training on each fold early if the log-likelihood of the test set does not continue to increase for more than $20$ iterations.

The cross-validation procedure produces an approximation for $f$ with an expansion whose cardinality is \emph{adapted} to the size of the training sample $n$. With a larger sample, we observe that the cross-validation procedure reliably adds more parameters, when needed, to reduce the bias in the approximation while controlling the variance of the estimated map parameters. Thus, we consider \ALG a \emph{semi-parametric} approach for approximating monotone triangular transport maps. Adaptation of the map complexity to the sample size $n$ will be demonstrated in the numerical results to follow.

\begin{algorithm}[!htb]
   \caption{\ALG algorithm for learning map component $S_k$} %
   \label{alg:learn-map-withcv}
\begin{algorithmic}[1]
   \STATE {\bfseries Input}: Training sample $\bchi = \{\bsx_{1:k}^{i}\}_{i=1}^{n}$, number of folds $\nu$
   \STATE Partition data into $\nu$ folds of equal size
   \FOR{$j = 1,\dots,\nu$}
   \STATE Partition $\bchi$: $\bchi^{\text{test}}_{j}$ is the $j$th subset of $\bchi$ and $\bchi^{\text{train}}_{j} = \bchi \setminus \bchi^{\text{test}}_{j}$
   \STATE Estimate $f_{t}$ using Algorithm~\ref{alg:learn-map} for $m = |\bchi^{\text{train}}_{j}|$ iterations with sample $\bchi^{\text{train}}_{j}$ 
   \STATE Store iterates $\widehat{S}_{k}^{j,1} \coloneqq \Rectifier_k(f_{1}),\dots,\widehat{S}_{k}^{j,m} \coloneqq \Rectifier_k(f_{m})$
   \STATE Evaluate log-likelihood $\log(\widehat{S}^{j,t}_{k})^\sharp\eta$ for iterations $t=1,\dots,m$ with sample $\bchi^{\text{test}}_{j}$ 
   \ENDFOR
   \STATE Define $m^\ast \in \{1, \ldots, m\}$ as the minimizer of negative log-likelihood $t \mapsto \sum_{j=1}^{\nu} -\log(\widehat{S}^{j,t}_{k})^\sharp\eta(\bchi^{\text{test}}_{j})$
   \STATE Estimate map $f^\ast$ using Algorithm~\ref{alg:learn-map} for $m^\ast$ iterations with  sample $\bchi$
   \STATE {\bfseries Output}: $\widehat{S}_{k} = \mathcal{R}_{k}(f^\ast)$
\end{algorithmic}
\end{algorithm}

\section{Numerical experiments} \label{sec:experiments}
In this section, we evaluate the performance of the \ALG algorithm for learning monotone triangular transport maps associated with a variety of target distributions. %
Sections~\ref{subsec:onedimensional} and \ref{subsec:twodim_data} evaluate and visualize the approximation power of the proposed method for strongly non-Gaussian targets.
Sections~\ref{subsec:adaptivity} illustrates the benefits of adaptivity over a range of sample sizes, and
Section~\ref{subsec:stoc_volatility} demonstrates how the estimated maps reveal and exploit conditional independence structure in the target distribution. %
Section \ref{sec:tabular} presents joint and conditional density estimation results for a suite of UCI datasets. %
Code for reproducing all of these numerical results is available online.\footnote{\url{https://github.com/baptistar/ATM}}

In all of the numerical experiments, we pre-process the data by standardizing each marginal, i.e., subtracting the empirical mean and dividing each component of $\bsx$ by its empirical standard deviation. %
Within the rectifier $\Rectifier_k$, we employ the modified soft-plus function $g(\xi) = \log(1 + 2^\xi)/\log(2)$, which satisfies the conditions of Propositions~\ref{prop:finite_objective}, \ref{prop:continuity_Rinv}, and \ref{prop:Convexity_ImR}. This function also satisfies $g(0) = 1$, so that $f(\bx_{1:k}) = 0$ is transformed into $\Rectifier_k(f)(\bx_{1:k}) = x_k$.
We evaluate the integral in~\eqref{eq:bijective_transformation} numerically %
using an adaptive quadrature method based on Clenshaw-Curtis points with a relative error of $10^{-3}$. At each iteration of the \ALG algorithm, we optimize over the coefficients $c_{\balpha}$ for $\balpha \in \Lambda_{t}$ using a BFGS (quasi-Newton) method~\cite{nocedal2006}.

\subsection{One-dimensional examples} \label{subsec:onedimensional}

We first consider a one-dimensional mixture of Gaussians with  density $\pi(x) = 0.5\mathcal{N}(x;-2,0.5) + 0.5\mathcal{N}(x;2,2)$. To estimate $\pi$, we generate $n = 10^4$ training samples $\{\sx^i\}_{i=1}^{n}$ and use them to estimate the map $S_{\textrm{KR}}$ that pushes forward $\pi$ to $\eta$. In one dimension, the KR rearrangement (which coincides with the optimal transport map~\cite{santambrogio2015optimal}) is simply the increasing function $S_{\textrm{KR}} = F_\eta^{-1} \circ F_\pi$, where $F_\eta$ and $F_\pi$ denote the cumulative distribution functions of the reference and target distributions, respectively. We can thus compare our estimated maps to this exact expression. 

We approximate $S$ as the transformation $\Rectifier(f)$ of a smooth non-monotone function $f$, seeking $f$ in the finite-dimensional space $V_1^p \subset V_{1}$ spanned by polynomials of degree at most $p$. 
Figure~\ref{fig:mog-PDFapprox} plots the estimated pullback densities $\widehat\pi = \widehat{S}^\sharp\eta$ 
for increasing polynomial degree $p$, or equivalently an increasing number of terms $\#\Lambda$ in $f$ \eqref{eq:linear_exp}. Figure~\ref{fig:mog-KL} shows convergence of this
approximation in KL divergence, again for increasing $p$. We also observe convergence of the estimated map $\widehat{S}$ towards the true KR rearrangement in the $L^2_\pi$ norm, as expected from the upper bound in~\eqref{eq:KLboundsL2norm}. %
Figure~\ref{fig:mogA_map_approximation} illustrates the approximation to the map $S$ and the associated non-monotone function $f$ for different choices of basis $\{\psi_\alpha\}_\alpha$. Here we observe that approximating $f$ using the modified Hermite polynomials \eqref{eq:linearixed_poly} very closely tracks the KR rearrangement $S_{\KR}$ and the function $f_{\KR} = \Rectifier^{-1}(S_{\KR})$ as $x \to \pm \infty$, as compared to standard Hermite polynomials or Hermite functions~\cite{boyd1984asymptotic}. 
Enforcing linear asymptotic behavior %
of the transport map ensures that the approximate distribution has Gaussian tails in regions where there are few samples.

\begin{figure}[!ht]
	\centering
	\begin{subfigure}[b]{0.45\linewidth}
	    \centering
		\includegraphics[width=\textwidth]{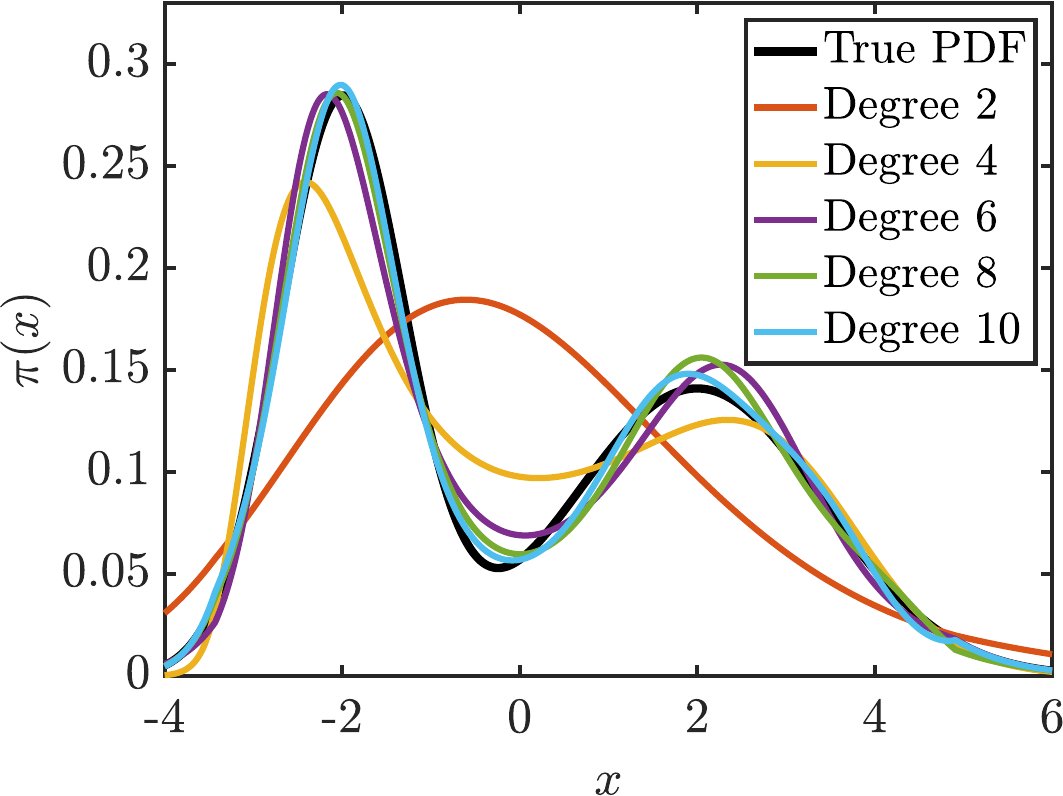}
		\caption{\label{fig:mog-PDFapprox}}
		\vspace{-0.2cm}
	\end{subfigure}
	\hspace{1cm}
	\begin{subfigure}[b]{0.45\linewidth}
	    \centering
		\includegraphics[width=0.96\textwidth]{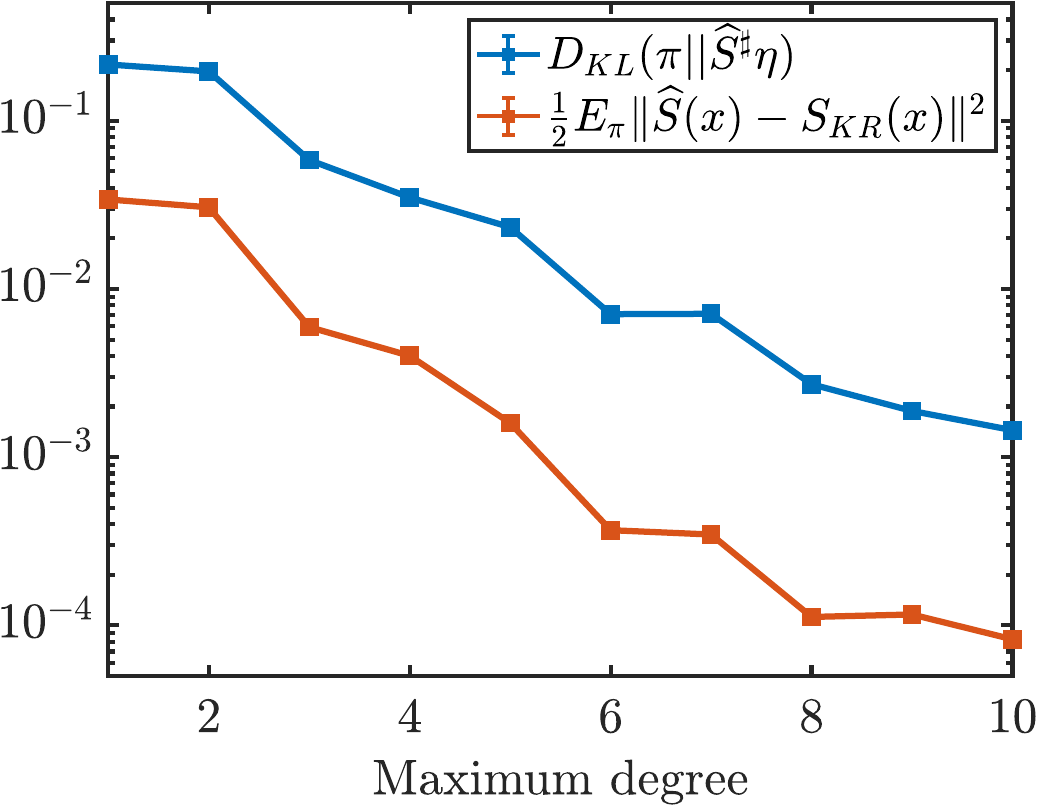}
		\caption{\label{fig:mog-KL}}
		\vspace{-0.2cm}
	\end{subfigure}
	\caption{(a) The pullback density $\widehat{S}^\sharp\eta$ approaches the target Gaussian mixture density $\pi$ %
	when increasing the maximum polynomial degree $p$ of the space $V_1^p \subset V_1$. (b) With increasing $p$, the pullback density converges to $\pi$ in KL divergence, and the estimated map converges to $S_{\textrm{KR}}$ in $L^2_\pi$.
	\label{fig:mog_bimodal}}
\end{figure}

\begin{figure}[!ht]
	\centering
	\begin{subfigure}{0.45\linewidth}
	    \centering
		\includegraphics[width=\textwidth]{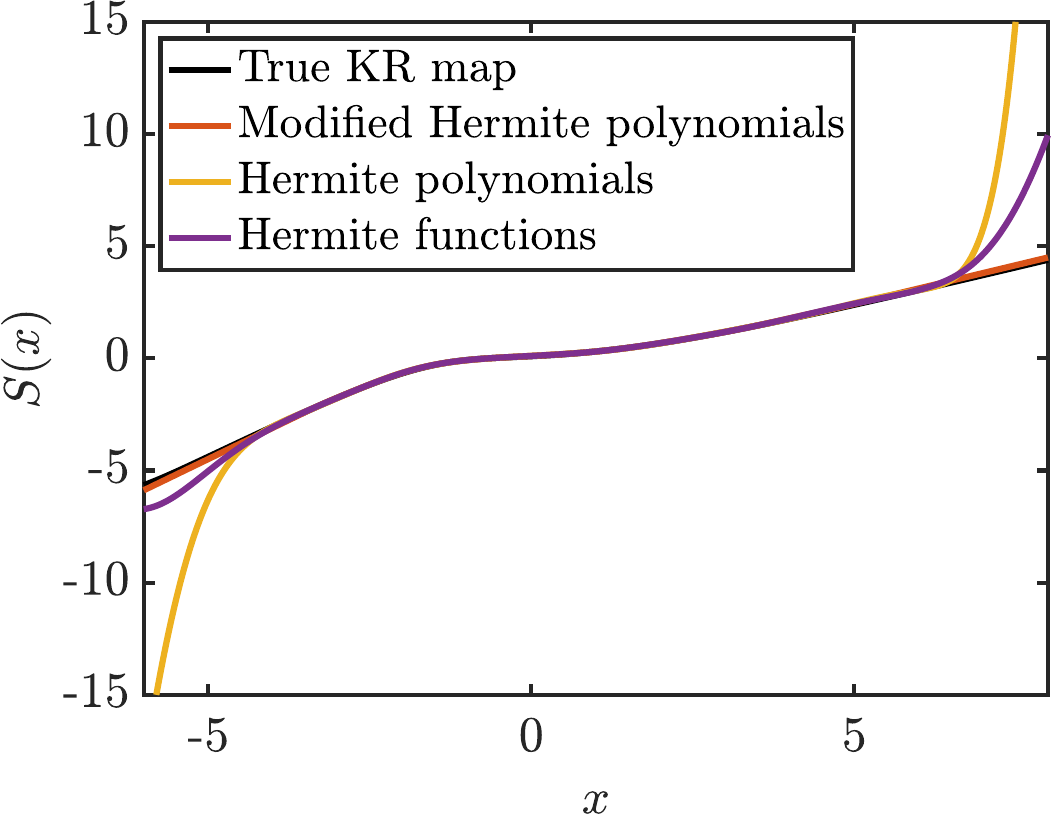}
		\caption{}
		\vspace{-0.2cm}
	\end{subfigure}
	\hspace{1cm}
	\begin{subfigure}{0.45\linewidth}
	    \centering
		\includegraphics[width=\textwidth]{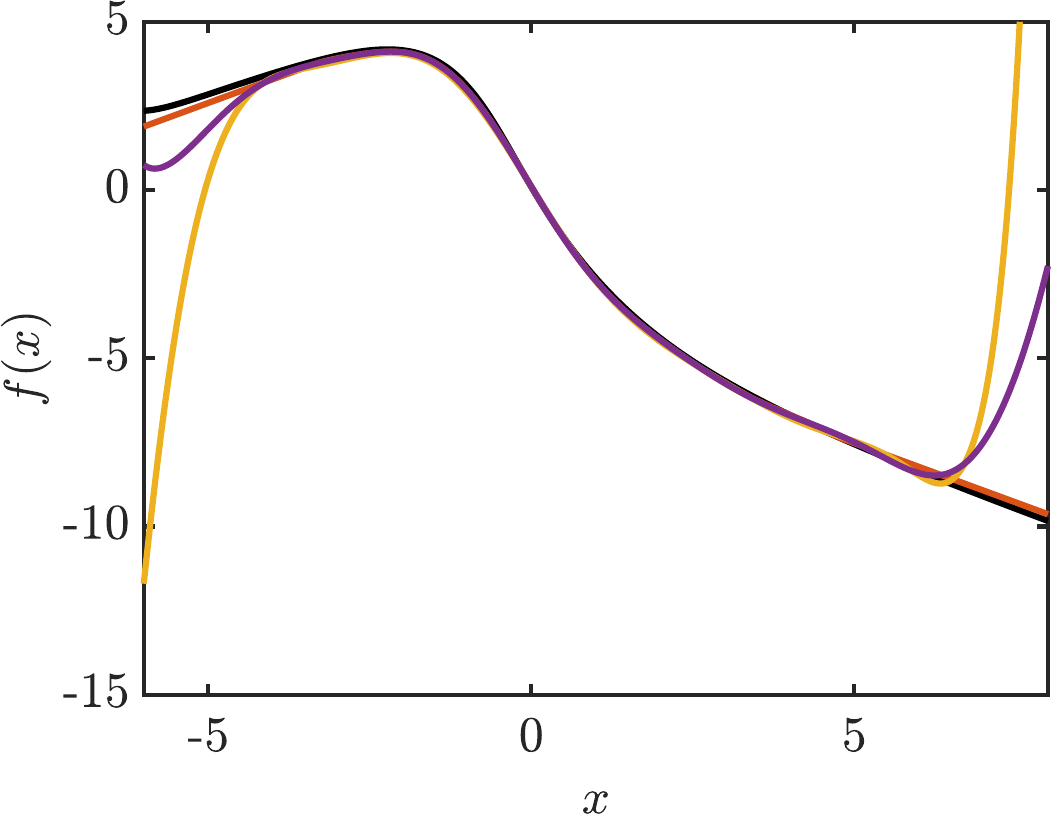}
		\caption{}
		\vspace{-0.2cm}
	\end{subfigure}
	\caption{(a) The approximate transport maps $\widehat{S}$ compared to $S_{\KR}$ (black), and (b) the corresponding non-monotone functions $f$ compared to $f_{\KR} \coloneqq \Rectifier^{-1}(S_{\KR})$. Both subfigures illustrate different choices of basis $\{\psi_\alpha\}_\alpha$, for the Gaussian mixture target of Figure~\ref{fig:mog_bimodal}. The modified Hermite polynomial basis provides the closest approximation to $S_{\KR}$ and $f_{\KR}$. \label{fig:mogA_map_approximation}}
\end{figure}

Next we consider a Gaussian mixture with density $\pi(x) = 0.5\mathcal{N}(x;0,1) + 0.5\mathcal{N}(x;0,0.025)$.%
This mixture is a common test case for sampling and density estimation methods, as it evaluates their ability to capture densities with multiple scales~\cite{sisson2007sequential}. Given $n = 10^4$ samples, we compute the approximate map using either modified Hermite polynomials (Section~\ref{subsec:polynomials}) or Mexican hat wavelets (Section~\ref{subsec:wavelets}). Figure~\ref{fig:mog_multiple_scales}(a) plots approximations to the target density using an $m = 15$ term expansion \eqref{eq:linear_exp} for $f$, while Figure~\ref{fig:mog_multiple_scales}(b) shows convergence in KL divergence for both the polynomial and wavelet bases. We observe that the polynomial approximation suffers from oscillations, while the wavelets better capture localized features. This results in a much faster convergence of the KL divergence for wavelets than for polynomials.

\begin{figure}[!ht]
	\centering
	\begin{subfigure}{0.45\linewidth}
	    \centering
		\includegraphics[width=\textwidth]{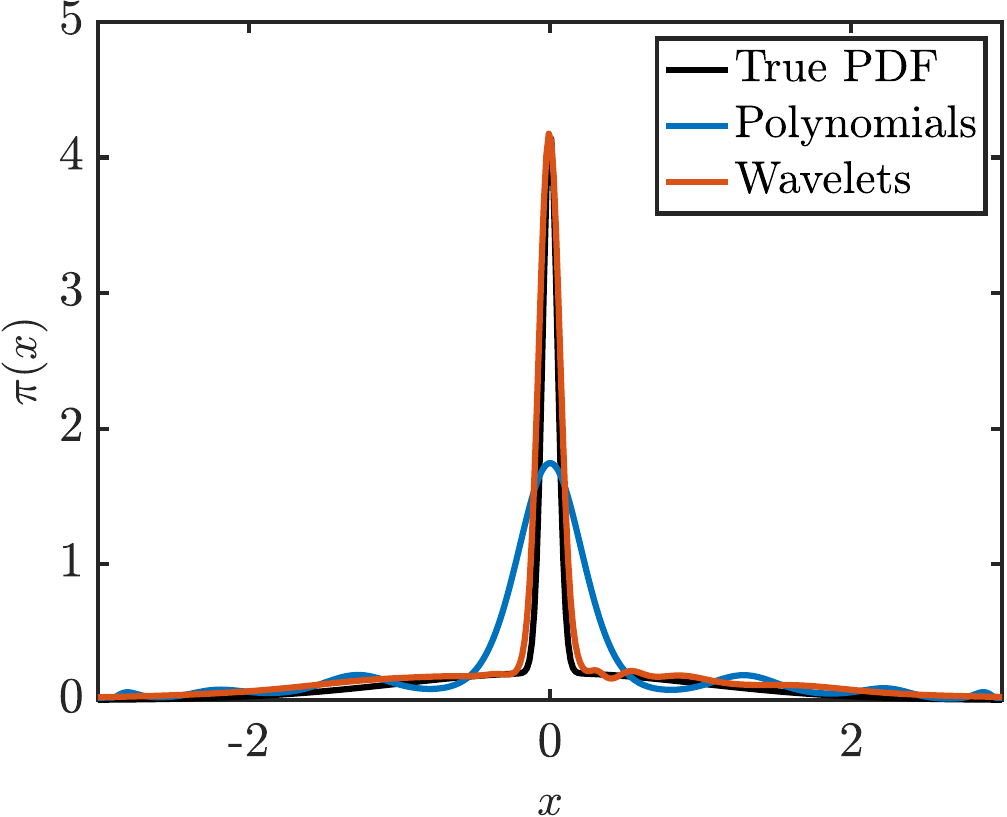}
		\caption{}
		\vspace{-0.2cm}
	\end{subfigure}
	\hspace{1cm}
	\begin{subfigure}{0.45\linewidth}
	    \centering
		\includegraphics[width=\textwidth]{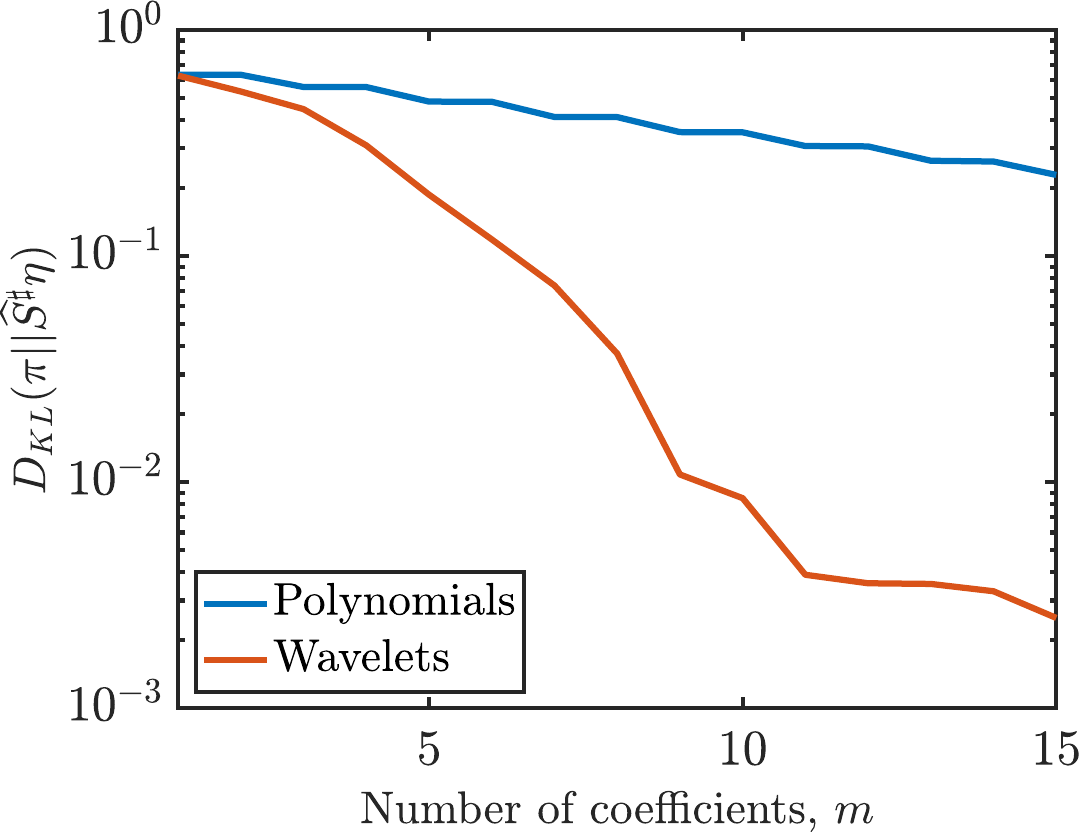}
		\caption{}
		\vspace{-0.2cm}
	\end{subfigure}
	\caption{Approximation of a scale mixture of Gaussians, using Hermite polynomials or Ricker wavelet expansions. (a) A wavelet expansion with $m = 15$ coefficients approximates the target density better than a polynomial expansion of the same size, and (b) yields significantly lower KL divergence for this example. 	\label{fig:mog_multiple_scales}}
\end{figure}

\subsection{Two-dimensional datasets} \label{subsec:twodim_data}

To demonstrate the expressiveness of the modified %
Hermite polynomial basis \eqref{eq:linearixed_poly}, we use the \ALG algorithm to approximate the Knothe--Rosenblatt rearrangement for several two-dimensional distributions with strongly non-Gaussian geometries: the ``banana,'' ``funnel,'' ``cosine,'' ``mixture of Gaussians (MoG),'' and ``ring'' distributions $\pi$ considered in~\cite{wenliang2019, jaini2019}. For each distribution, we generate an i.i.d.\ sample of size $n= 10^{4}$ from $\pi$ and apply a random rotation to the data using an angle uniformly distributed in $[0,\pi/2]$. %
Figure~\ref{fig:toy_problems} plots the true densities $\pi$ and the approximate densities $\widehat{\pi} \coloneqq \widehat{S}^{\sharp}\eta$ found using \ALG. We use $5$-fold cross-validation to identify the optimal number of elements in the multi-index set for each map component: $S_1$ and $S_2$. Figure~\ref{fig:toy_problems} indicates the total number of coefficients $\#\Lambda_{m^\ast}$ across both map components, underscoring the parsimony of the approximation. 

\begin{figure}[!ht]
	\captionsetup[subfigure]{labelformat=empty}
	\centering
	\subcaptionbox{}{\vspace{-10pt}\rotatebox[origin=t]{90}{Truth}}
	\begin{subfigure}{0.15\linewidth}
		\caption{Banana}
		\vskip -5pt
		\includegraphics[width=\textwidth]{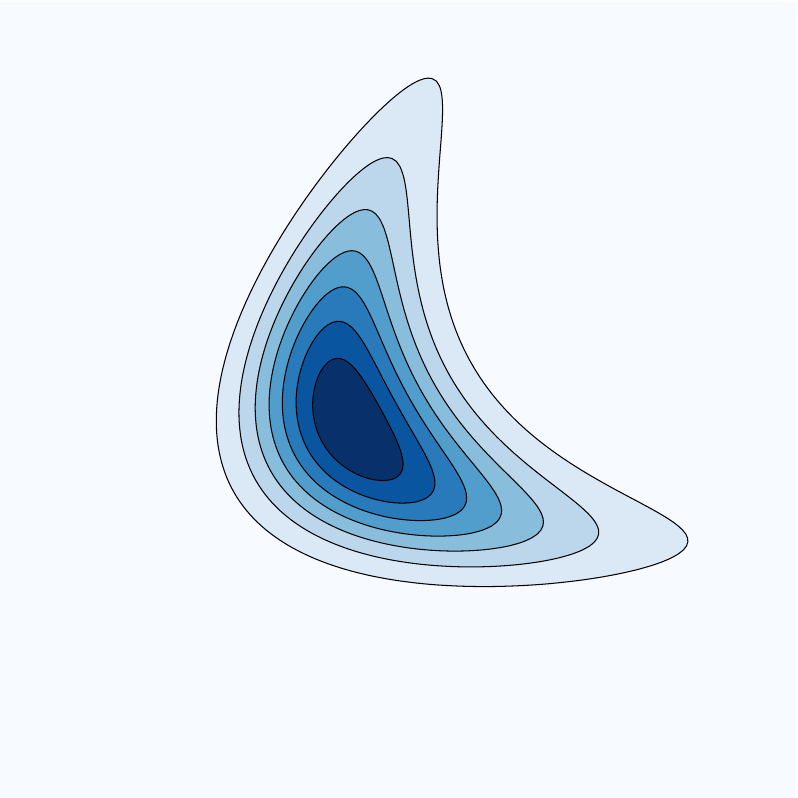}
	\end{subfigure}
	\begin{subfigure}{0.15\linewidth}
		\caption{Funnel}
		\vskip -5pt
		\includegraphics[width=\textwidth]{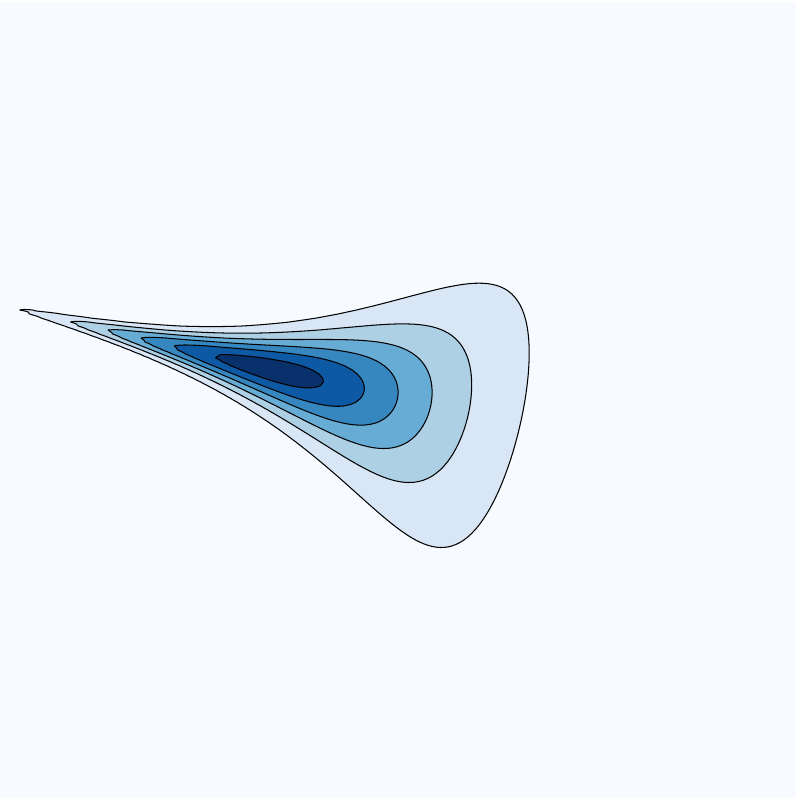}
	\end{subfigure}
	\begin{subfigure}{0.15\linewidth}
		\caption{Cosine}
		\vskip -5pt
		\includegraphics[width=\textwidth]{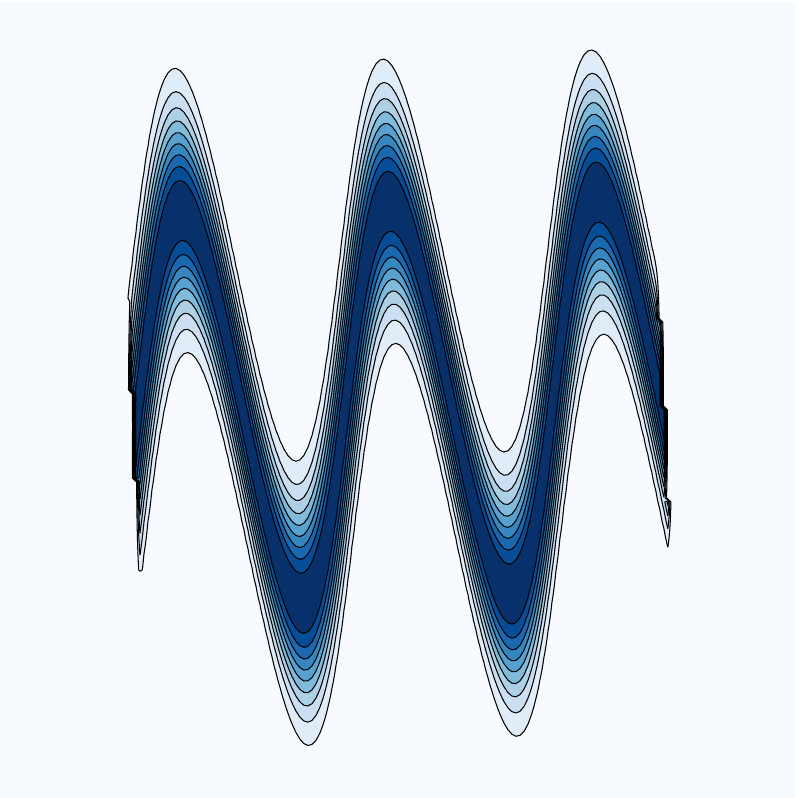}
	\end{subfigure}
	\begin{subfigure}{0.15\linewidth}
		\caption{MoG}
		\vskip -5pt
		\includegraphics[width=\textwidth]{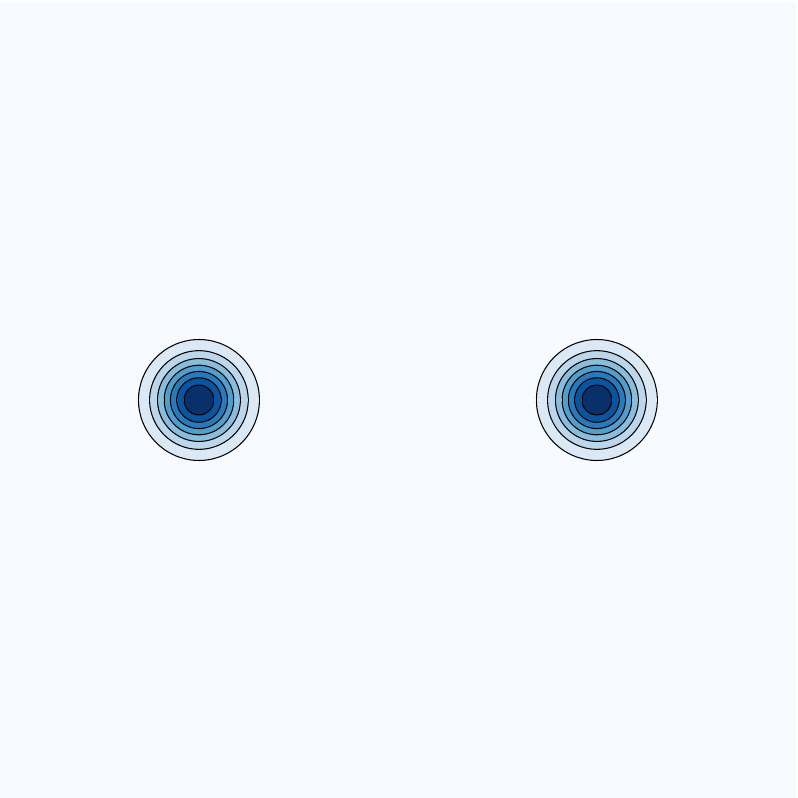}
	\end{subfigure}
	\begin{subfigure}{0.15\linewidth}
		\caption{Ring}
		\vskip -5pt
		\includegraphics[width=\textwidth]{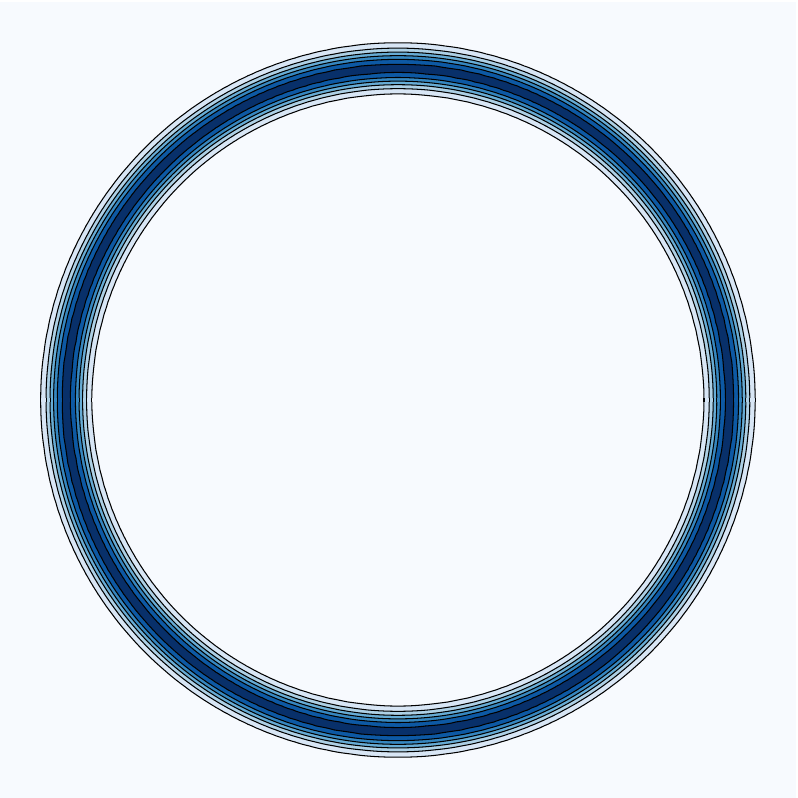}
	\end{subfigure}\\
	\subcaptionbox{}{\vspace{-5pt}\rotatebox[origin=t]{90}{\ALG}}
	\begin{subfigure}{0.15\linewidth}
	    \includegraphics[width=\textwidth]{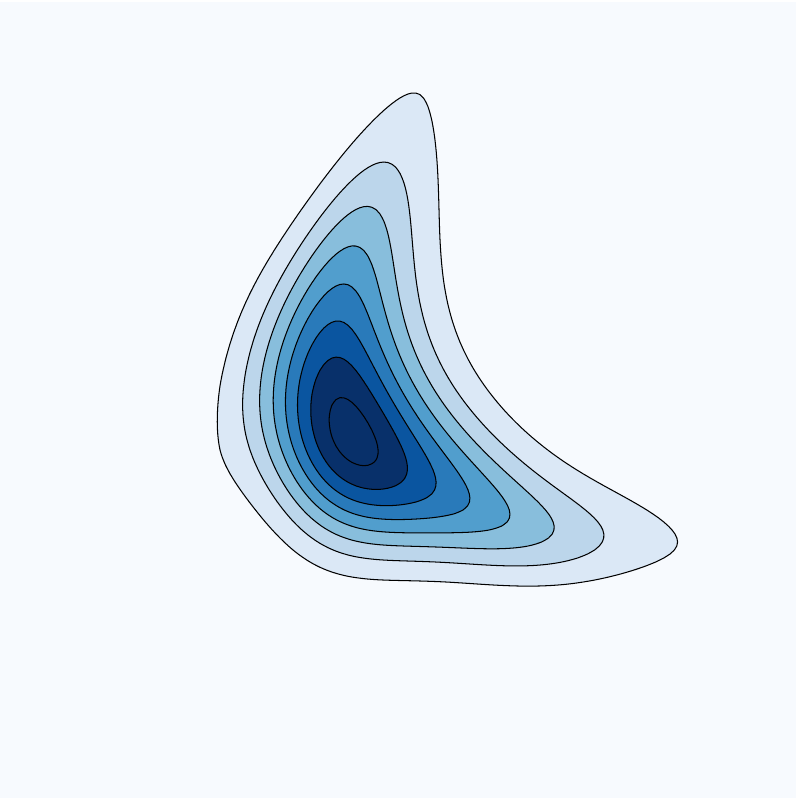}
		\caption{$\#\Lambda_{m^\ast} = 44$}
	\end{subfigure}
	\begin{subfigure}{0.15\linewidth}
	    \includegraphics[width=\textwidth]{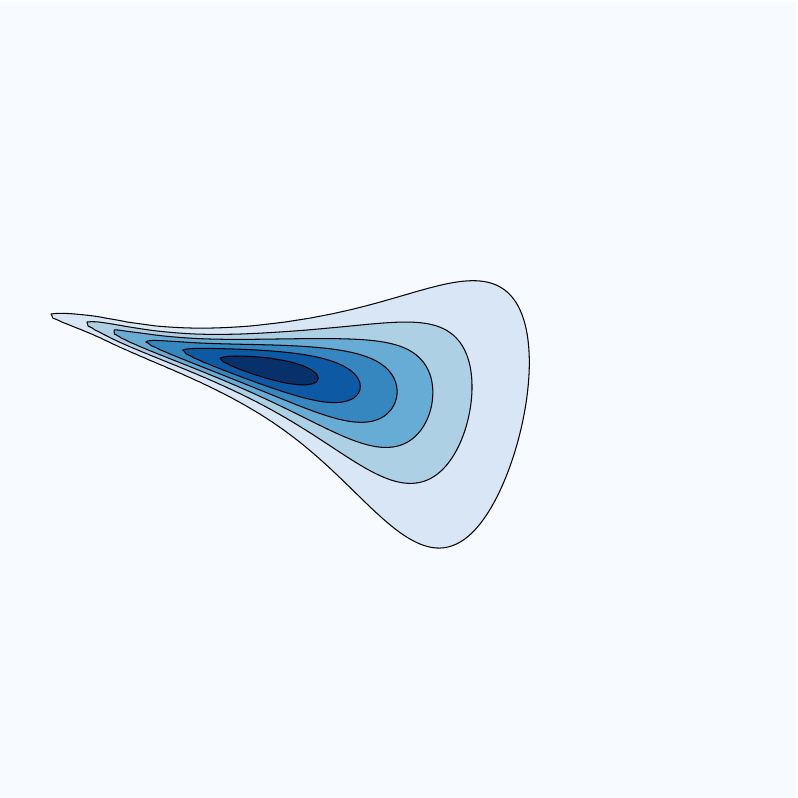}
		\caption{$\#\Lambda_{m^\ast} = 31$}
	\end{subfigure}
	\begin{subfigure}{0.15\linewidth}
	    \includegraphics[width=\textwidth]{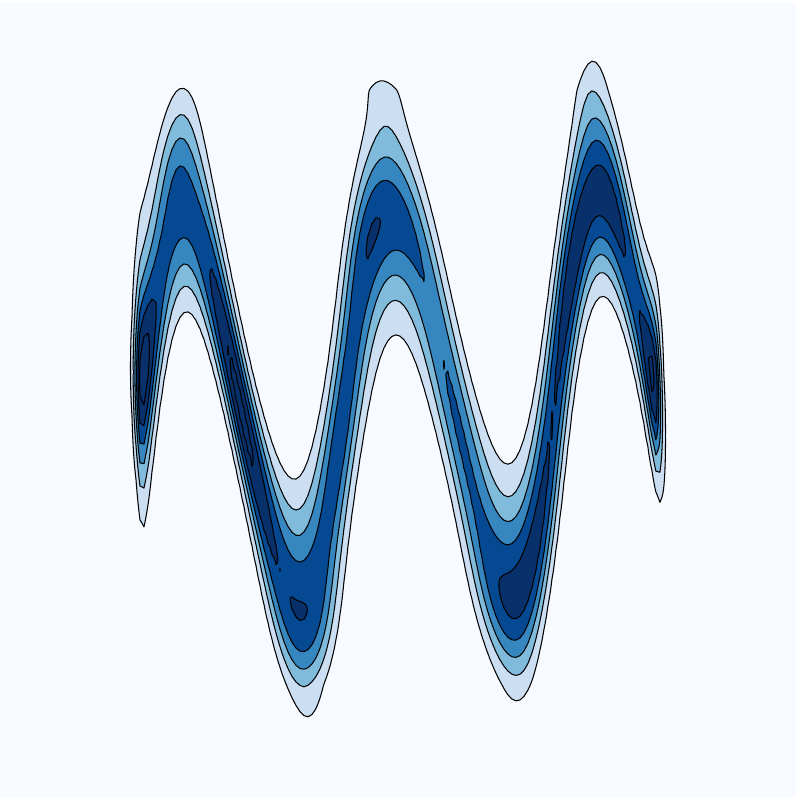}
		\caption{$\#\Lambda_{m^\ast} = 72$}
	\end{subfigure}
	\begin{subfigure}{0.15\linewidth}
		\includegraphics[width=\textwidth]{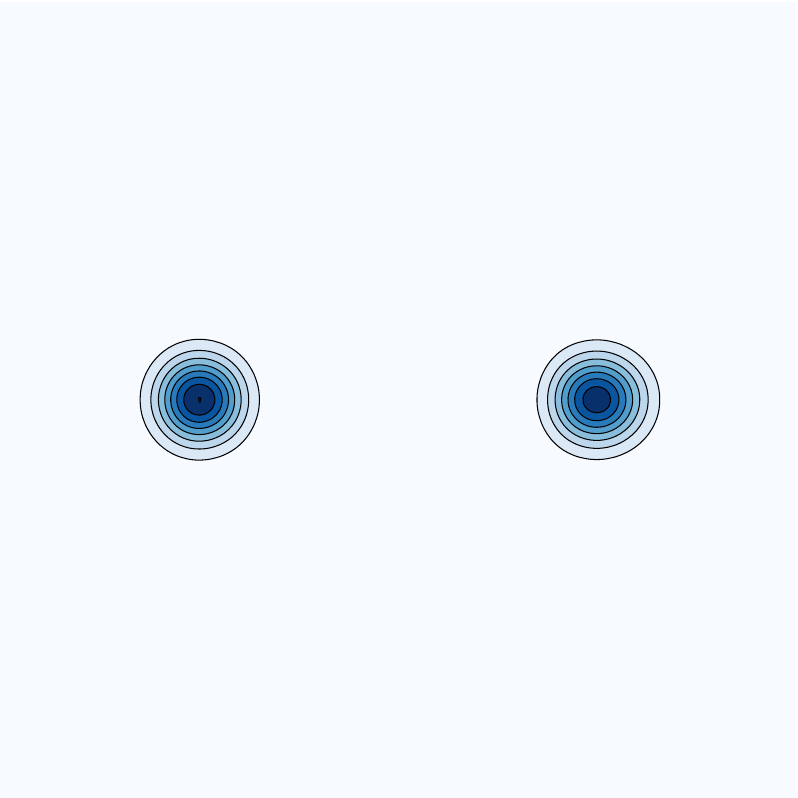}
		\caption{$\#\Lambda_{m^\ast} = 10$}
	\end{subfigure}
	\begin{subfigure}{0.15\linewidth}
	    \includegraphics[width=\textwidth]{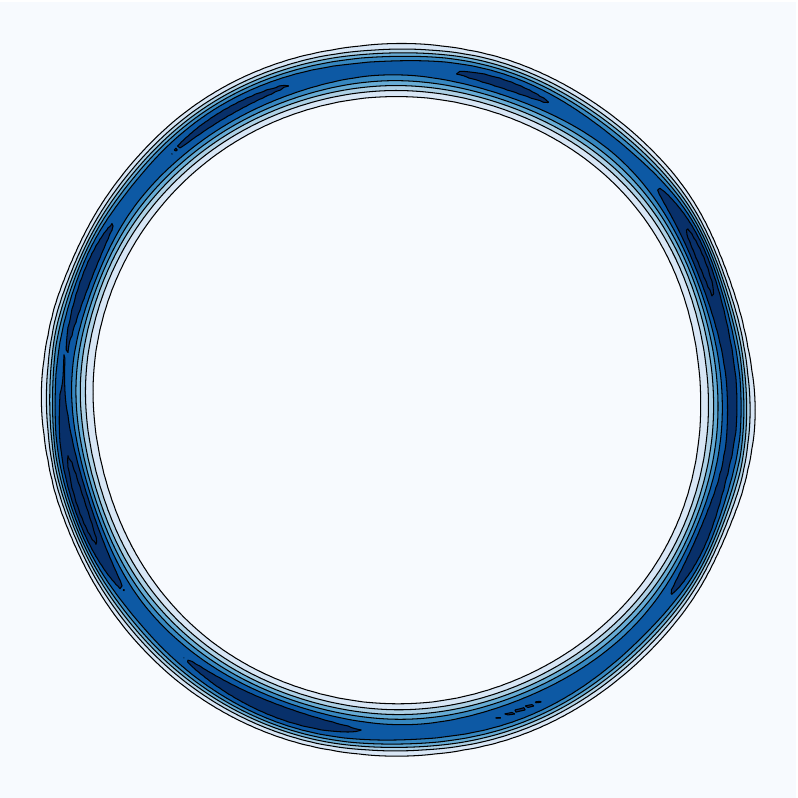}
		\caption{$\#\Lambda_{m^\ast} = 39$}
	\end{subfigure}\\
\caption{Five different  densities $\pi$ (top row) and their approximations $\widehat{\pi}$ (bottom row) built using the \ALG algorithm with a Hermite polynomial basis for $V_k$, given a sample of size $n = 10^4$ from $\pi$. The total number of coefficients $\#\Lambda_{m^\ast}$ in both map components is indicated below each density. \label{fig:toy_problems}} %
\end{figure}

\subsection{Random mixture of Gaussians} \label{subsec:adaptivity}

In this example, we evaluate the stability and accuracy of the \ALG algorithm for learning transport maps in the small sample regime. We consider a three-dimensional distribution $\pi$ defined as a mixture of Gaussians centered at the eight vertices of the hypercube $[-4,4]^3$, each with covariance matrix $I_{d}$. The weights of the mixture components are randomly sampled from a uniform distribution $\mathcal{U}([0,1])$ and then normalized so that $\int_{\R^3}\d\pi=1$. %
To estimate $\pi$, we generate a training sample $\{\bX^{i}\}_{i=1}^{n} \sim \pi$ and use the \ALG algorithm to estimate the KR rearrangement that pushes forward $\pi$ to $\eta$ using $5$-fold cross validation.  %

Figure~\ref{fig:adaptMoG_KL} plots the KL divergence of $\widehat{\pi}$ from $\pi$ averaged over 10 experiments. Here, the training sets used to build $\widehat{\pi}=\widehat{S}^\sharp\eta$ are of varying size $n \in [10^1, 10^4]$ and the reported KL divergence is computed on a common test set of $10^{4}$ samples.
Figure~\ref{fig:adaptMoG_ncoeffs} plots the total number of coefficients $\#\Lambda_{m^*}$ %
identified by the \ALG algorithm.
The performance of \ALG is compared to a non-adaptive method where $\Lambda = \Lambda(p) \coloneqq \{\balpha \in \mathbb{N}_{0}^{k}, \|\balpha\|_{1} \leq p\}$ is arbitrarily fixed, for $p=1,3,5$. Note that $\Lambda(p)$ corresponds to polynomial $f$ with total degree $p$.
For each sample size $n$, \ALG consistently finds a better estimator of $\pi$ (in the sense of KL divergence) than a non-adaptive method with a fixed degree $p$, by adaptively identifying the basis $\{\psi_{\balpha}\}_{\balpha\in\Lambda_t}$ to represent the map components. In addition, the \ALG estimator achieves smaller KL divergence with fewer total map coefficients, as seen in Figure~\ref{fig:adaptMoG_ncoeffs}. %

\begin{figure}[!bht]
    \centering
    \begin{subfigure}{.45\linewidth}
        \centering
        \includegraphics[width=\textwidth]{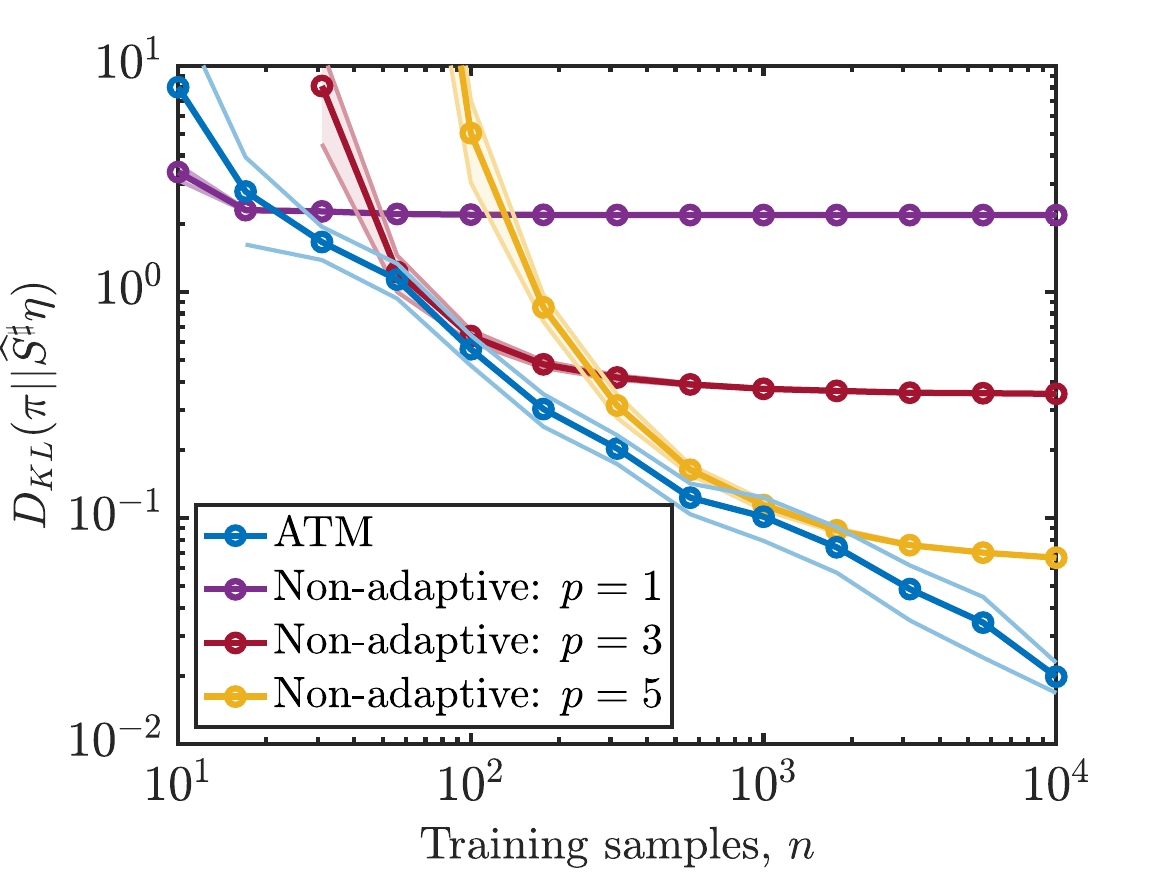}
        \caption{\label{fig:adaptMoG_KL}}
	    \vspace{-0.2cm}
    \end{subfigure}
	\hspace{1cm}
    \begin{subfigure}{.45\linewidth}
        \centering
        \includegraphics[width=\textwidth]{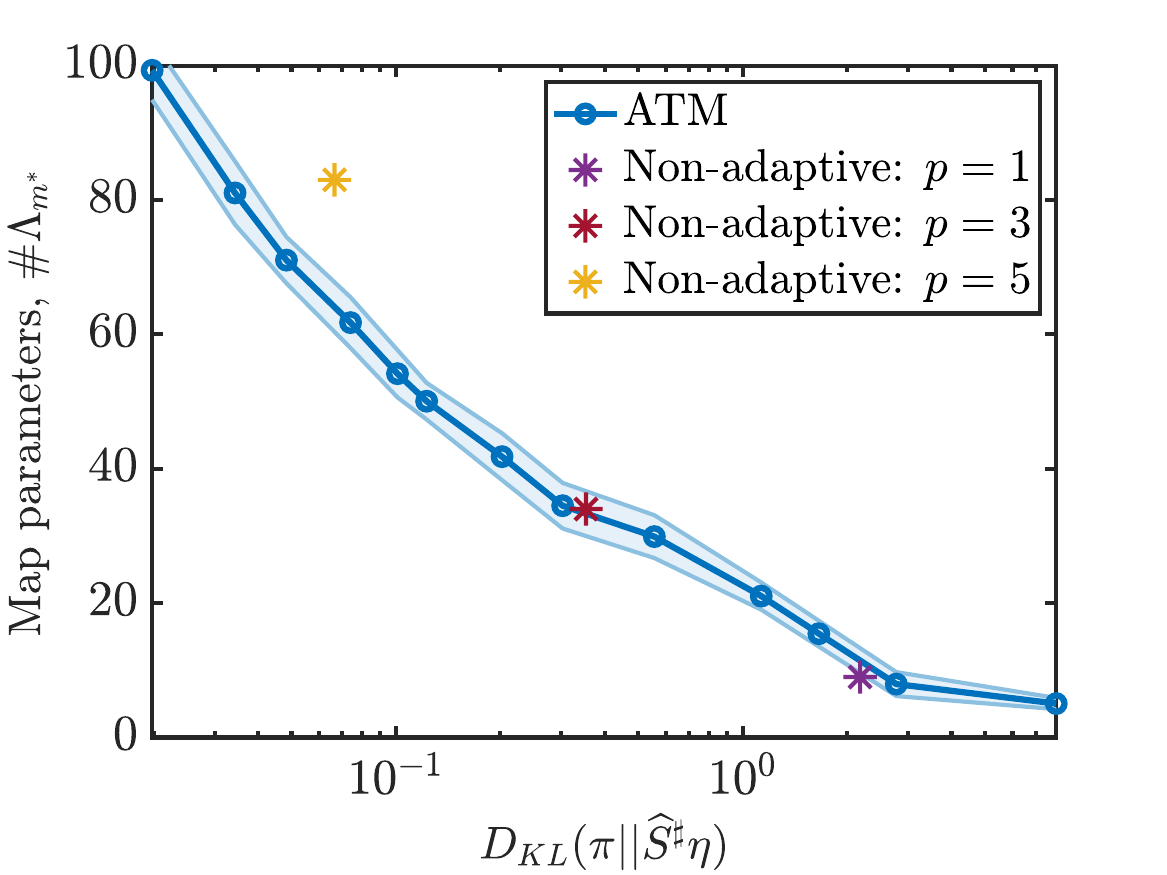}
        \caption{\label{fig:adaptMoG_ncoeffs}}
        \vspace{-0.2cm}
    \end{subfigure}
    \caption{(a) KL divergence over 10 sets of training samples is lower for \ALG than for non-adapted expansions. (b) Trade-off between the approximation quality and number of coefficients using the adaptive algorithm for different $n$. For any given number of map coefficients, the \ALG algorithm finds a representation with similar or lower KL divergence than the lowest achievable error of each non-adaptive approximation (indicated with stars in plot (b)).} %
\end{figure}

\subsection{Stochastic volatility} \label{subsec:stoc_volatility}

Next we consider data from a Markov process that describes the volatility of the return on a financial asset over time. The model has two hyperparameters $\mu$ and $\phi$, with a state $({Z}_k)_k$ that represents the log-volatility at times $k=1, \ldots, T$. The two hyperparameters are drawn from the distributions $\mu \sim \mathcal{N}(0,1)$ and $\phi = 2\exp(\phi^*)/(1 + \exp(\phi^*))$ for $\phi^* \sim \mathcal{N}(3,1)$. The log-volatility obeys the order-one autoregressive process $Z_{k+1} = \mu + \phi(Z_{k} - \mu) + \epsilon_{k}$ for all $k > 1$, where $\epsilon_{k} \sim \mathcal{N}(0,1)$ is independent of all other variables and $Z_0|\mu,\phi \sim \mathcal{N}(\mu,\frac{1}{1 - \phi^2})$. While the states are Gaussian conditioned on the hyperparameters, the joint distribution of 
$$
 \bX=(\mu,\phi, Z_1, \ldots, Z_T),
$$
is non-Gaussian. In this example, the dimension of $\bX$ is made arbitrarily large by increasing $T$. %

A key property of triangular transport maps is that they inherit sparsity from the conditional independence structure of $\pi$. \cite{spantini2018inference} showed that the Markov structure of $\pi$ yields a lower bound on the sparsity of the map $S$ (i.e., the functional dependence of each component $S^{k}$ on the input variables $\bx_{1:k-1}$). This sparsity was exploited to learn the structure of undirected probabilistic graph models from data in~\cite{morrison2017beyond,baptista2021learning}. %
From the conditional independence properties of the stochastic volatility model described above, we know that the KR rearrangement $S_{\KR}$ between the joint distribution of $\bX$ and a standard Gaussian reference $\eta$ is sparse. Moreover, the exact sparsity of $S_{\KR}$ can be derived from Theorem 5.1 in~\cite{spantini2018inference}. 
We now show that the \ALG algorithm is capable of detecting and exploiting this structure---without knowing in advance that it exists.

Figure~\ref{fig:SV_sparsity} compares the variable dependence of the true KR rearrangement $S_{\KR}$ and the map $\widehat{S}$ learned by the \ALG algorithm for a distribution with $T = 40$  %
using a sample of size $n = 10^3$ from $\pi$; a non-filled entry $(j,k)$ entry of the plot indicates that $k$-th map component does not depend on variable $x_j$. The dependence of component $S_k$ on $(x_k,x_{k-1})$ shows that each state $Z_k$ strongly depends on the previous state in time. Most of the map components also show dependence on the hyperparameters $(\mu,\phi)$. The estimated sparsity
closely matches the exact %
sparsity of the KR rearrangement. Furthermore, the sparse variable dependence of the $k$th map component $\widehat{S}_k$ on parent nodes $\text{Pa}(k) \subseteq \{1,\dots,k-1\}$ produces an approximation to the $k$th marginal conditional density given by %
$$\widehat\pi(x_{k}|\bx_{<k}) = \widehat\pi(x_k|\bx_{\text{Pa}(k)}).$$
By identifying the parent nodes $\text{Pa}(k)$ of each variable $k$, we also learn a sparse Bayesian network or directed acyclic graphical (DAG) model representing the target distribution~\cite{koller2009probabilistic}. As a result, we can see the \ALG algorithm as a technique for learning DAGs from samples, given a prescribed variable ordering.

Next, we consider the approximation of $S_{\KR}$ for Markov models of \emph{increasing} state dimension $T$ and hence increasing map dimension $d = T+2$; these experiments use a fixed sample size $n=10^3$ as before. Figure~\ref{fig:SV_negloglik} plots the KL divergence %
from the \ALG approximation to the joint density of 
$\bX \in \R^d$, for increasing dimensions $d$. 
We compare \ALG with a variant of \ALG where the exact sparsity pattern of the the KR rearrangement is provided to the algorithm in advance. This variant, labeled ``sparsity-aware \ALG'' %
in Figure \ref{fig:SV_negloglik}, differs from the \ALG algorithm in that it only activates multi-indices $\balpha^\ast_t$ which match the sparsity of $S_\text{KR}$ (meaning of the form $\balpha = (\alpha_1,\alpha_2,0,\hdots,0,\alpha_{k-1},\alpha_k)$).
We also compare \ALG with non-adaptive maps of degree $p=1$ and $p=2$; these maps do not exploit the conditional independence structure of $\pi$, and hence depend on all input variables. For low dimensions, we see that degree $p = 2$ maps can better capture the non-Gaussian target distributions than $p = 1$ linear maps. As the dimension $d$ increases, however, the growing number of coefficients in the $p=2$ maps results in higher-variance estimators, outweighing their smaller bias, and hence we see a rapidly increasing KL divergence as the dimension $d$ increases (with crossover around $d=21$). In contrast, \ALG outperforms the non-adaptive maps for all $d$ and, more importantly, achieves a KL divergence that grows linearly with $d$; %
indeed, it performs similarly to the best case ``oracle sparsity'' of sparsity-aware \ALG.

\begin{figure}[!bht]
    \centering
    \begin{subfigure}{.4\linewidth}
        \centering
        \includegraphics[width=\textwidth]{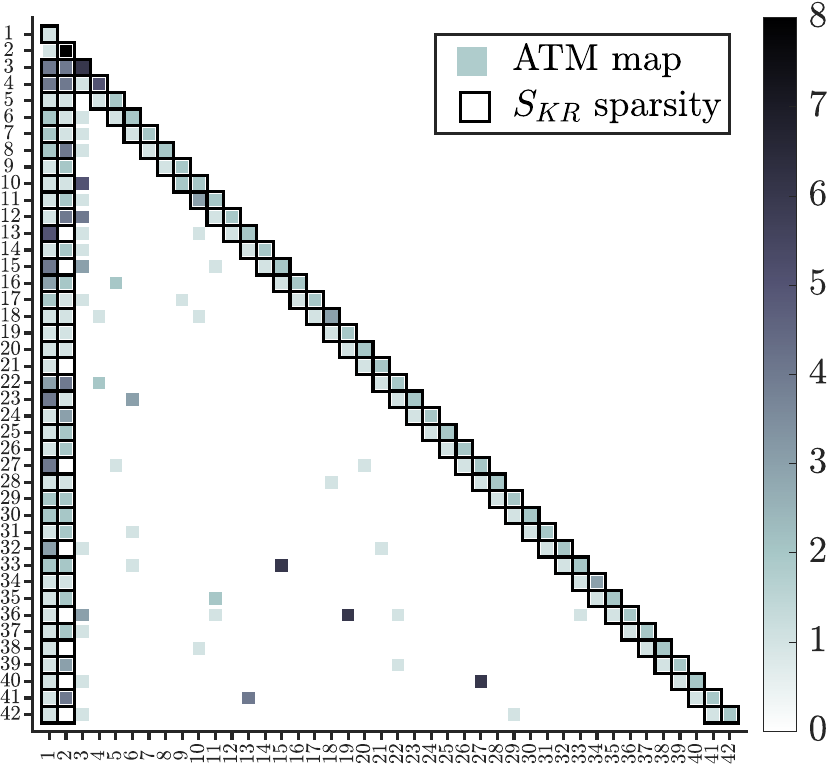}
        \caption{\label{fig:SV_sparsity}}
	    \vspace{-0.2cm}
    \end{subfigure}
	\hspace{1cm}
    \begin{subfigure}{.45\linewidth}
        \centering
        \includegraphics[width=\textwidth]{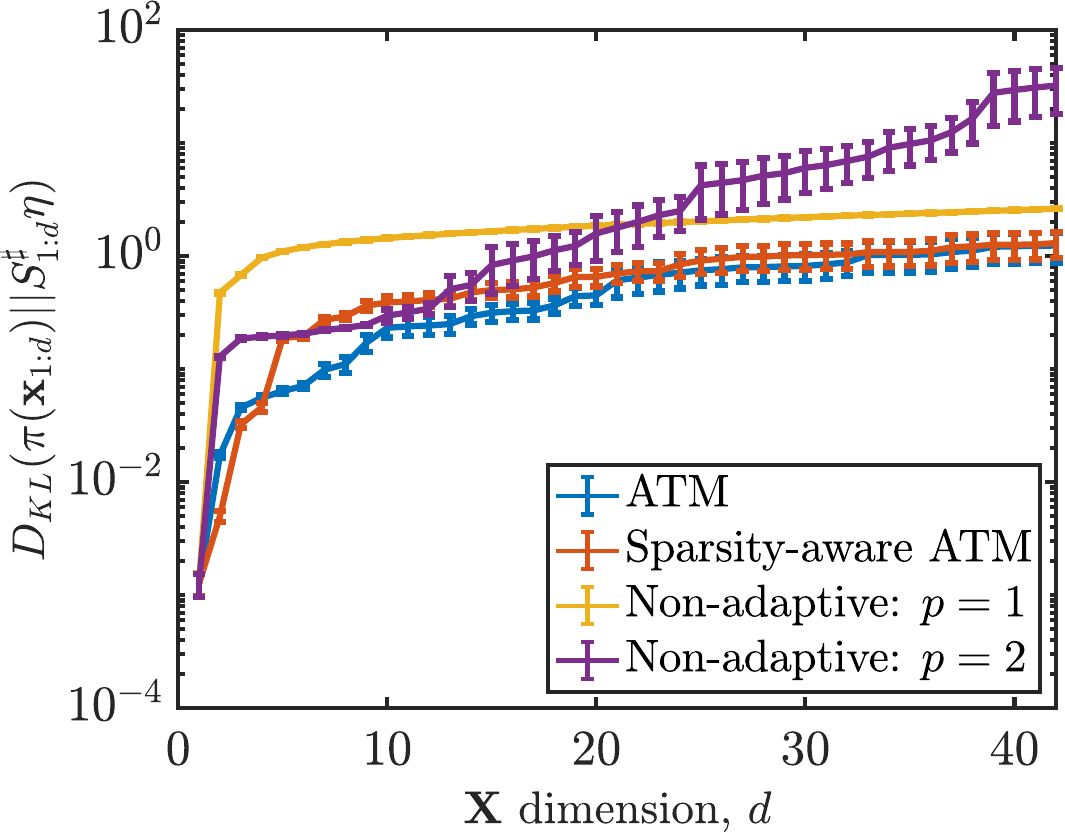}
        \caption{\label{fig:SV_negloglik}}
        \vspace{-0.2cm}
    \end{subfigure}
    \caption{(a) The sparsity of an \ALG map $\widehat{S}$ and of the KR rearrangement $S_\KR$ for the stochastic volatility model with $T = 40$ time steps. (b) KL divergence using \ALG with increasing dimension $d$, compared to adaptive map estimators with known sparsity and non-adaptive maps estimators.}
\end{figure}

\subsection{Tabular datasets} \label{sec:tabular}

Lastly, we evaluate \ALG's performance for density estimation on a suite of UCI datasets~\cite{UCIdataset} that were also considered in~\cite{uria2013rnade}. %
These datasets have dimensionalities between $d=10$ and $15$ and sample sizes between $n=506$ and $5875$.
We pre-process each dataset to eliminate discrete-valued variables and one variable in every pair that has Pearson correlation coefficient greater than 0.98, following the procedure in~\cite{uria2013rnade}.  
We consider 10 splits of the data. For each split, we use one fold (i.e., 10\% of the data) as a {test set} and the remaining 9 folds (i.e., 90\% of the data) as the {training set} to build $\widehat{S}$. To assess the quality of our estimated map $\widehat{S}$, we evaluate the negative log-likelihood of the pullback density $\widehat{\pi} = \widehat{S}^\sharp \eta$ on the test set. %
The negative log-likelihood %
is an empirical estimator of $-\mathbb{E}_{\pi}[\log \widehat{S}^{\sharp}\eta] = \mathcal{D}_\text{KL}(\pi || \widehat{S}^{\sharp}\eta) - \int\log\pi \d\pi$ and is, up to the unknown constant $-\int\log\pi \d\pi$, the KL divergence from $\widehat{S}^{\sharp}\eta$ to $\pi$. 
Table~\ref{tab:uci_results} presents the mean of the negative log-likelihoods over the $10$ splits and a 95\% confidence interval for the mean. For all datasets, we observe an improvement using \ALG over non-adaptive $p=2$ maps and multivariate Gaussian approximations (i.e., $p = 1$ maps), with a lower number of total coefficients.

{ \setlength{\tabcolsep}{4pt}
\begin{table}[!bht]
\caption{Mean negative log-likelihood for UCI datasets over 10 sets of training samples.\\The estimator with best performance (lowest negative log-likelihood) is  highlighted in bold. \label{tab:uci_results}}
\centering
\begin{tabular}{l|c|c|cc|cc}
\toprule
Dataset & $(d,n)$ & Gaussian & \ALG & \# coefficients & $p=2$ maps &  \# coefficients \\ %
\midrule
Boston & $(10,506)$ & $11.3 \pm 0.5$ & $\mathbf{3.1 \pm 0.6}$  & $228 \pm 7$  & $6.5 \pm 0.4$  & $285$ \\
Red Wine & $(11,1599)$ & $13.2 \pm 0.3$  & $\mathbf{9.8 \pm 0.4}$ & $289 \pm 9$ & $10.5 \pm 0.2$  & $363$ \\
White Wine & $(11,4898)$ & $13.2 \pm 0.5$  & $\mathbf{11.0 \pm 0.2}$ & $342 \pm 26$ & $12.0 \pm 1.0$  & $363$ \\
Parkinsons & $(15,5875)$ & $10.8 \pm 0.4$  & $\mathbf{2.8 \pm 0.4}$ & $783 \pm 17$ & $5.1 \pm 0.4$  & $815$ \\
\bottomrule
\end{tabular}
\end{table}
}

Lastly, we evaluate \ALG's performance for \emph{conditional} density estimation on a different suite of UCI datasets. We follow a similar procedure as above %
to pre-process each dataset. Each dataset has a one-dimensional predictor variable $X_d$ and covariates $\bX_{<d}$ of varying dimension. To approximate the conditional density $\pi_d(x_d|\bx_{<d})$, we estimate one map component $S_{d}$ as a function of the predictor and the covariates using joint samples $\{(\bX_{<k}^i,X_k^i)\}_{i=1}^{n}$. %
Table~\ref{tab:uci_conditional_results} presents the mean negative conditional log-likelihoods over $10$ splits of the training data. The negative conditional log-likelihood is an empirical estimator of $-\mathbb{E}_\pi[\log\widehat{S}_d^\sharp\eta] = \mathbb{E}_{\pi(\bx_{<d})}[\mathcal{D}_{\textrm{KL}}(\pi_{d}(x_d|\bx_{<d})||\widehat{S}_d^\sharp\eta)] -\int \log \pi_{d}(x_d|\bx_{<d}) \textrm{d}\pi_{d}$ and is, up to the unknown constant $-\int \log \pi_{d}(x_d|\bx_{<d}) \textrm{d}\pi_{d}$, the expected KL divergence from $\widehat{S}_d^{\sharp}\eta$ to the marginal conditional $\pi_d(x_d|\bx_{<d})$.   %

In Table~\ref{tab:uci_conditional_results}, we compare \ALG
to conditional kernel density estimation (\CKDE)~\cite{silverman1982estimation}, $\epsilon$-neighborhood kernel density estimation (\NKDE), and kernel mixture networks (\KMN)~\cite{ambrogioni2017kernel} using the implementation provided by~\cite{rothfuss2019conditional}. We also include results for high-capacity parametric conditional density estimation methods that rely on neural networks: mixture density networks (\MDN)~\cite{bishop1994mixture} and conditional normalizing flows (\NF)~\cite{trippe2018conditional}; implementation details of these methods are provided in Appendix~\ref{app:architectures}. For all datasets, we observe that \ALG has performance comparable to the neural-network based methods and improved performance over the nonparametric approaches (e.g., on the Concrete and Yacht datasets). 

In Table~\ref{tab:uci_conditional_timing}, we compare the number of coefficients in the \ALG and \MDN approximations and the computational time required to learn each of these conditional density representations. The runtime of the greedy procedure in Algorithm~\ref{alg:learn-map} (sequentially, for all $d$ map components) is generally less than that of \MDN, achieving similar approximation performance with about two orders of magnitude fewer coefficients. It is more realistic, however, to consider \ALG with the cross-validation procedure of Algorithm~\ref{alg:learn-map-withcv} and to compare this with the hyperparameter tuning required to achieve the reported performance of \MDN. \ALG only requires performing cross-validation for a single hyper-parameter, i.e., the total number of coefficients $m^\ast$ (see Algorithm~\ref{alg:learn-map-withcv}). In comparison, tuning several hyperparameters in the neural-network based methods (e.g., the learning rate, and the structure of the networks) by performing cross-validation over a tensor product grid increases runtimes by more than an order of magnitude relative to \ALG, as presented below. %

{ \setlength{\tabcolsep}{4pt}
\begin{table}[!bht]
\caption{Mean negative conditional log-likelihood for UCI datasets over 10 sets of training samples.\\ The method with best performance from both categories is highlighted in bold. \label{tab:uci_conditional_results}}
\centering
\begin{tabular}{l|c|ccc|ccc}
\toprule
Dataset & $(d,n)$ & \ALG & \CKDE & \NKDE & \MDN  & \KMN & \NF \\
\midrule
Boston & (12,506) & $\mathbf{2.6 \pm 0.2}$ & $\mathbf{2.6 \pm 0.2}$  & $3.1 \pm 0.2$  & $\mathbf{2.4 \pm 0.2}$  & $2.7 \pm 0.2$  & $\mathbf{2.4 \pm 0.1}$ \\
Concrete & (9,1030) & $\mathbf{3.1 \pm 0.1}$  & $3.2 \pm 0.1$  & $3.9 \pm 0.1$  & $\mathbf{2.9 \pm 0.1}$  & $3.5 \pm 0.1$  & $3.2 \pm 0.2$ \\ 
Energy & (10,768) & $1.5 \pm 0.1$  & $\mathbf{1.0 \pm 0.1}$  & $2.1 \pm 0.2$ & $\mathbf{1.2 \pm 0.1}$  & $1.7 \pm 0.1$  & $1.7 \pm 0.3$ \\
Yacht & (7,308) & $\mathbf{0.5 \pm 0.2}$  & $1.1 \pm 0.3$  & $3.8 \pm 0.2$ & $\mathbf{0.7 \pm 0.2}$  & $1.8 \pm 0.2$  & $1.3 \pm 0.5$  \\
\bottomrule
\end{tabular} 
\end{table}
}

\begin{table}[!bht]
\caption{The number of coefficients in the \ALG and \MDN approximations, as well as the runtimes (in seconds) to identify coefficients via optimization and to perform cross-validation of hyperparameters. \label{tab:uci_conditional_timing}}
\centering
\footnotesize
\begin{tabular}{l|ccc|ccc}
\toprule
& \multicolumn{3}{c|}{\ALG} & \multicolumn{3}{c}{\MDN} \\
Dataset & \# coefficients & runtime (s) & runtime w/CV & \# coefficients  & runtime (s) & runtime w/CV \\
\midrule
Boston & $31 \pm 3$ & $\;\;6 \pm 1$ & $\;\;88 \pm 14$ & $(5 \pm 2) \times 10^3$ &  $18 \pm 8$ & $1584 \pm 13$ \\
Concrete & $42 \pm 5$ & $14 \pm 3$ & $230 \pm 18$ & $(2.6 \pm 0.3) \times 10^3$ & $11 \pm 1$ & $1846 \pm 10$ \\ 
Energy & $41 \pm 7$ & $12 \pm 5$ & $193 \pm 29$ & $(8 \pm 3) \times 10^3$  & $\;\,32 \pm 12$ & $1758 \pm 17$ \\
Yacht & $28 \pm 4$ & $\;\;7 \pm 2$ & $\;\;99 \pm 20$ & $(6 \pm 2) \times 10^3$ &  $23 \pm 6$ & $1453 \pm 15$\\
\bottomrule
\end{tabular}
\normalsize
\end{table}

\section{Conclusions} \label{sec:conclusions}

This paper has presented and analyzed a functional framework for learning monotone triangular transport maps. Our approach represents monotone component functions of a triangular map through the action of an invertible \emph{rectification} operator $\mathcal{R}_k$ on smooth, generally non-monotone, functions. Imposing appropriate structure on this operator and the function space $V_k$ on which it acts, along with conditions on the target density $\pi$, yields an unconstrained optimization problem for learning the map components that has many desirable and useful properties. First, under certain assumptions on $\mathcal{R}_k$ and $\pi$, we show that the optimization objective is bounded, continuous, and differentiable for all $f \in V_k$. This permits the use of deterministic gradient-based optimization methods to find the minimizers. Next, under certain conditions on $\mathcal{R}_k$, we show that the optimization problem has \emph{no spurious local minima} on $V_k$. In practice, this yields important robustness to the initial guess and other parameters of the optimization algorithms. Finally, under the same (cumulative) conditions on $\mathcal{R}_k$ and some additional assumptions on the target density, we show that the optimization problem has a unique global minimizer on $V_k$ that corresponds to the canonical KR rearrangement.

Our functional framework also enables the construction of novel transport map estimators, based on maximum likelihood, given a finite sample from the target. The procedure we develop here is semi-parametric, making use of different hierarchical bases for $V_k$. In particular, we propose a greedy, adaptive algorithm that identifies a sparse set of basis functions, automatically tailored to the target density and to the sample size. Sparsity in the basis selection reflects conditional independence structure in the target. More generally, the maps we build can capture the structure of complex probability distributions with sufficient data, but also are robust in settings with few observations. This is crucial to deploying these algorithms within large-scale applications, such as data assimilation, where triangular transport maps must be learned online~\cite{spantini2019coupling}.
We demonstrate good computational performance of our algorithm---and the parsimonious representations it produces---in a variety of joint and conditional density estimation problems. 
We outline some interesting future research directions below.

\paragraph{KR recovery for broader classes of target distributions} %
While our result guaranteeing no spurious local minima (Theorem~\ref{prop:globalMinima}) makes no explicit assumptions on $\pi$, our subsequent results on recovery of the KR rearrangement (see Corollary~\ref{cor:KR_recovery}) make particular use of Gaussian tail assumptions, relating the target $\pi$ to the Gaussian reference $\eta$. In principle, however, the space $V_k$ contains functions $\Rectifier_k^{-1}(S_{\KR,k})$ corresponding to a broader class of target densities; $V_k$ includes functions that grow up to sub-exponentially fast, as a result of its Gaussian weighting. Hence, it includes the inverse images of KR rearrangements for lighter-tailed $\pi$, whose component functions grow faster than linear as $x_k \to \pm \infty$. \rev{For instance, \cite[Sec.\ 2.1]{cui2022prior} notes that the increasing rearrangement from the univariate exponential power distribution $\pi(x)\propto e^{-|x|^p}$ with $p > 2$ to a Gaussian $\eta$ grows as $\mathcal{O}(x^{p/2})$ when $x\rightarrow\infty$; see \cite{jaini2019} for the asymptotic behavior of monotone maps for other light-tailed densities.} %
It may be possible to obtain unique recovery results analogous to Corollary~\ref{cor:KR_recovery} for such densities, although some modifications to our analysis would be needed.

\paragraph{Approximation theory and statistical guarantees} An open research topic, to the best of our knowledge, is to analyze the finite-dimensional approximation of triangular transport maps on unbounded domains. (See~\cite{zech2022sparse,zech2021sparse} for an approximation theory of triangular transports between analytic densities $\pi,\eta$ on bounded domains.) These results would show how approximate maps converge to the KR rearrangement with different parameterizations (e.g., polynomial spaces or neural networks of increasing size), and we expect that our rectification framework could be a useful ingredient of such analyses.
A parallel line of work is to develop non-asymptotic statistical convergence results for triangular transports (e.g., in the context of density estimation), to understand how the quality of density estimates and of map estimates depends on the sample size $n$. \rev{The analyses in~\cite{irons2022triangular,wang2022minimax} address this question for certain classes of distributions on \emph{bounded} domains.} %
Combining these lines of work could provide lower bounds on the number of samples required to learn maps of a given complexity, and also some analytical guidance for how to navigate the bias-variance tradeoff of a map estimator given finite samples. %

\paragraph{Other sources of low-dimensional structure} The \ALG algorithm in this paper identifies maps with sparse variable dependence by exploiting conditional independence structure in the target distribution~\cite{spantini2018inference}. In future work it will be interesting to investigate other notions of low-dimensional structure and how they could facilitate the learning of maps from small samples. One promising notion is when the target density $\pi$ departs from the reference $\eta$ only along a low-dimensional subspace. In this case, the triangular map $S$ pushing forward $\pi$ to $\eta$ can be written as a low-dimensional perturbation of the identity map, after a variable rotation. See~\cite{bigoni2019greedy} for a recent contribution in this direction applied to variational inference. 

\paragraph{Map ordering} One disadvantage of triangular maps is that the approximation depends on the choice of variable ordering. In particular, each variable ordering yields a different factorization of the target density and a different Knothe--Rosenblatt rearrangement $S_{\KR}$. Thus, it is of interest to develop variable ordering algorithms that minimize the finite-sample error of the estimated transport map $\widehat{S}$, the estimated pullback density $\widehat{\pi} = \widehat{S}^\sharp \eta$, or other goal-oriented metrics. For target distributions that induce sparse variable dependence in the map, one approach is to find the permutation that maximizes the sparsity of $\widehat{S}$. This is equivalent to finding the the sparsest Bayesian network or directed acyclic graphical (DAG) model for a distribution from samples. While this problem is in general NP-complete, effective algorithms \cite{raskutti2018learning} have been proposed. This algorithm reduces to finding a sparse maximum likelihood estimator for the Cholesky factor of the inverse covariance matrix of $\pi$ in the Gaussian setting. Since a linear transport map is precisely this Cholesky factor for Gaussian $\pi$ and standard normal $\eta$ (see~\cite[Section 3]{baptista2021learning}), we expect that such sparse permutation algorithms could be generalized to the nonlinear transport map setting.

\paragraph{Nonparametric methods} Instead of finding  a particular finite-dimensional basis for the map components $S_k = \Rectifier_k(f)$, nonparametric methods do not limit the functional form of $f$ (e.g., as linear combinations of multivariate polynomials). A broadly useful nonparametric approach involves seeking $f$ in a reproducing kernel Hilbert space (RKHS). Thanks to the representer theorem~\cite{scholkopf2001generalized}, the optimal $f$ in the RKHS can still be identified by solving a finite-dimensional optimization problem. Choosing the kernel so that the RKHS is a suitable weighted Sobolev space (see, e.g., \cite{novak2018reproducing}) will yield map representations that fall within the framework of Section~\ref{subsec:choice_Vk}. 
A related approach uses Gaussian process representations for the map components~\cite{katzfuss2021scalable}. It will be interesting to compare the finite-sample performance of such methods to the semi-parametric procedure proposed in this work.

\section*{Acknowledgments}
RB, YM, and OZ gratefully acknowledge support from the INRIA associate team Unquestionable. RB and YM are also grateful for support from the AFOSR Computational Mathematics program (MURI award FA9550-15-1-0038) and the US Department of Energy AEOLUS center. RB acknowledges support from an NSERC PGSD-D fellowship.
OZ also acknowledges support from the ANR JCJC project MODENA (ANR-21-CE46-0006-01).

\appendix

\section{Proofs and theoretical details} \label{app:optimization}
\subsection{Proof of Proposition~\ref{prop:KLboundsL2norm}}
\label{proof:KLboundsL2norm}

\begin{proof}
Recall that the KR rearrangement $S_{\mathrm{KR}}$ is a transport map that satisfies $S_{\mathrm{KR}}^\sharp\eta=\pi$, where $\eta$ is the density of the standard Gaussian measure on $\R^d$ and $\pi$ is the target density. 
Corollary 3.10 in \cite{bogachev2005triangular} states that for any PDF $\varrho$ on $\R^d$ of the form $\varrho(\bx):=f(\bx)\eta(\bx)$ with $f\log f\in L^1_\eta$, the inequality
\begin{equation}\label{eq:tmp435t6}
 \int \|\bx-T(\bx)\|^2 \eta(\bx)\d\bx \leq 2 \int f(\bx)\log f(\bx) \eta(\bx)\d\bx,
\end{equation}
holds, where $T$ is the KR rearrangement such that $T_\sharp \eta=\varrho$.
Let $S$ be an increasing lower triangular map as in \eqref{eq:increasingMaps} and let $\varrho = S_\sharp \pi$.
Thus we have $T = S\circ S_{\mathrm{KR}}^{-1}$ and so the left-hand side of \eqref{eq:tmp435t6} becomes
$$
 \int \|\bx-T(\bx)\|^2 \eta(\bx)\d\bx 
 = \int \|\bx-S\circ S_{\mathrm{KR}}^{-1}(\bx)\|^2 \eta(\bx)\d\bx 
 = \int \|S_{\mathrm{KR}}(\bx)-S(\bx)\|^2 \pi(\bx)\d\bx ,
$$
and the right-hand side becomes
$$
 2 \int f(\bx)\log f(\bx) \eta(\bx)\d\bx = 2 \mathcal{D}_\text{KL}( \varrho||\eta ) = 2\mathcal{D}_\text{KL}( \pi|| S^\sharp \eta ),
$$
which yields \eqref{eq:KLboundsL2norm}.
\end{proof}

\subsection{Convexity of $s\mapsto \mathcal{J}_k(s)$}\label{proof:ConvexityOfJ}

\begin{lemma} \label{lem:convexity}
The optimization problem $\min_{\{s\colon \partial_k s > 0\}} \mathcal{J}_{k}(s)$ is strictly convex.
\end{lemma}

\begin{proof}
 \red{Let $s_{1}, s_{2}\colon  \mathbb{R}^k \rightarrow \mathbb{R}$ be two functions such that $\partial_{k} s_{1}(\bx_{\leq k}) > 0$ and $\partial_{k} s_{2}(\bx_{\leq k}) > 0$. Let $ s_t = t s_1 + (1-t)s_2$ for $0<t<1$. Then $s_t$ also satisfies $\partial_{k} s_{t}(\bx_{\leq k}) > 0$}.
 Finally, because both $\xi\mapsto \frac{1}{2}\xi^2$ and $\xi\mapsto -\log(\xi)$ are strictly convex functions, we have
 \begin{align*}
     \mathcal{J}_{k}(s_t) 
     &\overset{\eqref{eq:opt_map_component}}{=} \int \left( \frac{1}{2}s_t(\bx_{\leq k})^2-\log\partial_k s_t(\bx_{\leq k}) \right)\pi(\bx)\d\bx \\
     &\;<\; \int  \Big( t \frac{1}{2}s_1(\bx_{\leq k})^2 + (1-t) \frac{1}{2}s_2(\bx_{\leq k})^2 \Big)-\Big( t \log\partial_k s_1(\bx_{\leq k}) + (1-t) \log\partial_k s_2(\bx_{\leq k}) \Big)  \pi(\bx)\d\bx \\
     &\;=\; t \mathcal{J}_{k}(s_1) + (1-t)\mathcal{J}_{k}(s_2) ,
 \end{align*}
 which shows that $\mathcal{J}_{k}$ is strictly convex.
\end{proof}

\subsection{Proof of Proposition~\ref{thm:trace}}\label{proof:trace}

To prove Proposition~\ref{thm:trace} we need the following lemma.
\begin{lemma}\label{thm:RKHS}
 Let
 $$
  H^1([0,1])=\left\{f\colon[0,1]\rightarrow\R \text{ such that } \|f\|^2_{H^1([0,1])}\coloneqq \int_0^1 f(t)^2+f'(t)^2 \, \d t\right\} .
 $$
 Then
 \begin{equation}\label{eq:RKHS}
     |f(0)|\leq \sqrt{2} \|f\|_{H^1([0,1])},
 \end{equation}
 holds for any $f\in H^1([0,1])$.
 
\end{lemma}

\begin{proof}
 Because $\mathcal{C}^{\infty}([0,1])$ is dense in $H^1([0,1])$, it suffices to show \eqref{eq:RKHS} for any $f\in \mathcal{C}^{\infty}([0,1])$.
 By the mean value theorem, there exists $0\leq z \leq 1$ such that 
 $$
  f(z) = \frac{1}{1-0}\int_{0}^{1} f(t) \d t.
 $$
 Thus we can write
 \begin{align*}
  |f(0)|^2
  &\leq 2|f(z)-f(0)|^2 + 2|f(z)|^2 \\
  &= 2\left| \int_{0}^{z}  f'(t)\d t\right|^2 +  2\left| \int_{0}^{1} f(t) \d t \right|^2  \\
  &\leq 2\int_{0}^{1} \left| f'(t) \right|^2 \d t +   2\int_{0}^{1} \left| f(t) \right|^2 \d t .
 \end{align*}
 This concludes the proof.
\end{proof}

We now prove Proposition~\ref{thm:trace}.
\begin{proof}
  For any $f\in V_k$, Lemma \ref{thm:RKHS} permits us to write
  \begin{align*}
      \int |f(\bx_{<k},0)|^2 \eta_{<k}(\bx_{<k})\d \bx_{<k}
      &\overset{\eqref{eq:RKHS}}{\leq} \int \left(2\int_0^1 |f(\bx_{<k},t)|^2 + |\partial_k f(\bx_{<k},t)|^2 \d t \right) \eta_{<k}(\bx_{<k})\d \bx_{<k} \\
      &\;\leq\;  C_T\int \int_0^1 \Big( |f(\bx_{<k},t)|^2 + |\partial_k f(\bx_{<k},t)|^2    \Big)  \eta_{<k}(\bx_{<k})\eta_1(t) \d \bx_{<k}\d t \\
      &\;\leq\;  C_T\int \int_{-\infty}^{+\infty} \Big( |f(\bx_{<k},t)|^2 + |\partial_k f(\bx_{<k},t)|^2    \Big)  \eta_{\leq k}(\bx_{<k},t) \d \bx_{<k}\d t \\
      &\;=\; C_T\|f\|_{V_k}^2 ,
  \end{align*}
 where $C_T=2 \sup_{0\leq t\leq 1} \eta_1(t)^{-1}$.
\end{proof}

\subsection{Proof of Proposition~\ref{prop:RectifierIsLip}} \label{proof:RectifierContinuity}

The proof relies on Proposition~\ref{thm:trace} and on the following generalized integral Hardy inequality, see~\cite{muckenhoupt1972hardy}.

\begin{lemma} \label{lemma:hardy_application}
Let $\eta_{\leq k}$ be the standard Gaussian density on $\mathbb{R}^k$. Then there exists a constant $C_{H}$ such that for any $v \in L^2_{\eta}(\mathbb{R}^{k})$, 
\begin{equation} \label{eq:Hardy}
 \int \left(\int_{0}^{x_{k}} v(\bx_{<k},t)\mathrm{d}t\right)^2\eta_{\leq k}(\bx)\mathrm{d}\bx
 \leq C_{H} \int v(\bx)^2\eta_{\leq k}(\bx)\mathrm{d}\bx.
 \end{equation}
\end{lemma}

\begin{proof}[Proof of Lemma~\ref{lemma:hardy_application}] 
Let us recall the integral Hardy inequality \cite{muckenhoupt1972hardy}.

\begin{theorem}[from \cite{muckenhoupt1972hardy}] \label{thm:hardy}
For weight $\rho\colon \mathbb{R}_{+} \rightarrow \mathbb{R}_{+}$ and $u \in L^2_{\rho}(\mathbb{R})$, there exists a constant $C_{H} < \infty$ such that
\begin{equation}\label{eq:Hardy_onedimensional}
 \int_{0}^{+\infty} \left(\int_{0}^{x} u(t)\mathrm{d}t\right)^2 \rho(x) \mathrm{d}x \leq C_{H} \int_{0}^{+\infty} u(x)^2\rho(x)\mathrm{d}x
\end{equation}
if and only if 
\begin{equation} \label{eq:rho_tail}
\sup_{x > 0} \, \left(\int_{x}^{+\infty}\rho(t)\mathrm{d}t \right)^{1/2} \left(\int_{0}^{x} \rho(t)^{-1} \mathrm{d}t \right)^{1/2} < +\infty.
\end{equation}
\end{theorem}

We apply Theorem \ref{thm:hardy} with the one-dimensional standard Gaussian density $\rho = \eta$ for $x > 0$. In order to check condition~\eqref{eq:rho_tail}, we need to show that
$$
D(x) \coloneqq \left(\int_{x}^{+\infty}\rho(t)\mathrm{d}t \right)^{1/2} \left(\int_{0}^{x} \rho(t)^{-1}\mathrm{d}t \right)^{1/2} = \left(\int_{x}^{+\infty} e^{-t^2/2}\mathrm{d}t \right)^{1/2} \left(\int_{0}^{x} e^{t^2/2}\mathrm{d}t \right)^{1/2} ,
$$
is bounded. Since $x\mapsto D(x)$ is a continuous function with a finite limit as $x \rightarrow 0$, it is sufficient to show that $D(x)$ has a finite limit when $x \rightarrow \infty$. 
For $x > 1$, $\int_{x}^{+\infty} e^{-t^2/2} \mathrm{d}t \leq e^{-x^2/2}$ and $D(x)^2 \leq e^{-x^2/2} \int_{0}^{x} e^{t^2/2} \mathrm{d}t$. Furthermore, using integration-by-parts we have $\int_{0}^{x} e^{t^2/2} \mathrm{d}t = \int_{0}^{1} e^{t^2/2}\mathrm{d}t + e^{x^2/2}/x - \sqrt{e} + \int_{1}^{x} e^{t^2/2}/t^2\mathrm{d}t$. As $x \rightarrow \infty$ the dominating term in the sum is $e^{x^2/2}/x$. Thus, $e^{-x^2/2} \int_{0}^{x} e^{t^2/2}\mathrm{d}t$ behaves asymptotically as $\mathcal{O}(\frac{1}{x})$, so that $D(x)\rightarrow 0$ when $x\rightarrow \infty$. Thus, condition~\eqref{eq:rho_tail} is satisfied.

Then, by the Hardy inequality in~\eqref{eq:Hardy_onedimensional} for $u \in L_{\eta}^2(\mathbb{R})$ we have
\begin{equation} \label{eq:hardy_Gaussian_right_tail}
 \int_{0}^{+\infty} \left(\int_{0}^{x_{k}} u(t)\mathrm{d}t\right)^2 \eta(x_{k}) \mathrm{d}x_{k} \leq C_{H} \int_{0}^{+\infty} u(x_{k})^2\eta(x_{k}) \mathrm{d}x_{k}.
\end{equation}
For the symmetric density $\eta(x_{k}) = \eta(-x_{k})$ %
we also have
\begin{equation} \label{eq:hardy_Gaussian_left_tail}
 \int_{-\infty}^{0} \left(\int_{0}^{x_{k}} u(t)\mathrm{d}t\right)^2 \eta(x_{k}) \mathrm{d}x_{k} \leq C_{H} \int_{-\infty}^{0} u(x_{k})^2 \eta(x_{k}) \mathrm{d}x_{k}.
\end{equation}
Combining the results in~\eqref{eq:hardy_Gaussian_right_tail} and~\eqref{eq:hardy_Gaussian_left_tail} we have %
\begin{equation*}%
 \int_{-\infty}^{+\infty} \left(\int_{0}^{x_{k}} u(t)\mathrm{d}t\right)^2 \eta(x_{k}) \mathrm{d}x_{k} \leq C_{H} \int_{-\infty}^{+\infty} u(x_{k})^2 \eta(x_{k}) \mathrm{d}x_{k}.
\end{equation*}
Setting $u(t) = v(\bx_{<k},t)$ and integrating both sides over $\bx_{<k} \in \mathbb{R}^{k-1}$ with the standard Gaussian weight function $\eta(\bx_{<k})$ gives the result.
\end{proof}

We now prove Proposition~\ref{prop:RectifierIsLip}. 
\begin{proof} By Proposition~\ref{thm:trace}, Lemma \ref{lemma:hardy_application} and by the Lipschitz property of $g$, we can write
\begin{align}
\| \mathcal{R}_k(f_1) - \mathcal{R}_k(f_2)\|_{L^2_{\eta_{\leq k}}}^2 
&\leq 2 \int \Big( f_1(\bx_{<k},0)-f_2(\bx_{<k},0) \Big)^2 \eta_{\leq k}(\bx)\d\bx \nonumber\\
&~ + 2 \int \Big( \int_0^{x_k} g\big(\partial_k f_1(\bx_{<k},t)\big) - g\big(\partial_k f_2(\bx_{<k},t)\big)  \mathrm{d}t \Big)^2\eta_{\leq k}(\bx)\d\bx \nonumber \\
&\leq 2C_T \|f_1-f_2 \|_{V_k}^2 + 2 C_H \| g(\partial_k f_1)-g(\partial_k f_2) \|_{L^2_{\eta_{\leq k}}}^2 \nonumber\\
&\leq 2C_T \|f_1-f_2 \|_{V_k}^2 + 2 C_H L^2 \| \partial_k f_1-\partial_k f_2 \|_{L^2_{\eta_{\leq k}}}^2 \nonumber\\
&\leq 2(C_T+C_H L^2) \|f_1-f_2 \|_{V_k}^2 , \label{eq:tmp58317}
\end{align}
for any $f_1,f_2\in V_k$. Furthermore, using the Lipschitz property of $g$ we have 
\begin{align}
\| \partial_k \mathcal{R}_k(f_1) -  \partial_k \mathcal{R}_k(f_2)\|_{L^2_{\eta_{\leq k}}}^2 
&= \int \Big( g(\partial_k f_1(\bx_{<k},t)) - g(\partial_k f_2(\bx_{<k},t)) \Big)^2 \eta_{\leq k}(\bx)\d\bx \nonumber\\
&\leq L^2 \int \Big( \partial_k f_1(\bx_{<k},t) - \partial_k f_2(\bx_{<k},t) \Big)^2 \eta_{\leq k}(\bx)\d\bx \nonumber\\
&\leq L^2 \|f_1-f_2 \|_{V_k}^2.
\label{eq:tmp58318}
\end{align}
Combining \eqref{eq:tmp58317} with \eqref{eq:tmp58318}, we obtain 
\eqref{eq:RectifierIsLip} with $C = \sqrt{2(C_T+C_H L^2)+L^2}$. 

It remains to show that $\| \mathcal{R}_k(f) \|_{V_k}<\infty$ for any $f\in V_k$. Letting with $f_1=f$ and $f_2=0$ in \eqref{eq:RectifierIsLip}, the triangle inequality yields
$$
\| \mathcal{R}_k(f) \|_{V_k} \leq \| \mathcal{R}_k(0) \|_{V_k} + C \|f\|_{V_k}
$$
Because $\mathcal{R}_k(0)$ is the affine function $\bx\mapsto g(0)x_k$, we have that $\|\mathcal{R}_k(0) \|_{L^2_{\eta_{\leq k}}}^2= g(0)^2\int x_k^2 \eta(\bx)\d\bx$ and $\| \partial_k \Rectifier_k(0) \|_{L^2_{\eta_{\leq k}}}^2 = g(0)^2$ are finite and so is $\| \mathcal{R}_k(0) \|_{V_k}$. Thus, $\mathcal{R}_k(f)\in V_k$ for all $f\in V_k$.
\end{proof}

\subsection{Proof of Proposition~\ref{prop:finite_objective}}\label{proof:finite_objective}

\begin{proof}
For any $f\in V_k$ we have
\begin{align*}
 \left| \mathcal{L}_{k}(f) \right| 
 &\;=\; \left|\int \left(\frac{1}{2}\mathcal{R}_k(f)^2 - \log(\partial_{k} \mathcal{R}_k(f)) \right)\d\pi\right|  \\
 &\overset{\eqref{eq:AssumptionPiBounded}}{\leq} \frac{C_\pi}{2} \|\mathcal{R}_k(f)\|_{L^2_{\eta_{\leq k}}}^2 + C_\pi \int \left|\log(g(\partial_{k} f)) \right| \d\eta_{\leq k} \\
 &\;\leq\; \frac{C_\pi}{2} \|\mathcal{R}_k(f)\|_{L^2_{\eta_{\leq k}}}^2 + C_\pi \int |\log(g(0))| + \left|\log(g(\partial_{k} f)) -\log(g(0)) \right| \d\eta_{\leq k} \\
 &\overset{\eqref{eq:lipschitz_logg}}{\leq} \frac{C_\pi}{2} \|\mathcal{R}_k(f)\|_{L^2_{\eta_{\leq k}}}^2 + C_\pi  |\log(g(0))| + C_\pi L \int \left|\partial_{k} f -0 \right| \d\eta_{\leq k} \\
 &\;\leq\; \frac{C_\pi}{2} \|\mathcal{R}_k(f)\|_{L^2_{\eta_{\leq k}}}^2 + C_\pi  |\log(g(0))| + C_\pi L \|f\|_{V_k}^2 .
\end{align*}
Because Proposition \ref{prop:RectifierIsLip} ensures $\mathcal{R}_k(f)\in V_k\subset L^2_{\eta_{\leq k}}$, we have that $\mathcal{L}_{k}(f)$ is finite for any $f\in V_k$.
Now, for any $f_{1},f_2\in V_k$, we can write
\begin{align*} %
\left| \mathcal{L}_{k}(f_1) - \mathcal{L}_{k}(f_2) \right| 
&\;=\; \left|\int \left(\frac{1}{2}\mathcal{R}_k(f_{1})^2 - \frac{1}{2}\mathcal{R}_k(f_{2})^2 - \log(\partial_{k} \mathcal{R}_k(f_{1})) + \log(\partial_{k} \mathcal{R}_k(f_{2})) \right)\d\pi\right|\\
&\overset{\eqref{eq:AssumptionPiBounded}}{\leq}C_\pi \int \frac{1}{2}\Big|\mathcal{R}_k(f_{1})^2 - \mathcal{R}_k(f_{2})^2 \Big| + \Big|\log(g(\partial_{k} f_{1})) - \log(g(\partial_{k}f_{2}))\Big| \d\eta\\
&\overset{\eqref{eq:lipschitz_logg}}{\leq}  \frac{C_\pi}{2} \|\mathcal{R}_k(f_1)+\mathcal{R}_k(f_2)\|_{L^2_{\eta_{\leq k}}}\|\mathcal{R}_k(f_1)-\mathcal{R}_k(f_2)\|_{L^2_{\eta_{\leq k}}} + C_\pi L \|\partial_{k} f_{1} - \partial_{k}f_{2}\|_{L^2_{\eta_{\leq k}}} \\
&\overset{\eqref{eq:RectifierIsLip}}{\leq } C_\pi \frac{\|\mathcal{R}_k(f_1)\|_{L^2_{\eta_{\leq k}}} + \|\mathcal{R}_k(f_2)\|_{L^2_{\eta_{\leq k}}}}{2} C\|f_1-f_2\|_{V_k} + C_\pi L \|f_{1} - f_{2}\|_{V_k}.
\end{align*}
This shows that $\mathcal{L}_{k}\colon V_k\rightarrow \R$ is continuous.
To show that $\mathcal{L}_{k}$ is differentiable, we let $f,\varepsilon\in V_k$ so that
\begin{align*}
 \mathcal{L}_{k}(f+\varepsilon)
 &= \int \left(\frac{1}{2}\mathcal{R}_k(f+\varepsilon)^2 - \log(\partial_{k} \mathcal{R}_k(f+\varepsilon)) \right)\d\pi \\
 &= \int \left(\frac{1}{2}\left(f(\bx_{<k},0) + \varepsilon(\bx_{<k},0) + \int_0^{x_k} g\big(\partial_k (f+\varepsilon)(\bx_{<k},t)\big)\mathrm{d}t\right)^2 - \log\circ g(\partial_{k}(f+\varepsilon)(\bx)) \right)\pi(\bx)\d\bx \\
 &= \int \left(\frac{1}{2}\mathcal{R}_k(f)(\bx)^2 + \mathcal{R}_k(f)(\bx)\left(\varepsilon(\bx_{<k},0) + \int_0^{x_k} g'\big(\partial_k f(\bx_{<k},t)\big)\partial_k \varepsilon(\bx_{<k},t)\mathrm{d}t\right) \right)\pi(\bx)\d\bx \\
 &~~- \int \log\circ g(\partial_{k}f) +( \log\circ g)'(\partial_{k}f)\partial_{k}\varepsilon \d\pi + \mathcal{O}(\|\varepsilon\|_{V_k}^2) \\
 &= \mathcal{L}_{k}(f) + \ell(\varepsilon) + \mathcal{O}(\|\varepsilon\|_{V_k}^2)
\end{align*}
where $\ell\colon V_k\rightarrow\R$ is the linear form defined by
$$
 \ell(\varepsilon) = \int \mathcal{R}_k(f)(\bx)\left(\varepsilon(\bx_{<k},0) + \int_0^{x_k} g'\big(\partial_k f(\bx_{<k},t)\big)\partial_k \varepsilon(\bx_{<k},t)\mathrm{d}t\right) - (\log\circ g)'(\partial_{k}f(\bx))\partial_{k}\varepsilon(\bx) ~\pi(\bx)\d\bx.
$$
If $\ell$ is continuous, meaning if there exists a constant $C_\ell$ such that $|\ell(\varepsilon)|\leq C_\ell \|\varepsilon\|_{V_k}$ for any $\varepsilon\in V_k$, then the Riesz representation theorem states that there exists a vector $\nabla\mathcal{L}_k(f)\in V_k$ such that $\ell(\varepsilon)=\langle \nabla\mathcal{L}_k(f),\varepsilon\rangle_{V_k}$. This proves $\mathcal{L}_k$ is differentiable everywhere.

To show that $\ell$ is continuous, we write
\begin{align*}
 |\ell(\varepsilon)|
 &\overset{\eqref{eq:AssumptionPiBounded}}{\leq } C_\pi\int \Big|\mathcal{R}_k(f)(\bx)\left(\varepsilon(\bx_{<k},0) + \int_0^{x_k} g'\big(\partial_k f(\bx_{<k},t)\big)\partial_k \varepsilon(\bx_{<k},t)\mathrm{d}t\right) \Big| \eta_{\leq k}(\bx)\d\bx \\
 &~~+ C_\pi\int\Big|(\log\circ g)'(\partial_{k}f(\bx))\partial_{k}\varepsilon(\bx) \Big|\eta_{\leq k}(\bx)\d\bx \\
 &\overset{\eqref{eq:lipschitz_logg}}{\leq} C_\pi \|\mathcal{R}_k(f)\|_{L^2_{\eta_{\leq k}}} \sqrt{\int \Big|\varepsilon(\bx_{<k},0) + \int_0^{x_k} g'\big(\partial_k f(\bx_{<k},t)\big)\partial_k \varepsilon(\bx_{<k},t)\mathrm{d}t \Big|^2 \eta_{\leq k}(\bx)\d\bx }+ C_\pi L\|\partial_{k}\varepsilon \|_{L_{\eta_{\leq k}}^2} \\
 &\overset{\eqref{eq:lipschitz_g}}{\leq} C_\pi \|\mathcal{R}_k(f)\|_{L^2_{\eta_{\leq k}}} \sqrt{2 C_T\|\varepsilon\|_{V_k}^2 +2C_H L^2\int \big| \partial_k \varepsilon(\bx)) \big|^2 \eta_{\leq k}(\bx)\d\bx  } + C_\pi L\|\partial_{k}\varepsilon \|_{L_{\eta_{\leq k}}^2} \\
 &\;\leq\; C_\pi \Big( \|\mathcal{R}_k(f)\|_{L^2_{\eta_{\leq k}}} \sqrt{2 C_T + 2C_H L^2} + L \Big) \|\varepsilon\|_{V_k},
\end{align*}
where the second last inequality also uses Proposition~\ref{thm:trace} and Lemma~\ref{lemma:hardy_application}. This concludes the proof.
\end{proof}

\subsection{Proof of the local Lipschitz regularity \eqref{eq:Lipschitz_gradient}} \label{proof:Lipschitz_gradient}

\begin{proposition}
In addition to the assumptions of Theorem~\ref{prop:finite_objective}, we further assume there exists a constant $L<\infty$ such that for all $\xi,\xi'\in \R$ we have
\begin{align}
    |g'(\xi)-g'(\xi')| &\leq L |\xi-\xi'| \label{eq:gPrimeLip}\\
    |(\log\circ g)'(\xi)-(\log\circ g)'(\xi')| &\leq L |\xi-\xi'|. \label{eq:loggPrimeLip}
\end{align} %
Then there exists $M<\infty$ such that
$$
 \|\nabla \mathcal{L}_k(f_1)-\nabla\mathcal{L}_k(f_2) \|_{V_k} \leq M(1+\| \Rectifier_k(f_2)\|_{V_k}) \| f_1 - f_2 \|_{\overline{V}_k}  ,
$$
for any $f_1,f_2 \in \overline{V}_k$, where $\overline{V}_k = \{f \in V_k, \partial_k f \in L^\infty\}$ is the space endowed with the norm $\| f \|_{\overline{V}_k} = \|f\|_{V_k} + \|\partial_k f \|_{L^\infty}$.

\end{proposition}
\begin{proof}
    Recall the definition \eqref{eq:gradL} of $\nabla \mathcal{L}_k(f)$ 
    \begin{align*}
         \langle \nabla\mathcal{L}_k(f) , \varepsilon \rangle_{V_k} 
         &= \int \mathcal{R}_k(f)(\bx)\left(\varepsilon(\bx_{<k},0) + \int_0^{x_k} g'\big(\partial_k f(\bx_{<k},t)\big)\partial_k \varepsilon(\bx_{<k},t)\mathrm{d}t\right) \pi(\bx)\d\bx \\
         &- \int(\log\circ g)'(\partial_{k}f(\bx))\partial_{k}\varepsilon(\bx) \pi(\bx)\d\bx.
    \end{align*}
    Then for any $f_1,f_2\in \overline{V}_k$. we can write
    $$
    \langle  \nabla \mathcal{L}_k(f_1)-\nabla\mathcal{L}_k(f_2) , \varepsilon \rangle_{V_k} 
    = A+B+C+D ,
    $$
    where
    \begin{align*}
        A&= \int  \Big(\Rectifier_k(f_1)(\bx) - \Rectifier_k(f_2)(\bx) \Big)\varepsilon(\bx_{<k},0) \pi(\bx) \d\bx \\
        B&=\int  \Big(\Rectifier_k(f_1)(\bx) - \Rectifier_k(f_2)(\bx)\Big)\left(\int_0^{x_k} g'(\partial_k f_1(\bx_{<k},t)) \partial_k \varepsilon(\bx_{<k},t) \d t \right)\pi(\bx) \d\bx  \\
        C&= \int  \Rectifier_k(f_2)(\bx)\left(\int_0^{x_k} \Big(g'(\partial_k f_1(\bx_{<k},t)) - g'(\partial_k f_2(\bx_{<k},t)) \Big) \partial_k \varepsilon(\bx_{<k},t) \d t  \right)\pi(\bx) \d\bx\\
        D&= \int  \Big( (\log \circ g)'(\partial_k f_1(\bx)) - (\log \circ g)'(\partial_k f_2(\bx)) \Big)\partial_k \varepsilon(\bx)   \pi(\bx) \d\bx. 
    \end{align*}
    For the first term $A$ we write
    \begin{align*}
        |A|
        &\overset{\eqref{eq:AssumptionPiBounded}}{\leq } C_\pi \int  \Big|\Rectifier_k(f_1)(\bx) - \Rectifier_k(f_2)(\bx) \Big| |\varepsilon(\bx_{<k},0)| \eta(\bx) \d\bx \\
        &\;\leq\; C_\pi \|\Rectifier_k(f_1) - \Rectifier_k(f_2)\|_{V_k} \left( \int |\varepsilon(\bx_{<k},0)|^2 \eta_{<k}(\bx)\d\bx \right)^{1/2} \\
        &\overset{\eqref{eq:trace_theorem}}{\leq } C_\pi \sqrt{C_T} \|\Rectifier_k(f_1) - \Rectifier_k(f_2)\|_{V_k} \|\varepsilon \|_{V_k}  \\
        &\overset{\eqref{eq:RectifierIsLip}}{\leq } C_\pi \sqrt{C_T} C \| f_1 - f_2 \|_{V_k} \|\varepsilon \|_{V_k}.
    \end{align*}
    For the second term $B$ we write
    \begin{align*}
        |B|
        &\overset{\eqref{eq:AssumptionPiBounded}}{\leq } C_\pi \|\Rectifier_k(f_1) - \Rectifier_k(f_2)\|_{V_k}
        \left(\int  \left(\int_0^{x_k} g'(\partial_k f_1(\bx_{<k},t)) \partial_k \varepsilon(\bx_{<k},t) \d t \right)^2\eta(\bx) \d\bx \right)^{1/2}\\
        &\overset{\substack{\eqref{eq:Hardy} \\ \eqref{eq:RectifierIsLip}}}{\leq } C_\pi \sqrt{C_H} C \| f_1-f_2\|_{V_k}
        \Big(\int  \left(g'(\partial_k f_1(\bx_{\leq k})) \partial_k \varepsilon(\bx_{\leq k}) \Big)^2\eta(\bx) \d\bx \right)^{1/2} \\
        &\overset{\eqref{eq:lipschitz_g}}{\leq } C_\pi \sqrt{C_H}  C L \| f_1-f_2\|_{V_k}
        \Big(\int  \left( \partial_k \varepsilon(\bx_{\leq k}) \Big)^2\eta(\bx) \d\bx \right)^{1/2}\\
        &\;\leq\;  C_\pi \sqrt{C_H}  C L\| f_1-f_2\|_{V_k} \|\varepsilon\|_{V_k} .
    \end{align*}
    For the third term $C$ we write
    \begin{align*}
        |C|
        &\overset{\eqref{eq:AssumptionPiBounded}}{\leq } C_\pi \| \Rectifier_k(f_2)\|_{V_k}
        \left(\int  \left(\int_0^{x_k} \Big(g'(\partial_k f_1(\bx_{<k},t)) - g'(\partial_k f_2(\bx_{<k},t)) \Big) \partial_k \varepsilon(\bx_{<k},t) \d t  \right)^2\eta(\bx) \d\bx \right)^{1/2} \\
        &\overset{\eqref{eq:Hardy}}{\leq } C_\pi \sqrt{C_H} \| \Rectifier_k(f_2)\|_{V_k}
        \left(\int  \left(  \Big(g'(\partial_k f_1(\bx_{\leq k})) - g'(\partial_k f_2(\bx_{\leq k})) \Big) \partial_k \varepsilon(\bx_{\leq k}) \right)^2\eta(\bx) \d\bx \right)^{1/2} \\
        &\;\leq\; C_\pi \sqrt{C_H} \| \Rectifier_k(f_2)\|_{V_k}
        \left(\esssup \Big| g'\circ \partial_k f_1 - g'\circ \partial_k f_2 \Big| \right)
        \left(\int  \left(   \partial_k \varepsilon(\bx_{\leq k}) \right)^2\eta(\bx) \d\bx \right)^{1/2} \\
        &\overset{\eqref{eq:gPrimeLip}}{\leq } C_\pi \sqrt{C_H} L \| \Rectifier_k(f_2)\|_{V_k}
        \left(\esssup \Big| \partial_k f_1 - \partial_k f_2 \Big| \right) \|\varepsilon\|_{V_k} \\
        &\;\leq\; C_\pi \sqrt{C_H} L \| \Rectifier_k(f_2)\|_{V_k}
        \| f_1 - f_2 \|_{\overline{V}_k}  \|\varepsilon\|_{V_k}.
    \end{align*}
    For the last term $D$ we write
    \begin{align*}
        |D| 
        &\overset{\eqref{eq:AssumptionPiBounded}}{\leq } C_\pi \left( \int  \Big( (\log \circ g)'(\partial_k f_1(\bx)) - (\log \circ g)'(\partial_k f_2(\bx)) \Big)^2  \eta(\bx) \d\bx \right)^{1/2} \|\varepsilon\|_{V_k} \\
        &\overset{\eqref{eq:loggPrimeLip}}{\leq } C_\pi L \left( \int  \Big( \partial_k f_1(\bx) -  \partial_k f_2(\bx) \Big)^2   \eta(\bx) \d\bx \right)^{1/2} \|\varepsilon\|_{V_k} \\
        &\;\leq\; C_\pi L \|f_1-f_2\|_{V_k} \|\varepsilon\|_{V_k} .
    \end{align*}
    Thus, because $\|f_1-f_2\|_{V_k}  \leq \|f_1-f_2\|_{\overline{V}_k}  $ we obtain 
    \begin{align*}
     \frac{|\langle  \nabla \mathcal{L}_k(f_1)-\nabla\mathcal{L}_k(f_2) , \varepsilon \rangle_{V_k} |}{\|\varepsilon\|_{V_k}} 
     &\leq 
     C_\pi\Big( \sqrt{C_T} C + \sqrt{C_H}  C L +  \sqrt{C_H} L \| \Rectifier_k(f_2)\|_{V_k}+ L\Big)\| f_1 - f_2 \|_{\overline{V}_k} 
    \\
    &\leq M(1+\| \Rectifier_k(f_2)\|_{V_k}) \| f_1 - f_2 \|_{\overline{V}_k} ,
    \end{align*}
    where
    $$
     M = C_\pi \max\{  \sqrt{C_T} C + \sqrt{C_H}  C L + L;  \sqrt{C_H} L \}.
    $$
    This concludes the proof.
\end{proof}

\subsection{Proof of Proposition~\ref{prop:Convexity_ImR}} \label{proof:ConvexityImageR}
\begin{proof}
To show that $\Rectifier_k(V_k)=\{\Rectifier(f):f\in V_k\}$ is convex, let $f_1,f_2\in V_k$ and $0\leq \alpha\leq 1$. We need to show that there exists $f_\alpha\in V_k$ such that $\Rectifier(f_\alpha) = S_\alpha$ where
$$
 S_\alpha \coloneqq \alpha \Rectifier_k(f_1) + (1 - \alpha)\Rectifier_k(f_2).
$$
Let 
\begin{align}
f_\alpha(\bx_{\leq k}) \coloneqq \Rectifier^{-1}(S_\alpha)(\bx_{\leq k}) = S_\alpha(\bx_{<k},0) + \int_0^{x_k} g^{-1}(\partial_k S_\alpha(\bx_{<k},t)) \textrm{d}t. \nonumber 
\end{align}
It remains to show that $f_\alpha\in V_k$, meaning that $f_\alpha\in L^2_{\eta_{\leq k}}$ and $\partial_k f_\alpha\in L^2_{\eta_{\leq k}}$.
By convexity of $\xi\mapsto g^{-1}(\xi)^2$, we have
\begin{align}
    \|\partial_k f_\alpha\|_{L^2_{\eta_{\leq k}}}^2 
    &= \int g^{-1}\big(\alpha \partial_k \mathcal{R}_k(f_1) + (1 - \alpha) \partial_k \mathcal{R}_k(f_2)\big)^2 \textrm{d}\eta_{\leq k} \nonumber \\
    &= \int g^{-1}\big(\alpha g(\partial_k f_1) + (1 - \alpha) g(\partial_k f_2)\big)^2 \textrm{d}\eta_{\leq k} \nonumber \\
    &\leq \int \alpha g^{-1}(g(\partial_k f_1))^2 + (1 - \alpha) g^{-1}(g(\partial_k f_2))^2 \textrm{d}\eta_{\leq k} \nonumber \\
    &= \alpha \|\partial_k f_1 \|^2_{L^2_{\eta_{\leq k}}} + (1 - \alpha) \| \partial_k f_2 \|^2_{L^2_{\eta_{\leq k}}} . \label{eq:L2_partial_falpha}
\end{align}
Thus $\partial_k f_\alpha \in L^2_{\eta_{\leq k}}$. 
Furthermore we have
\begin{align}
    \|f_\alpha\|_{L^2_{\eta_{\leq k}}}^2
    &= \int \left(S_\alpha(\bx_{<k},0) + \int_0^{x_k} g^{-1}(\partial_k S_\alpha(\bx_{<k},t)) \textrm{d}t \right)^2 \eta_{\leq k}(\bx) \textrm{d}\bx \nonumber \\
    &\leq 2 \int  S_{\alpha}(\bx_{<k},0)^2 \eta_{<k}(\bx_{<k}) \textrm{d}\bx + 2 \int \left(\int_0^{x_k} g^{-1}(\partial_k S_\alpha(\bx_{<k},t)) \textrm{d}t \right)^2  \eta_{\leq k}(\bx) \textrm{d}\bx. \nonumber
\end{align}
To show that the above quantity is finite, Proposition~\ref{thm:trace} permits us to write
\begin{align*}
    \int  S_{\alpha}(\bx_{<k},0)^2 \eta_{<k}(\bx_{<k}) \textrm{d}\bx
    &= \int  \Big( \alpha f_1(\bx_{<k},0) + (1 - \alpha) f_2(\bx_{<k},0) \Big)^2 \eta_{<k}(\bx_{<k}) \textrm{d}\bx \\
    &\leq C_T \| \alpha f_1  + (1 - \alpha) f_2  \|^2_{V_k} ,
\end{align*}
which is finite.
Finally, because $g^{-1}(\partial_k S_\alpha) =\partial_k f_\alpha  \in L^2_{\eta_{\leq k}}$ by~\eqref{eq:L2_partial_falpha}, Lemma \ref{lemma:hardy_application} yields
\begin{align*}
 \int \left(\int_0^{x_k} g^{-1}(\partial_k S_\alpha(\bx_{<k},t)) \textrm{d}t \right)^2  \eta_{\leq k}(\bx) \d\bx 
 &\leq C_{H} \int g^{-1}(\partial_k S_\alpha(\bx_{\leq k}))^2 \eta_{\leq k}(\bx) \d\bx \\
 &= C_{H} \| \partial_k f_\alpha \|_{L^2_{\eta_{\leq k}}}^2,
\end{align*}
which is finite. We deduce that $f_\alpha \in L^2_{\eta_{\leq k}}$ and therefore that $f_\alpha\in V_k$.
\end{proof}

\subsection{Proof of Proposition~\ref{prop:continuity_Rinv}} \label{proof:continuity_Rinv}
\begin{proof}  %
Let $s_1,s_2 \in V_{k}$ be strictly increasing functions with respect to $x_k$ that satisfy $\partial_k s_i(\bx_{\leq k}) \geq c$ for $i=1,2$ and all $\bx_{\leq k} \in \R^{k}$. By the Lipschitz property of $g^{-1}$ on the domain $[c,\infty)$ with constant $L_c$, we can write
\begin{align}
    \|\partial_k \Rectifier_k^{-1}(s_1) - \partial_k \Rectifier_k^{-1}(s_2) \|_{L^2_{\eta_{\leq k}}}^2 &= \int (g^{-1}(\partial_k s_1(\bx_{\leq k})) - g^{-1}(\partial_k s_2(\bx_{\leq k})))^2 \eta_{\leq k}(\bx) \textrm{d}\bx  \nonumber \\
    &\leq L_{c}^2 \int (\partial_k s_1(\bx_{\leq k}) - \partial_k s_2(\bx_{\leq k}))^2 \eta_{\leq k}(\bx) \textrm{d}\bx \nonumber \\
    &\leq L_{c}^2 \|s_1 - s_2 \|_{V_k}^2. \label{eq:continuity_derivative_inverseR}
\end{align}

Applying Proposition~\ref{thm:trace} to $s_1,s_2 \in V_k$ and Lemma~\ref{lemma:hardy_application} to $\partial_k \Rectifier_k^{-1}(s_i) = g^{-1}(\partial_k s_i) \in L_{\eta_{\leq k}}^2$ for $i=1,2$ we have
\begin{align}
\| \Rectifier_k^{-1}(s_1) - \Rectifier_k^{-1}(s_2) \|_{L^2_{\eta_{\leq k}}}^2 &\leq 2 \int \left(s_1(\bx_{<k},0) - s_2(\bx_{<k},0) \right)^2 \eta_{\leq k}(\bx)d\bx \nonumber \\
&\quad+ 2\int \left(\int_0^{x_k} g^{-1}(\partial_k s_1) - g^{-1}(\partial_k s_2) \textrm{d}t \right)^2 \eta_{\leq k}(\bx) \textrm{d}\bx \nonumber \\
&\leq 2 C_T \|s_1 - s_2 \|_{V_k}^2 \nonumber \\
&\quad+ 2 C_H \int (g^{-1}(\partial_k s_1(\bx_{\leq k})) - g^{-1}(\partial_k s_2(\bx_{\leq k})))^2 \eta_{\leq k}(\bx) \textrm{d}\bx \nonumber \\ %
&\leq (2C_T + 2C_H L_{c}^2)\|s_1 - s_2 \|_{V_k}^2, \label{eq:continuity_inverseR}
\end{align}
where the last inequality follows from~\eqref{eq:continuity_derivative_inverseR}. %

It remains to show that $\|\Rectifier_k^{-1}(s)\|_{V_k} < \infty$ for any $s \in V_{k}$ such that $\essinf \partial_k s > 0$. Letting $s_1 = s$ and $s_2 = g(0)x_k$, the triangle inequality combined with~\eqref{eq:Lipschitz_Rinv} yields
\begin{equation*}
    \| \Rectifier_k^{-1}(s) \|_{V_k} \leq  \| \Rectifier_k^{-1}(g(0)x_k) \|_{V_k} + C_{c} \|s - g(0)x_k \|_{V_k}.
\end{equation*}
The function $\Rectifier_k^{-1}(g(0)x_k)$ is zero. Therefore, $\| \Rectifier_k^{-1}(s) \|_{V_k} \leq C_c\|s - g(0)x_k \|_{V_k} \leq C_c(\| s \|_{V_k} + \|g(0)x_k \|_{V_{k}})$. For a linear function, $\|g(0)x_k \|_{V_k}^2 = \|g(0)x_k \|^2_{L^2_{\eta_{\leq k}}} + \|g(0) \|^2_{L^2_{\eta_{\leq k}}} = 2g(0)^2$ is finite, and so $\| \Rectifier_k^{-1}(s) \|_{V_k} < \infty$ for $s \in V_k$. Furthermore, if $\partial_k s \geq c > 0$, then $\partial_k \Rectifier_k^{-1}(s) = g^{-1}(\partial_k s) \geq g^{-1}(c) > -\infty$ and so $\essinf \Rectifier_k^{-1}(s) > -\infty$. %
\end{proof}

\subsection{Proof for the KR rearrangement} 
\label{proof:KRrearrangementTail}

\begin{proof} 
Let $S_{\KR,k}$ be the $k$th component of the KR rearrangement, %
given by composing the inverse CDF of the standard Gaussian marginal $F_{\eta,k}(x_k)$ with the CDF of the target's $k$th marginal conditional $F_{\pi_k}(x_k|\bx_{<k})$.
That is, %
\begin{equation}\label{eq:tmp3850789274}
    S_{\KR,k}(\bx_{\leq k}) = F_{\eta_k}^{-1} \circ F_{\pi_k}(x_k|\bx_{<k}). 
\end{equation}
The goal is to show $S_{\KR,k} \in V_k$, that is, $S_{\KR,k} \in L^2_{\eta_{\leq k}}$ and $\partial_k S_{\KR,k} \in L^2_{\eta_{\leq k}}$. 

First we show $S_{\KR,k} \in L^2_{\eta_{\leq k}}$. From condition~\eqref{eq:margConditional_bothTailsBounded}, we have $F_{\eta_k}^{-1}(C_1 F_{\eta_k}(x_k)) \leq S_{\KR,k}(x_k|\bx_{<k}) \leq F_{\eta_k}^{-1}(C_2 F_{\eta_k}(x_k))$ for some constants $C_1,C_2 > 0$ so that
\begin{equation}\label{eq:tmp468923}
    S_{\KR,k}(x_k|\bx_{<k})^2 \leq \max\{ F_{\eta_k}^{-1}(C_1 F_{\eta_k}(x_k))^2 ; F_{\eta_k}^{-1}(C_2 F_{\eta_k}(x_k))^2 \},
\end{equation}
for all $\bx_{<k} \in \R^{k-1}$. To show that $S_{\KR,k} \in L^2_{\eta_{\leq k}}$, it is sufficient to prove that any function of the form $x_k\mapsto F_{\eta_k}^{-1}( C F_{\eta_k}(x_k))$ is in $ L^2_{\eta_{\leq k}}$ for any $C>0$.
From Theorems 1 and 2 in~\cite{chang2011chernoff}, there exists strictly positive constants $\alpha_i,\beta_i > 0$ for $i=1,2$ such that 
\begin{equation} \label{eq:bounds_Gaussian_cdf}
    1 - \alpha_1 \exp(-\beta_1 x_k^2) \leq F_{\eta_k}(x_k) \leq 1 - \alpha_2 \exp(-\beta_2 x_k^2),
\end{equation}
for $x_k > 0$. With a change of variable $u=F_{\eta_k}(x_k)$ we obtain $F^{-1}_{\eta_k}(u)^2 \leq 1/\beta_{2} \log(\alpha_2/(1-u))$ for all $u > F_{\eta_k}(0) = 1/2$. Letting $u=C F_{\eta_k}(x_k)$ yields
\begin{align*} 
  F^{-1}_{\eta_k}( C F_{\eta_k}(x_k) )^2 
  &\;\leq\; \frac{1}{\beta_{2}} \log\left(\frac{\alpha_2}{C F_{\eta_k}(x_k)}\right),\nonumber \\
  &\overset{\eqref{eq:bounds_Gaussian_cdf}}{\leq}  \frac{1}{\beta_{2}} \log\left(\frac{\alpha_2}{C \alpha_2 \exp(-\beta_2 x_k^2)}\right) \nonumber\\ \label{eq:bounds_inverse_Gaussian_cdf}
  &\;=\;  \frac{1}{\beta_{2}} \log\left(\frac{1}{C}\right) + x_k^2
\end{align*} 
for all $x_k > \max\{ F_{\eta_k}^{-1}(1/(2C)) , 0\}$.
Using the same argument, we obtain a similar bound on $F^{-1}_{\eta_k}( C F_{\eta_k}(x_k) )^2$ for all $x_k$ smaller than a certain value. Together with the continuity of $x_k\mapsto F^{-1}_{\eta_k}( C F_{\eta_k}(x_k) )^2$ these bounds ensure that $x_k\mapsto F^{-1}_{\eta_k}( C F_{\eta_k}(x_k) )$ is in $L^2_{\eta_{\leq k}}$ for any $C$. Then $S_{\KR,k} \in L^2_{\eta_{\leq k}}$. Furthermore, we have  $S_{\KR,k}(\bx_{\leq k}) = \mathcal{O}(x_k)$ as $|x_k| \rightarrow \infty$.

Now we show that $\partial_k S_{\KR,k} \in L^2_{\eta_{\leq k}}$ by showing $\partial_k S_{\KR,k}$ is a continuous and bounded function. From the absolute continuity of $\bmu$ and $\bnu$, we have that  %
\begin{equation} \label{eq:KRmapderivative}
    \partial_k S_{\KR,k}(\bx_{\leq k}) =  \frac{\pi_k(x_k|\bx_{<k})}{\eta_k(S_{\KR,k}(\bx_{\leq k}))} =  \frac{\pi_k(F_{\pi_k}^{-1}(F_{\pi_k}(x_k|\bx_{<k})|\bx_{<k})|\bx_{<k})}{\eta_k(F_{\eta_k}^{-1} \circ F_{\pi_k}(x_k|\bx_{< k}))},  %
\end{equation}
is continuous, where $F_{\pi_k}^{-1}(\cdot|\bx_{<k})$ denotes the inverse of the map $x_k \mapsto F_{\pi_k}(x_k|\bx_{<k})$ for each $\bx_{<k} \in \R^{k-1}$. %
Hence, it is sufficient to show that $\partial_k S_{\KR,k}$ goes to a finite limit as $|x_k| \rightarrow \infty$. %
For the right-hand limit, we can write
\begin{align} \label{eq:limit_left_tail}
    \lim_{x_k \rightarrow \infty} \partial_k S_{\KR,k}(\bx_{\leq k}) &= \lim_{u \rightarrow 1^{-}} \frac{\pi_{k}(F_{\pi_k}^{-1}(u|\bx_{<k})|\bx_{<k})}{\eta_k(F_{\eta_k}^{-1}(u))} \nonumber \\
    &= \lim_{u \rightarrow 1^{-}} \frac{(F_{\eta,k}^{-1})'(u)}{(F_{\pi_k}^{-1})'(u|\bx_{<k})} \nonumber \\
    &= \lim_{u \rightarrow 1^{-}} \frac{F_{\eta_k}^{-1}(u)}{F_{\pi_k}^{-1}(u|\bx_{<k})},
\end{align}
where in the second equality we used the inverse function theorem and the third equality follows from l'H\^{o}pital's rule. 
To analyze the ratio $F_{\eta_k}^{-1}(u)/F_{\pi_k}^{-1}(u|\bx_{<k})$, we combine the lower bound in~\eqref{eq:margConditional_bothTailsBounded} and the bounds in~\eqref{eq:bounds_Gaussian_cdf} to get
\begin{equation*}
    \frac{F_{\eta_k}^{-1}(u)}{F_{\pi_k}^{-1}(u|\bx_{<k})} \leq \frac{F_{\eta_k}^{-1}(u)}{F_{\eta_k}^{-1}((u-1)/C_1 + 1)} \leq  \sqrt{\frac{\beta_1(\log\alpha_2 - \log(1-u))}{\beta_2(\log\alpha_1 - \log((1-u)/C_1))}}.
\end{equation*}
Similarly, from the upper bound in~\eqref{eq:margConditional_bothTailsBounded} and the bounds in~\eqref{eq:bounds_Gaussian_cdf}, we have
\begin{equation*}
    \frac{F_{\eta_k}^{-1}(u)}{F_{\pi_k}^{-1}(u|\bx_{<k})} \geq \frac{F_{\eta_k}^{-1}(u)}{F_{\eta_k}^{-1}((u-1)/C_2 + 1)} \geq \sqrt{\frac{\beta_2(\log\alpha_1 - \log(1-u))}{\beta_1(\log\alpha_2 - \log((1-u)/C_2))}}.
\end{equation*}
Thus, $\partial_k S_{\KR,k}(\bx_{\leq k}) = \mathcal{O}(1)$ as $x_k \rightarrow \infty$, 
and we have $\partial_k S_{\KR,k} \in L^2_{\eta_k}$.

Lastly, taking the limit in~\eqref{eq:limit_left_tail} we have %
$\lim_{x_k \rightarrow \infty} \partial_k S_{\KR,k}(\bx_{\leq k}) \geq \sqrt{\beta_2/\beta_1}$. For a target distribution $\pi$ with full support, all marginal conditional densities satisfy $\pi_k(x_k|\bx_{<k}) > 0$ for each $\bx_{\leq k} \in \R^{k}$. Given that the $\partial_k S_{\KR,k}$ does not approach zero as $|x_k| \rightarrow \infty$, we can find %
a strictly positive constant $c_k > 0$ such that $\partial_k S_{\KR,k}(\bx_{\leq k}) \geq c_k$ for all $\bx_{\leq k} \in \mathbb{R}^k$. 
This shows that $\essinf \partial_k S_{\KR,k} > 0$. 
\end{proof}

\section{Multi-index refinement for the wavelet basis} \label{subsec:multi_index_wavelet}
In this section we show how to greedily enrich the index set $\Lambda_t$ for a one-dimensional wavelet basis parameterized by the tuple of indices $(l,q)$ representing the level $l$ and translation $q$ of each wavelet $\psi_{(l,q)}$. To define the allowable indices, we construct a binary tree where each node is indexed by $(l,q)$ and has two children with indices $(l+1,2q)$ and $(l+1,2q+1)$. The root of the tree has index $(0,0)$ and corresponds to the mother wavelet $\psi$. Analogously to the downward closed property for polynomial indices, we only add nodes to the tree (i.e., indices in $\Lambda_t$) if its parents have already been added. Given any set $\Lambda_t$, we define its reduced margin as %
$$\Lambda_{t}^{\text{RM}} = \left\{\alpha=(l,q) \not\in \Lambda_t \text{ such that } \begin{array}{ll} (l-1,q/2) \in \Lambda_t & \text{ for odd } q \\ (l-1,(q-1)/2) \in \Lambda_t & \text{ for even } q \end{array} \right\}.$$
Then, the \ALG algorithm with a wavelet basis follows from  Algorithm~\ref{alg:learn-map} with this construction for the reduced margin at each iteration.

\section{Architecture details of alternative methods} \label{app:architectures}
In this section we present the details of the alternative methods to \ALG that we consider in Section~\ref{sec:experiments}. %

For each normalizing flow model, we use the recommended stochastic gradient descent optimizer with a learning rate of $10^{-3}$. We partition 10\% of the samples in each training set to be validation samples and use the remaining samples for training the model. We select the optimal hyper-parameters for each dataset by fitting the density with the training data and choosing the parameters that minimize the negative log-likelihood of the approximate density on the validation samples. We also use the validation samples to set the termination criteria during the optimization.

We follow the implementation of~\cite{rothfuss2019conditional} to define the architectures of these models. The hyper-parameters we consider for the neural networks in the \MDN and \NF models are: $2$ hidden layers, $32$ hidden units in each layer, $\{5,10,20,50,100\}$ centers or flows, weight normalization, and a dropout probability of $\{0,0.2\}$ for regularizing the neural networks during training. For \CKDE and \NKDE we select the bandwidth of the kernel estimators using $5$-fold cross-validation.

\def\bibfont{\small}
\bibliographystyle{imsart-number}
\bibliography{references}  

\end{document}